\documentclass[final,5p,times,twocolumn]{elsarticle}

\usepackage[T1]{fontenc}
\usepackage{amsmath,amssymb,amsfonts,amsthm,mathtools,mathrsfs}
\usepackage{microtype}
\usepackage{graphicx}
\graphicspath{{}}
\DeclareGraphicsExtensions{.pdf,.png,.jpg,.jpeg}
\usepackage{array,booktabs,adjustbox}
\newcolumntype{L}[1]{>{\raggedright\arraybackslash}p{#1}}
\usepackage{tikz}
\usetikzlibrary{arrows.meta,positioning,fit,backgrounds,calc}
\usepackage{algorithm}
\usepackage{algpseudocode}
\usepackage[dvipsnames]{xcolor}
\usepackage{csquotes}
\usepackage{xurl}
\usepackage{dblfloatfix}
\usepackage{placeins}
\usepackage[pdfpagelabels,hyperindex,hidelinks]{hyperref}
\hypersetup{
  pdftitle={TwistedMerge: Certified Higher-Order Diagnostics and Abstention for Model Merging},
  pdfauthor={Ting Gong and Shitan Xu},
  pdfsubject={Certified higher-order diagnostics, no-go results, and abstention in model merging},
  pdfkeywords={model merging, gauge synchronization, higher-order diagnostics, cohomology, holonomy, abstention}
}

\biboptions{sort&compress}

\makeatletter
\newenvironment{breakablealgorithm}
{%
  \par
  \addvspace{\intextsep}
  \refstepcounter{algorithm}%
  \hrule height 0.8pt depth 0pt
  \kern 2pt
  \renewcommand{\caption}[2][\relax]{%
    {\raggedright
     \textbf{\ALG@name~\thealgorithm.} ##2\par}%
    \ifx\relax##1\relax
      \addcontentsline{loa}{algorithm}
      {\protect\numberline{\thealgorithm}##2}%
    \else
      \addcontentsline{loa}{algorithm}
      {\protect\numberline{\thealgorithm}##1}%
    \fi
    \kern 2pt
    \hrule height 0.8pt depth 0pt
    \kern 2pt
  }%
}
{%
  \kern 2pt
  \hrule height 0.8pt depth 0pt
  \par
  \addvspace{\intextsep}
}
\makeatother

\newtheorem{theorem}{Theorem}[subsection]
\newtheorem{corollary}[theorem]{Corollary}
\newtheorem{proposition}[theorem]{Proposition}

\theoremstyle{definition}
\newtheorem{definition}[theorem]{Definition}

\newtheorem{ex}[theorem]{Example}
\theoremstyle{remark}
\newtheorem{remark}[theorem]{Remark}

\newcommand{\ZZ}{\mathbb{Z}}
\newcommand{\CC}{\mathbb{C}}
\newcommand{\GG}{\mathbb{G}}
\newcommand{\cF}{\mathcal{F}}
\newcommand{\sC}{\mathscr{C}}
\newcommand{\sM}{\mathscr{M}}
\newcommand{\sN}{\mathscr{N}}

\newcommand{\sY}{\mathscr{Y}}
\DeclareMathOperator{\Br}{Br}
\DeclareMathOperator{\Def}{Def}
\DeclareMathOperator{\Obs}{Obs}
\DeclareMathOperator{\dist}{dist}
\DeclareMathOperator{\et}{\acute{e}t}
\DeclareMathOperator{\GL}{GL}
\DeclareMathOperator{\ind}{ind}
\DeclareMathOperator{\im}{im}
\DeclareMathOperator{\per}{per}
\DeclareMathOperator{\Pic}{Pic}
\DeclareMathOperator{\PGL}{PGL}
\newcommand{\indrep}{\operatorname{ind}_{\mathrm{rep}}}
\DeclareRobustCommand{\CtwoMthree}{\ensuremath{C^{2}M^{3}}}
\journal{Information Fusion}

\begin{document}
\begin{frontmatter}
\title{TwistedMerge: Certified Higher-Order Diagnostics and Abstention for Model Merging}

\author[uw]{Ting Gong\corref{cor1}}
\ead{tgong2@uw.edu}
\author[bicmr]{Shitan Xu}
\ead{shitanxu@bicmr.pku.edu.cn}
\cortext[cor1]{Corresponding author}

\affiliation[uw]{organization={Department of Mathematics, University of Washington},
            addressline={Box 354350},
            city={Seattle},
            postcode={WA 98195-4350},
            country={USA}}
\affiliation[bicmr]{organization={Beijing International Center for Mathematical Research, Peking University},
            addressline={No. 5 Yiheyuan Road, Haidian District},
            city={Beijing},
            postcode={100871},
            country={P.R. China}}

\begin{abstract}
Model merging combines independently trained or fine-tuned models, but pairwise alignability does not imply globally consistent alignment. We formulate merging as a finite descent problem in which checkpoints are local objects, alignment maps are transitions, and cycle products are residuals. TwistedMerge is a conservative certification pipeline that separates fixed-chart averaging, synchronization-removable gauge inconsistency, a certified central obstruction on a specified comparison complex, and nonabelian holonomy. A residual is promoted to a cohomology class only after inverse-consistency, coefficient-identification, centrality, and closure tests; otherwise the method abstains and returns an ordinary or synchronized fallback. We prove a constant-edge no-go result, frozen-complex three-way and predeclared-family error-control theorems, and a refinement test for comparison-complex sensitivity. A planted neural alignment defect is removed by cycle-consistent synchronization, showing that a nonzero cycle score alone is not a higher obstruction. Controlled central systems recover the predicted non-coboundary and projective-rank behavior, while noisy estimates move from certification to abstention without false lifts on the tested controls. A trained low-rank-adapter audit shows that naive factor averaging depends on the chosen $\GL_r$ representative, whereas global factor synchronization and dense-delta SVD are stable. On natural checkpoint collections, cycle residuals do not predict merge degradation and no natural central or period-index class is certified. The results position descent theory as a falsifiable certification and abstention framework.
\end{abstract}

\begin{keyword}
model merging \sep gauge synchronization \sep higher-order diagnostics \sep cohomology \sep holonomy \sep abstention
\end{keyword}
\end{frontmatter}

\section{Introduction}
Model merging combines independently trained or fine-tuned checkpoints without retraining a single model from scratch. Most methods either assume a common parameter chart or estimate pairwise alignments before averaging. A nonzero cycle residual is tempting to interpret as a higher obstruction. The main idea of this paper is that this inference is generally invalid. Cycle inconsistency may be caused by a synchronization-removable gauge defect, a misspecified coefficient system, incoherent comparison contexts, or estimation noise. A cohomological interpretation is warranted only after the comparison complex, coefficient identifications, centrality, closure, and permitted repair model have been fixed and tested.

This dependence on the comparison complex is decisive. If the checkpoint complex is a full simplex and the coefficient group is constant and abelian, then all positive-degree cohomology vanishes. A nonzero finite $H^2$ certificate therefore requires missing or unavailable higher comparison contexts, and those absences must be determined before the target residual is inspected. We accordingly treat the complex $K$ as part of the scientific design.

TwistedMerge is the resulting certification pipeline. Checkpoints are local objects, pairwise alignments are transitions, and triangle products are residuals. The pipeline separates four regimes: fixed-chart averaging, synchronization-removable gauge inconsistency, a certified central class on a specified comparison complex, and nonabelian holonomy. Its output is three-way---\emph{trivial}, \emph{nontrivial on $K$}, or \emph{uncertified}---and an uncertified structural branch returns to an ordinary or synchronized validation fallback. Projective or branch predictors are admitted only after separate construction and invariant-readout tests; they are not identified with the vanishing of the original class.

\begin{figure*}[!tbp]
\centering
\begin{tikzpicture}[
  flow/.style={draw, rounded corners, align=center, text width=2.75cm, minimum height=1.0cm, font=\scriptsize, inner sep=4pt},
  status/.style={draw, rounded corners, align=center, text width=3.25cm, minimum height=1.15cm, font=\scriptsize, inner sep=4pt},
  arrow/.style={-{Latex[length=2mm]}, thick},
  node distance=7mm and 7mm
]
\node[flow] (freeze) {Freeze $K=\Phi(D_K)$\\before residual or test access};
\node[flow, right=of freeze] (fit) {Fit transitions on $D_{\mathrm{align}}$};
\node[flow, right=of fit] (cert) {Apply inverse, projection, centrality, closure, and margin gates on $D_{\mathrm{cert}}$};

\node[status, below left=9mm and 5.0cm of cert] (strict) {Strict or synchronization-removable\\ordinary/synchronized candidate\\single model, $1\times$ inference};
\node[status, right=4mm of strict] (central) {Certified central\\branch\\trivial / nontrivial on $K$ / uncertified\\optional verified rank-$r$ predictor};
\node[status, right=4mm of central] (hol) {Noncentral holonomy\\optional invariant branch predictor\\branch multiplier recorded};
\node[status, right=4mm of hol] (fallback) {Uncertified structure\\ordinary or synchronized fallback\\single model, $1\times$ inference};

\draw[arrow] (freeze) -- (fit);
\draw[arrow] (fit) -- (cert);
\draw[arrow] (cert.south) -- ++(0,-4mm) -| (strict.north);
\draw[arrow] (cert.south) -- ++(0,-4mm) -| (central.north);
\draw[arrow] (cert.south) -- ++(0,-4mm) -| (hol.north);
\draw[arrow] (cert.south) -- ++(0,-4mm) -| (fallback.north);

\node[flow, below=14mm of central, xshift=1.55cm] (select) {Select admitted candidates on $D_{\mathrm{select}}$; evaluate the frozen choice once on $D_{\mathrm{test}}$};
\draw[arrow] (strict.south) |- (select.west);
\draw[arrow] (central.south) -- (select.north west);
\draw[arrow] (hol.south) -- (select.north east);
\draw[arrow] (fallback.south) |- (select.east);
\end{tikzpicture}
\caption{Core certification flow. The comparison complex is frozen before transition fitting and residual inspection. Structural branches are admitted only after their own certificates, and any increase in rank or branch count is reported as an output cost. Uncertified branches return to a $1\times$ ordinary or synchronized fallback.}
\label{fig:core-certification-flow}
\end{figure*}
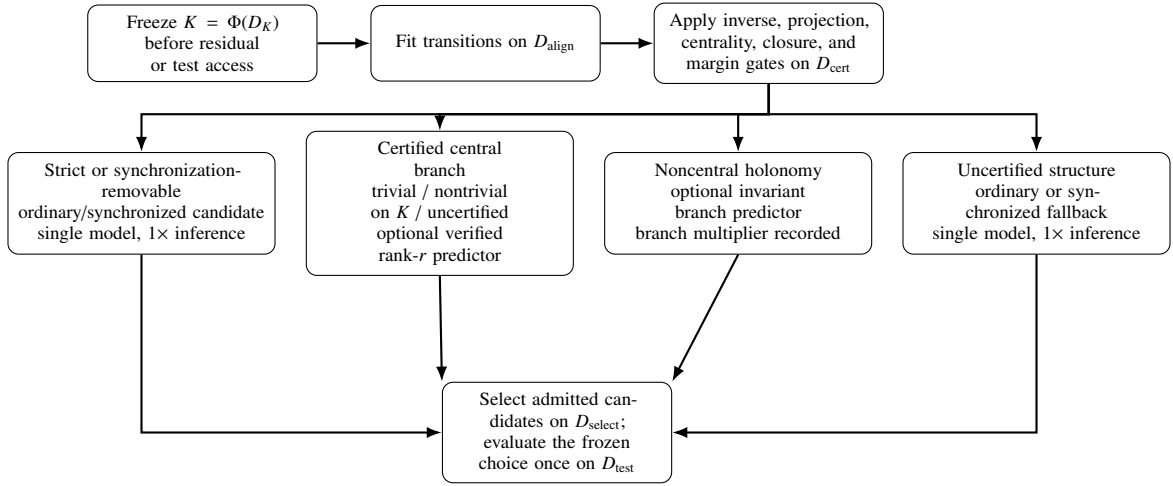

\subsection{Contributions and claim boundary}
The paper makes three contributions.
\begin{enumerate}
\item We formalize a finite certification problem that distinguishes raw cycle inconsistency, strict gauge synchronization, a closed central $H^2$ obstruction on a frozen comparison complex, and nonabelian holonomy. The construction uses disjoint roles for complex selection, transition fitting, certification, candidate selection, and final testing.
\item We prove paper-specific no-go and error-control results. Context-independent central edge matrices cannot realize the one-negative-face tetrahedral class; a frozen-complex three-way rule controls false trivial and false nontrivial declarations under estimated transitions; a simultaneous extension controls multiplicity over a predeclared family of complexes; and a refinement proposition states exactly which certificates persist, disappear, or fail to extend when simplices are changed.
\item We validate the framework through causal, controlled, and natural experiments. A planted neural alignment defect is removed by strict synchronization; trained low-rank adapters expose the representation dependence of naive factor averaging; controlled projective systems verify certification and abstention gates; and natural checkpoint collections fail the proposed higher-order predictive and central-class tests.
\end{enumerate}

The contribution is a certification, falsification, and no-go framework. A nonzero class in $H^2(K;A)$ rules out only the tested $A$-valued edge repair on the stated $K$ and no natural Brauer or period-index class is certified in the present experiments.

\begin{table*}[!tbp]
\centering
\caption{Claim--evidence matrix. Metrics in different rows answer different hypotheses and are not combined into one performance leaderboard.}
\label{tab:claim-evidence-matrix}
\small
\begin{adjustbox}{max width=\textwidth}
\begin{tabular}{L{0.20\textwidth}L{0.18\textwidth}L{0.14\textwidth}L{0.16\textwidth}L{0.25\textwidth}}
\toprule
Hypothesis & Data regime & Output type & Primary metric & Conclusion \\
\midrule
Cycle inconsistency can be synchronization-removable & exact-copy MNIST MLPs with one corrupted edge & single synchronized model & merge degradation & strict synchronization removes the planted defect; cycle score is not an $H^2$ certificate \\
Naive adapter-factor averaging is representation-dependent & five groups of trained rank-$4$ residual adapters & single merged update & update and prediction invariance & factor averaging changes under equivalent $\GL_r$ representatives; global synchronization and dense SVD are stable \\
A prescribed central class can obstruct the permitted repair & tetrahedral $\mu_2$ oracle witness & diagnostic only & exact coboundary test & the one-negative-face class is nontrivial on $\partial\Delta^3$ but is not an end-to-end constant-edge neural realization \\
Conservative abstention prevents unsupported lifts & controlled finite-Heisenberg systems under noise & diagnostic and optional projective realization & coverage, uncertainty, false-lift rate & small-noise positives are certified, intermediate noise abstains, and tested negative controls do not lift \\
Natural cycle residual predicts merge failure & $120$ natural checkpoint collections & diagnostic only & leave-one-setting-out $R^2$ & the hypothesis is unsupported; validation quantities are more predictive \\
\bottomrule
\end{tabular}
\end{adjustbox}
\end{table*}

\subsection{Relation to recent work}

Model-merging methods make different assumptions about the parameter coordinates in which the local models are compared. Git Re-Basin aligns permutation-equivalent networks to a reference model before averaging \cite{Ainsworth2022GitReBasin}, while $C^2M^3$ imposes cycle consistency on permutation alignments across several models \cite{Crisostomi2024C2M3}. Task Arithmetic \cite{ilharco2023editing}, TIES-Merging \cite{yadav2023tiesmerging}, and DARE \cite{yu2024language} operate on parameter updates relative to a shared base model. SLERP instead interpolates along a spherical path in a chosen parameter geometry \cite{shoemake1985animating}. Yang et al.\ give a broad survey of model-merging methods, theories, and applications \cite{YangEtAl2024ModelMergingSurvey}. A recent large-scale study evaluates six merging methods across four open-weight language models, twelve fine-tuned checkpoints per base model, and sixteen benchmarks, and finds that Task Arithmetic is the only tested method that reliably improves performance in its in-the-wild setting \cite{HititEtAl2026InTheWild}.

Recent work has also moved beyond purely pairwise model-merging diagnostics. Zheng and Allen-Blanchette use Hodge decomposition to diagnose higher-order merge failures on a simplicial structure \cite{ZhengAllenBlanchette2026BeyondPairwise}. Karuturi et al.\ analyze the $\GL_r$ gauge symmetry of LoRA through its principal-bundle geometry and discuss its implications for adapter merging \cite{KaruturiEtAl2026LoRAGauge}. Javidnia develops a sheaf-theoretic atlas of neural representations and measures shearing, jamming, and loop holonomy in a frozen large language model \cite{Javidnia2026GaugeSuperposition}. These works are adjacent to the present paper, although they address different structural and certification questions.

\begin{table*}[!tbp]
\centering
\caption{Comparison with recent higher-order, gauge, and sheaf-theoretic work.}
\label{tab:recent-related-work}
\small
\begin{adjustbox}{max width=\textwidth}
\begin{tabular}{L{0.19\textwidth}L{0.18\textwidth}L{0.18\textwidth}L{0.19\textwidth}L{0.20\textwidth}}
\toprule
Work & Object of study & Higher-order object & Gauge or coefficient treatment & Primary objective \\
\midrule
Zheng--Allen-Blanchette \cite{ZhengAllenBlanchette2026BeyondPairwise}
& model-merge residuals on a simplicial structure
& Hodge components of higher-order failure
& simplicial residual decomposition
& diagnosis of failures beyond a star-shaped or pairwise merging pipeline
\\
Karuturi et al.\ \cite{KaruturiEtAl2026LoRAGauge}
& LoRA factor space
& principal-bundle geometry of $\GL_r$ symmetry
& exact rank-space gauge
& geometric implications for optimization and adapter merging
\\
Javidnia \cite{Javidnia2026GaugeSuperposition}
& local semantic charts in an LLM
& loop holonomy, shearing, and jamming
& chart transport with information-geometric weighting
& interpretability and transfer bounds in a frozen LLM
\\
This paper
& checkpoint transitions on a frozen finite complex
& central $2$-cochain or class, or nonabelian holonomy
& coefficient identification, closure, no-go tests, and abstention
& certification of when higher-order language is warranted in model merging
\\
\bottomrule
\end{tabular}
\end{adjustbox}
\end{table*}

The distinction from Hodge decomposition is important. A harmonic component gives a useful numerical decomposition of a residual. In our framework, a residual determines a central cohomology class only after the coefficient system, centrality, and cocycle closure have been certified. The distinction from the LoRA gauge analysis is also explicit. The $\GL_r$ symmetry is standard background, while our experiment audits the representation dependence, global synchronization, systems cost, and fallback behavior of concrete factor-merging rules. The relation to gauge-theoretic superposition is conceptual: our local objects are checkpoints and adapter updates, while Javidnia studies semantic feature charts inside a frozen language model.

The implementation, reports, and audited experimental artifacts used in this paper are available in the TwistedMerge repository \cite{gong2026twistedmerge}.

\subsection{Organization}
Section~\ref{sec:theory} defines the frozen comparison protocol, separates standard background from paper-specific results, proves the no-go, single-complex and familywise error-control, refinement, and invariant-readout statements, and gives the core certification pipeline. Section~\ref{sec:experiments} presents the primary causal, adapter, controlled-certification, and natural falsification experiments. Section~\ref{sec:limitations} states the realization, scale, comparison-complex, and decision-utility limitations. Extended proofs, the full algorithm, and secondary experiments are collected in the appendices.

\section{Finite comparison and certification theory}\label{sec:theory}
In this section, we specify a finite certification problem. Standard descent and projective-representation statements are used as background; the paper-specific contribution lies in the comparison-complex no-go conditions, the frozen-complex three-way certificate, and the conservative interpretation of the resulting diagnostics. Extended proofs and background material are included in the appendices of this paper, beginning with Appendix~\ref{app:additional-theory}.

\subsection{Status of the mathematical results}\label{subsec:result-status}
\begin{table*}[!tbp]
\centering
\caption{Status and role of the principal mathematical statements. ``Adapted'' means that a standard fact is specialized to the model-merging setting; it is not claimed as a new theorem of algebraic geometry.}
\label{tab:theory-status}
\scriptsize
\setlength{\tabcolsep}{2pt}
\begin{adjustbox}{max width=\textwidth}
\begin{tabular}{L{0.28\textwidth}L{0.12\textwidth}L{0.24\textwidth}L{0.28\textwidth}}
\toprule
Result & Status & Source or basis & Role in the pipeline \\
\midrule
Effective descent and central obstruction & standard; adapted & descent and gerbe theory \cite{Giraud_71,Lieblich_07} & separates ordinary gluing from a fixed-cover central obstruction \\
Full-simplex higher-cohomology vanishing & standard & contractibility of a simplex & prevents a nonzero constant-coefficient certificate on a complete nerve \\
Constant-edge tetrahedral no-go & paper-specific & Proposition~\ref{prop:constant-edge-tetrahedron} & shows that the one-negative-face witness cannot come from six context-independent central edge matrices \\
Distance and three-way margin rule & paper-specific adaptation & finite-dimensional distance stability & controls trivial/nontrivial/uncertified decisions under estimated transitions \\
Frozen-complex error control & new in this paper & Theorem~\ref{thm:frozen-complex-certificate} & prevents the comparison complex from being chosen using the residual under test \\
Predeclared-family error control & new in this paper & Theorem~\ref{thm:complex-family-certificate} & controls multiplicity when several complex thresholds are audited \\
Refinement-persistence test & new in this paper & Proposition~\ref{prop:complex-refinement} & formalizes sensitivity to adding, deleting, or subdividing comparison simplices \\
Projective rank thresholds and invariant pooling & standard; adapted & projective representations and invariant maps & supplies necessary rank and path-independence gates \\
Exact ReLU and LoRA gauges & standard; adapted & positive homogeneity and rank-factor symmetry & identifies representation-dependent averaging rules \\
\bottomrule
\end{tabular}
\end{adjustbox}
\end{table*}

\subsection{Finite comparison data and descent background}
In this subsection, we review the traditional mathematical theory of twisted sheaves and descent obstructions. We use the abstract language of algebraic geometry here; readers who want the learning-theoretic interpretation may compare the terms below with the dictionary in Section~\ref{subsec:dictionary}.

Before we begin, we fix a small notational convention. We will sometimes write the trivial obstruction as $1$ and sometimes as $0$. This is only a matter of notation: for multiplicative groups the identity is written as $1$, while for rings, modules, or additive cohomology groups the identity is written as $0$. Thus a condition such as $g_{ij}g_{jk}g_{ki}=1$ is the multiplicative form of strict descent, while the corresponding cohomology class is written as $[c]=0$.

\begin{definition}[Learning site, cover, and overlaps]
\label{def:learning-site-cover}
A \emph{learning site} $\sC_{\mathrm{learn}}$ is a category equipped with a Grothendieck topology. Its objects are local learning contexts, and its covering families specify which collections of local contexts are regarded as jointly covering a global learning problem.

Let $X\in\sC_{\mathrm{learn}}$. A finite learning cover of $X$ is a covering family
\[
\mathcal U=\{u_i:U_i\longrightarrow X\}_{i\in I}
\]
in the chosen Grothendieck topology.

Assume that the relevant fiber products exist. For an ordered tuple $(i_0,\ldots,i_p)$, we write
\[
U_{i_0\cdots i_p}=U_{i_0}\times_X U_{i_1}\times_X\cdots\times_X U_{i_p}.
\]
In particular,
\[
U_{ij}=U_i\times_X U_j, \qquad U_{ijk}=U_i\times_X U_j\times_X U_k.
\]
The object $U_{ij}$ is the context on which two local models are compared, while $U_{ijk}$ is the context on which the compatibility of three pairwise comparisons can be tested.
\end{definition}

\begin{definition}
\label{def:nerve-learning-cover}
Let $\mathcal U=\{U_i\to X\}_{i\in I}$ be a finite learning cover. Its \emph{nerve} $N(\mathcal U)$ is the simplicial object whose $p$-simplices are the overlap objects $U_{i_0\cdots i_p}$, with the usual face and degeneracy maps.

After passing to an incidence complex, a simplex $(i_0,\ldots,i_p)$ is present when the corresponding overlap object is noninitial, or nonempty in settings where an underlying-space notion of emptiness is available. No higher overlap is deleted merely for computational convenience. A deliberately selected or truncated family of simplices will instead be called a finite comparison complex in the sense of Definition~\ref{def:finite-comparison-complex}.
\end{definition}

\begin{definition}
\label{def:finite-comparison-complex}
In a computational model-merging problem, one may instead specify a finite simplicial complex $K$ directly. Its vertices index local models or checkpoints, an edge $(ij)$ indicates that models $M_i$ and $M_j$ are compared on prescribed data $D_{ij}$, and a face $(ijk)$ indicates that the three pairwise alignments can be composed and evaluated on a common comparison
context $D_{ijk}$.

We call such a $K$ a \emph{finite comparison complex}. It need not be the nerve of a cover in a Grothendieck topology. We call $K$ a \emph{cover nerve} only when there is a specified cover $\mathcal U=\{U_i\to X\}$ for which $K=N(\mathcal U)$ and the comparison contexts $D_{i_0\cdots i_p}$ arise from the corresponding overlaps $U_{i_0\cdots i_p}$.
\end{definition}

\begin{definition}\label{def:finite-descent-instance}
A \emph{finite model-merging descent instance} is a tuple
\[
\mathfrak D=\left(K,\{M_i\}_{i\in K_0},\{D_\sigma\}_{\sigma\in K},\{g_{ij}\}_{(ij)\in K_1},A\right),
\]
where:
\begin{enumerate}
\item $K$ is a finite comparison complex;
\item $M_i$ is the local model attached to the vertex $i$;
\item $D_\sigma$ is the common comparison context attached to a simplex
$\sigma\in K$;
\item $g_{ij}:M_i|_{D_{ij}}\longrightarrow M_j|_{D_{ij}}$ is an alignment or transition map attached to the oriented edge $(ij)$;
\item $A$ is a coefficient system in which central residuals are recorded.
\end{enumerate}

Throughout the paper, products of transition maps are written in \emph{path order}: $g_{ij}g_{jk}$ means first apply $g_{ij}$ and then apply $g_{jk}$, equivalently $g_{jk}\circ g_{ij}$ in the conventional right-to-left notation for composition. Matrix formulas may be read either as path-ordered symbols or in row-vector coordinates. In the exact setting, we impose the inverse convention $g_{ji}=g_{ij}^{-1}$. For an oriented face $(ijk)$, the corresponding triangle residual is $h_{ijk}=g_{ij}g_{jk}g_{ki}$. When $K=N(\mathcal U)$ and the objects $M_i$, comparison contexts $D_\sigma$, and maps $g_{ij}$ arise by restriction from a sheaf or stack on $\sC_{\mathrm{learn}}$, the finite descent instance is the finite restriction of an ordinary \v{C}ech descent problem. For a general finite comparison complex $K$, the tuple $\mathfrak D$ is a combinatorial diagnostic model. It is not, by itself, a descent datum on a learning site, and a class in $H^2(K;A)$ is not automatically a class in the cohomology of $\sC_{\mathrm{learn}}$.
\end{definition}

\begin{remark}[Three levels of cohomology]
\label{rem:three-levels-cohomology}
The following objects should be distinguished:
\[
H^2(K;A),\qquad\check H^2(\mathcal U;A),\qquad H^2(\sC_{\mathrm{learn}};A).
\]
The first is the simplicial cohomology of a finite comparison complex. The second is the \v{C}ech cohomology of a specified cover. The third is the cohomology of the full learning site.

When $K=N(\mathcal U)$, a finite central cocycle may be interpreted as a class in $\check H^2(\mathcal U;A)$. It may be identified with a class in $H^2(\sC_{\mathrm{learn}};A)$ only when an appropriate comparison theorem applies, for example when the chosen cover is sufficiently acyclic for the coefficient system $A$. Without such a hypothesis, the obstruction computed by TwistedMerge is a class or residual on the chosen finite complex only.
\end{remark}

\begin{remark}[Scope of a finite obstruction certificate]
A certificate is relative to the permitted corrections and to the chosen comparison contexts. A nonzero class in $H^2(K;A)$ means that the measured central residual cannot be removed by multiplying the transition maps by an $A$-valued $1$-cochain on the edges of $K$. In particular, a different gauge group, a refined or enlarged comparison complex, a nonlinear retraining procedure, a projective representation, or an invariant branch readout may change the computational problem. For this reason, the implementation records the three statuses \emph{trivial}, \emph{nontrivial on $K$}, and \emph{uncertified}. The last status is used whenever centrality, closure, coefficient identification, or the statistical margin is not established.
\end{remark}

\begin{remark}
\label{rem:full-simplex-sanity}
Suppose that $K$ contains every subset of its vertex set as a simplex. Then $K$ is a full simplex and is contractible. Consequently, for a constant abelian coefficient group $A$,
\[
H^q(K;A)=0\qquad \text{for every }q>0.
\]
Thus a nonzero finite $H^2$ certificate requires nontrivial incidence data: some higher overlaps must be absent, unavailable, or deliberately excluded from the comparison complex.

In particular, the tetrahedral benchmark used below takes $K=\partial\Delta^3$, the boundary of a tetrahedron, rather than the full simplex $\Delta^3$. The missing $3$-simplex represents the absence of a fourfold compatibility datum in that controlled construction. We use this example as an explicit finite obstruction witness. The tetrahedral boundary is part of the controlled incidence structure of the synthetic witness.
\end{remark}

\begin{definition}[Frozen construction of the comparison complex]\label{def:frozen-comparison-complex}
Let the available data be separated into
\[
D_K,\qquad D_{\mathrm{align}},\qquad D_{\mathrm{cert}},\qquad D_{\mathrm{select}},\qquad D_{\mathrm{test}}.
\]
The complex-construction data $D_K$ contain only information used to decide whether a common comparison context is available, such as sample overlap, task overlap, communication availability, or a predeclared reliability score for composing restrictions. For every nonempty set of vertices $\sigma$, fix an availability score $q_\sigma(D_K)$ and a threshold $\lambda_{\dim\sigma}$ before fitting transition maps or inspecting residuals. Define
\[
\begin{aligned}
\sigma\in K
\quad\Longleftrightarrow\quad
&q_\tau(D_K)\geq\lambda_{\dim\tau}\text{ for every nonempty }\tau\subseteq\sigma,\\
&\text{and a common comparison context }D_\tau\\
&\text{is available for each such }\tau.
\end{aligned}
\]
This downward-closed rule produces a simplicial complex. The rule, thresholds, and data split are frozen before $g_{ij}$, $h_{ijk}$, $\Def(\widehat c)$, validation performance, or test performance are examined. Absence of a simplex means that the corresponding common context was not certified as available by this rule.
\end{definition}

\begin{theorem}[Frozen-complex three-way error control]\label{thm:frozen-complex-certificate}
Let $K=\Phi(D_K)$ be constructed by Definition~\ref{def:frozen-comparison-complex}, and assume that $D_K$ is independent of the certification data. Conditional on $K$, let $c\in Z^2(K;V)$ and suppose that
\[
\Pr\bigl(\|\widehat c-c\|\leq\varepsilon\mid K\bigr)\geq 1-\delta.
\]
Set
\[
L=\max\{0,\Def(\widehat c)-\varepsilon\},
\qquad
U=\Def(\widehat c)+\varepsilon,
\]
and fix $0\leq a<b$. After centrality, coefficient identification, and closure have passed, return
\[
\begin{aligned}
&\text{trivial if }U\leq a,\\
&\text{nontrivial on }K\text{ if }L\geq b,\\
&\text{uncertified otherwise}.
\end{aligned}
\]
On the stated event, a class with $\Def(c)>a$ is never declared trivial, and a class with $\Def(c)<b$ is never declared nontrivial. Moreover, every class with $\Def(c)\leq a-2\varepsilon$ is declared trivial, while every class with $\Def(c)\geq b+2\varepsilon$ is declared nontrivial. The same probability bound holds unconditionally because $K$ was selected independently of $D_{\mathrm{cert}}$.
\end{theorem}

\begin{proof}
By the $1$-Lipschitz property of distance to the fixed subspace $B^2(K;V)$,
\[
|\Def(\widehat c)-\Def(c)|\leq\varepsilon.
\]
Hence $L\leq\Def(c)\leq U$. The two no-false-decision statements follow immediately. If $\Def(c)\leq a-2\varepsilon$, then $U\leq\Def(c)+2\varepsilon\leq a$. If $\Def(c)\geq b+2\varepsilon$, then $L\geq\Def(c)-2\varepsilon\geq b$. Conditioning on the independently constructed $K$ and then averaging over $K$ preserves the probability bound.
\end{proof}

\begin{theorem}[Simultaneous control over a predeclared complex family]\label{thm:complex-family-certificate}
Let $\Lambda$ be finite, and let $K_\lambda=\Phi_\lambda(D_K)$, $\lambda\in\Lambda$, be a predeclared family of comparison complexes constructed without access to $D_{\mathrm{cert}}$. For each $\lambda$, let $c_\lambda\in Z^2(K_\lambda;V)$ and suppose that, conditional on $D_K$,
\[
\Pr\!\left(\|\widehat c_\lambda-c_\lambda\|\leq\varepsilon_\lambda\mid D_K\right)\geq 1-\delta_\lambda.
\]
Define
\[
L_\lambda=\max\{0,\Def(\widehat c_\lambda)-\varepsilon_\lambda\},
\qquad
U_\lambda=\Def(\widehat c_\lambda)+\varepsilon_\lambda.
\]
Then, with probability at least $1-\sum_{\lambda\in\Lambda}\delta_\lambda$, all intervals
\[
[L_\lambda,U_\lambda]
\]
simultaneously contain $\Def(c_\lambda)$. On this event, any complex selected after inspecting the certification data still satisfies the no-false-trivial and no-false-nontrivial conclusions of Theorem~\ref{thm:frozen-complex-certificate}, provided that its decision uses the simultaneous interval. Moreover, a declaration that a class is nontrivial throughout a predeclared subfamily $\Lambda_0\subseteq\Lambda$ implies $\Def(c_\lambda)\geq b_\lambda$ for every $\lambda\in\Lambda_0$, while a declaration of triviality throughout $\Lambda_0$ implies $\Def(c_\lambda)\leq a_\lambda$ for every $\lambda\in\Lambda_0$.
\end{theorem}

\begin{proof}
For each $\lambda$, the distance-to-coboundaries map is $1$-Lipschitz, so the stated estimation event implies $L_\lambda\leq\Def(c_\lambda)\leq U_\lambda$. A union bound gives simultaneous coverage for the finite family. Every subsequent choice of $\lambda$, including a data-dependent one, is then made on an event on which all intervals are valid. The final two statements follow by applying the lower or upper bound to every member of $\Lambda_0$.
\end{proof}

\begin{proposition}[Refinement persistence and sensitivity]\label{prop:complex-refinement}
Let $i:K\hookrightarrow L$ be an inclusion of finite simplicial complexes and let $A$ be an abelian coefficient group.
\begin{enumerate}
\item If $c_L\in Z^2(L;A)$, then $i^*c_L\in Z^2(K;A)$ and $i^*[c_L]=[i^*c_L]$. Consequently, if the restricted class is nonzero on $K$, then $[c_L]$ is nonzero on $L$.
\item A class on $K$ need not extend to $L$. In particular, the one-negative-face $\mu_2$ cocycle on $\partial\Delta^3$ does not extend to a cocycle on $\Delta^3$, because its alternating boundary product is $-1$.
\item Removing one face from $\partial\Delta^3$ produces a $2$-ball and kills $H^2$ with constant coefficients, while barycentric subdivision preserves the class under the canonical cohomology isomorphism.
\end{enumerate}
Thus a certificate should be reported together with a predeclared filtration or sensitivity family of comparison complexes. Persistence across admissible refinements is stronger evidence than nontriviality for one post hoc complex.
\end{proposition}

\begin{proof}
The restriction map is a cochain map and therefore induces the stated map on cohomology. If $[c_L]=0$, then its restriction is zero, proving the contrapositive of the first claim. The second claim is the tetrahedral closure calculation. The third follows from contractibility of a triangulated $2$-ball and homotopy invariance under subdivision.
\end{proof}

\begin{table*}[!tbp]
\centering
\caption{Exact sensitivity of the controlled tetrahedral certificate to the comparison complex. This is a topology audit, not a natural-data experiment.}
\label{tab:tetrahedral-complex-sensitivity}
\small
\begin{tabular}{L{0.23\textwidth}L{0.24\textwidth}L{0.42\textwidth}}
\toprule
Complex & Status of the one-negative-face data & Conclusion \\
\midrule
$\partial\Delta^3$ & closed and non-coboundary & nontrivial class in $H^2(\partial\Delta^3;\mu_2)$ \\
$\Delta^3$ & boundary assignment fails $3$-simplex closure & no extension as a cocycle on the filled simplex \\
$\partial\Delta^3$ with one face removed & restriction lies on a $2$-ball & $H^2=0$; no higher certificate remains \\
Barycentric subdivision of $\partial\Delta^3$ & pullback cocycle & class persists under subdivision \\
\bottomrule
\end{tabular}
\end{table*}

\begin{definition}
A \emph{sheaf of local models} on $\sC_{\mathrm{learn}}$ is a sheaf of sets whose elements over $U$ are local learned objects and whose restriction maps satisfy the sheaf condition. More generally, a \emph{stack of local models} $\sM$ is a category fibered in groupoids satisfying descent; we write $\sM(U)$ for its groupoid of objects over $U$. A stack is the natural ambient structure when local models have nontrivial automorphisms, such as hidden-unit permutations, ReLU scalings, or other gauge symmetries.

Let $\mathcal U=\{U_i\to X\}$ be a cover. A \emph{descent datum} consists of local objects $M_i\in \sM(U_i)$ and transition isomorphisms $g_{ij}:M_i|_{U_{ij}}\longrightarrow M_j|_{U_{ij}}$ satisfying $g_{ji}=g_{ij}^{-1}$. Strict descent requires the cocycle condition $g_{ij}g_{jk}g_{ki}=1$ on every triple overlap $U_{ijk}$. If descent is effective for $\sM$ over $\mathcal U$, such a datum glues to an ordinary global object in $\sM(X)$.

More generally, the triple product may fail by a defect $g_{ij}g_{jk}g_{ki}=c_{ijk}$, where $c_{ijk}\in A(U_{ijk})$ for a coefficient sheaf $A$. If the $c_{ijk}$ form a \v{C}ech 2-cocycle on the cover $\mathcal U$, we write $[c]_{\mathcal U}\in\check H^2(\mathcal U;A)$ for its fixed-cover class. If this class determines a class in the cohomology of the learning site through the relevant comparison map, we denote the resulting class by $[c]\in H^2(\sC_{\mathrm{learn}};A)$. A system of local objects whose gluing is controlled by this class is called a \emph{$[c]$-twisted descent datum}. When the coefficient system and the ambient stack support the standard notion of twisted sheaves, we also call it a \emph{$[c]$-twisted sheaf}. If $A$ is an abelian coefficient sheaf lying in the center of the relevant gauge groups, then we call $[c]$ a central obstruction class. In particular, cyclic coefficient groups such as $\mu_n$ give the basic examples of central obstructions.
\end{definition}

\begin{ex}
    Below are several examples arising in learning theory.
    \begin{itemize}
        \item \emph{Model merging.} The objects $U_i$ are training runs, clients, tasks, or local model charts. The objects $M_i$ are checkpoints or local models, and the transition maps $g_{ij}$ are alignments, such as permutation, orthogonal, diagonal, or monomial alignments. A nontrivial triple product $g_{ij}g_{jk}g_{ki}\neq 1$ is a cycle-consistency or holonomy defect.
        \item \emph{Same-base task vectors and fixed-chart methods.} Suppose that we have chosen a common chart $V\to X$ of the global learning space. A base model $M_0$ is a reference section over this chart. If every local model $M_i$ is represented in the same chart, then one can form the task vector $\Delta_i=M_i-M_0$. Thus Task Arithmetic, TIES, and DARE work after a choice of global trivialization: the transition maps have effectively been fixed in advance. SLERP is similar, except that it uses the chosen chart as a weight geometry in which interpolation is defined. Without a common chart $V\to X$, the expressions $\Delta_i$ and the interpolation paths are not intrinsic.
        \item \emph{Projective latent representations.} If local latent coordinates are comparable only projectively, then the transition maps live in $\PGL_r$. Lifting them to $\GL_r$ may fail by a $\GG_m$-valued $2$-cocycle. This is the obstruction measured by the boundary map arising from
\[
1\to \GG_m\to \GL_r\to \PGL_r\to 1.
\]
        Thus a projective system of local latent representations may define a projective descent datum whose obstruction lies in $H^2(\sC_{\mathrm{learn}};\GG_m)$. The cohomological Brauer group is the torsion subgroup $\Br'(\sC_{\mathrm{learn}}):=H^2(\sC_{\mathrm{learn}};\GG_m)_{\mathrm{tors}}$. The Azumaya Brauer group $\Br(\sC_{\mathrm{learn}})$ maps naturally to $\Br'(\sC_{\mathrm{learn}})$; we identify them only in settings where this map is known to be an isomorphism.
    \end{itemize}
\end{ex}

\begin{remark}[Finite central classes and Brauer classes]
Let $X$ be a scheme, or an algebraic stack for which the Kummer sequence is exact on the chosen \'{e}tale site, and suppose that $n$ is invertible on $X$. Then the Kummer sequence gives an exact sequence
\[
0\longrightarrow \Pic(X)/n\longrightarrow H^2_{\et}(X;\mu_n)\longrightarrow \Br'(X)[n]\longrightarrow 0.
\]
Consequently, a nonzero class in $H^2_{\et}(X;\mu_n)$ need not have nonzero image in $H^2_{\et}(X;\GG_m)$. It may lie in the image of $\Pic(X)/n$. In the finite computational setting, a $\mu_n$-valued certificate should therefore first be called a finite central obstruction on the chosen comparison complex. We use the term Brauer class only after a specified site-level comparison and a nonzero torsion image in $H^2_{\et}(X;\GG_m)$, equivalently a nonzero class in $\Br'(X)$, have been established.
\end{remark}

\begin{theorem}[Descent obstruction]\label{thm:descent-obstruction}
Let $\sC_{\mathrm{learn}}$ be a learning site, and let $\sM$ be a stack of local models with effective strict descent for a cover $\mathcal U=\{U_i\to X\}$. Let $G$ be the sheaf of gauge groups acting on the local objects, and let $A$ be an abelian coefficient sheaf lying in the center of $G$. Suppose that we are given local objects $M_i\in \sM(U_i)$ and transition isomorphisms $g_{ij}:M_i|_{U_{ij}}\longrightarrow M_j|_{U_{ij}}$ such that, on triple overlaps, $g_{ij}g_{jk}g_{ki}=c_{ijk}$ for some $c_{ijk}\in A(U_{ijk})$. Then $c=\{c_{ijk}\}$ is a \v{C}ech $2$-cocycle. Its cohomology class $[c]\in \check H^2(\mathcal U;A)$ is invariant under objectwise gauge changes and under changes of the edge representatives by $A$-valued $1$-cochains.

Moreover, if $[c]=0$, then the transition maps can be corrected by an $A$-valued edge correction so that they satisfy strict descent; hence, by effective descent, the local objects glue to an ordinary global object. If $[c]\neq 0$, then no such edge correction can make the datum into strict descent. Thus the data remain nontrivial as twisted descent data on the chosen cover. If the induced class in the cohomology of the learning site is nonzero, the same conclusion holds at the site level. When the ambient theory admits effective descent for $[c]$-twisted objects, these data define a nontrivial twisted object rather than an ordinary global object.
\end{theorem}

\begin{proof}
We omit restriction symbols from the notation. Since $A$ lies in the center of the gauge group, the defect terms $c_{ijk}$ commute with the transition maps. The relation $g_{ij}g_{jk}g_{ki}=c_{ijk}$ is equivalently $g_{ij}g_{jk}=c_{ijk}g_{ik}$. On a quadruple overlap $U_{ijkl}$, compare the product $g_{ij}g_{jk}g_{kl}$ in two ways. First,
\[
(g_{ij}g_{jk})g_{kl}=c_{ijk}g_{ik}g_{kl}=c_{ijk}c_{ikl}g_{il}.
\]
Second,
\[
g_{ij}(g_{jk}g_{kl})=g_{ij}c_{jkl}g_{jl}=c_{jkl}g_{ij}g_{jl}=c_{jkl}c_{ijl}g_{il}.
\]
So $c_{ijk}c_{ikl}=c_{jkl}c_{ijl}$. This is exactly the \v{C}ech $2$-cocycle condition for the multiplicative cochain $c=\{c_{ijk}\}$, up to the standard convention for ordering the indices. Thus $c$ defines a class $[c]\in \check H^2(\mathcal U;A)$.

We next check gauge invariance. A change of local gauges replaces the transition maps by conjugate transition maps. Since $A$ is central, the corresponding defect terms are unchanged. More generally, if we modify the transition maps by an $A$-valued $1$-cochain $a=\{a_{ij}\}$, setting $g'_{ij}=a_{ij}g_{ij}$, then the new defect is
\[
c'_{ijk}=a_{ij}a_{jk}a_{ki}c_{ijk}=(\delta a)_{ijk}c_{ijk}.
\]
Thus the cocycle changes only by a \v{C}ech coboundary, so the class $[c]$ is well-defined.

If $[c]=0$, then there exists an $A$-valued $1$-cochain $a=\{a_{ij}\}$ such that $(\delta a)_{ijk}=c_{ijk}^{-1}$. For the corrected transition maps $g'_{ij}=a_{ij}g_{ij}$, the new defect is $c'_{ijk}=(\delta a)_{ijk}c_{ijk}=1$. Hence the corrected transition maps satisfy strict descent. Since strict descent is effective for $\sM$ over the cover, the local objects glue to an ordinary global object.

Conversely, if some $A$-valued edge correction made the transition maps satisfy strict descent, then the new defect would be trivial. The formula above would then imply that $c$ differs from $1$ by a \v{C}ech coboundary, so $[c]=0$. Therefore, if $[c]\neq 0$, no such edge correction is possible. The datum remains nontrivial as twisted descent data on the chosen cover. If the induced class in $H^2(\sC_{\mathrm{learn}};A)$ is nonzero, the same obstruction persists at the site level. Existence of a corresponding twisted object is a separate effectivity statement in the ambient theory.
\end{proof}

\begin{remark}
Under the centrality assumption in Theorem~\ref{thm:descent-obstruction}, the defect terms take values in an abelian coefficient sheaf $A$, or more generally in a central subsheaf of the gauge group. They form an ordinary \v{C}ech $2$-cocycle on the chosen cover and define a class $[c]_{\mathcal U}\in\check H^2(\mathcal U;A)$. When the corresponding comparison map is available, this fixed-cover class determines a class $[c]\in H^2(\sC_{\mathrm{learn}};A)$. This is the situation modeled by the controlled central and projective examples, including the $\mu_n$- and $\GG_m$-valued cases.

For a genuinely nonabelian gauge group $G$, the same formula should not be read as an ordinary abelian cohomology class. If $g_{ij}g_{jk}=c_{ijk}g_{ik}$, then on quadruple overlaps the terms $c_{ijk}$ satisfy a twisted nonabelian compatibility condition, involving conjugation by the transition maps:
\[
c_{ijk}c_{ikl}=g_{ij}c_{jkl}g_{ij}^{-1}c_{ijl}.
\]
Thus the obstruction is a nonabelian descent or holonomy object rather than an element of ordinary $H^2(\sC_{\mathrm{learn}};G)$.

Equivalently, in the usual language of nonabelian cohomology and gerbes, one does not only remember a $G$-valued $2$-cocycle. One also has a \emph{band}, or outer automorphism datum, recording how the local copies of $G$ are identified up to inner automorphism. This is why outer automorphism groups naturally appear in nonabelian descent. Once the band is fixed, the remaining central ambiguity is controlled by the center $Z(G)$, and central obstruction classes live in ordinary abelian cohomology with coefficients in this center or in the corresponding central coefficient sheaf. For this theory, see \cite{Giraud_71}.

In our terminology, the central/projective experiments are modeled by the ordinary $H^2$ theory above. By contrast, the executed two-loop $S_3$ and $D_4$ experiment should be read as a nonabelian holonomy experiment. The branch construction does not kill an abelian $H^2$ class; rather, it represents the nontrivial holonomy on branches and then applies invariant pooling. This is the condition $P\rho(h)=P$, which makes the holonomy invisible after pooling even though it remains nontrivial before pooling. In the executed construction, this is a structural certificate and not an accuracy-advantage result.
\end{remark}

\subsection{The computational aspect of the theory}

In this subsection, we discuss the computational implementation of the descent-theoretic framework. The goal is to explain how the abstract objects above become concrete procedures for model merging. In particular, we describe greedy soup as fixed-chart search over a finite family of candidate predictors, analogous to empirical global-section search; task-vector methods as fixed-chart descent; ReLU gauge alignment as a monomial-gauge synchronization problem; the $H^2(\mu_2)$ certificate as a central obstruction test; and the period-index mechanism as an arithmetic gate for admissible projective ranks in the controlled systems.

Let $K$ be either the nerve of a specified finite cover or a finite comparison complex in the sense of Definition~\ref{def:finite-comparison-complex}.

\begin{definition}
\label{def:triangle-residual}
Let
\[
\mathfrak D=\left(K,\{M_i\}_{i\in K_0},\{D_\sigma\}_{\sigma\in K},\{g_{ij}\}_{(ij)\in K_1}, A \right)
\]
be a finite model-merging descent instance, and let $G$ be the structure group containing the transition maps. We assume that every unoriented edge is given one orientation and that $g_{ji}=g_{ij}^{-1}$ in the exact setting.

For an oriented face $(ijk)\in K_2$, define the \emph{triangle residual}
\[
h_{ijk}=g_{ij}g_{jk}g_{ki}=g_{ij}g_{jk}g_{ik}^{-1}.
\]
The element $h_{ijk}\in G$ measures the failure of the three pairwise transition maps to satisfy strict descent on the face $(ijk)$. The collection $h=\{h_{ijk}\}_{(ijk)\in K_2}$ is initially a collection of $G$-valued triangle holonomies.
\end{definition}

\begin{definition}
\label{def:central-residual-cochain}
Let $A\subseteq Z(G)$ be an abelian central subgroup. We say that the triangle residual is \emph{exactly $A$-central} if $h_{ijk}\in A$ for every oriented face $(ijk)\in K_2$.

In this case, the assignment $c_{ijk}=h_{ijk}$ defines a multiplicative 2-cochain $c\in C^2(K;A)$. We call $c$ the \emph{central residual 2-cochain}. If the residuals are only approximately central, and if a projection or rounding map $\pi_A:G\longrightarrow A$ has been fixed, then one may define a projected 2-cochain $\widehat c_{ijk}=\pi_A(h_{ijk})$. The projected cochain $\widehat c$ is a numerical diagnostic.
\end{definition}

\begin{definition}
\label{def:cocycle-closure}
Assume that $A$ is a constant abelian coefficient group. For a multiplicative 2-cochain $c\in C^2(K;A)$, its coboundary is the 3-cochain $\delta_2 c\in C^3(K;A)$ given on an oriented 3-simplex $(ijkl)$ by
\[
(\delta_2c)_{ijkl}=c_{jkl}c_{ikl}^{-1}c_{ijl}c_{ijk}^{-1}.
\]
The cochain $c$ is a 2-cocycle if and only if $(\delta_2c)_{ijkl}=1_A$ for every oriented 3-simplex $(ijkl)\in K_3$.

Thus
\[
Z^2(K;A)=\ker\left(\delta_2:C^2(K;A)\longrightarrow C^3(K;A) \right).
\]
Only when $c\in Z^2(K;A)$ does it determine a cohomology class
\[
[c]\in H^2(K;A)=Z^2(K;A)/B^2(K;A).
\]
\end{definition}

\begin{remark}
\label{rem:coefficient-systems-closure}
Definition~\ref{def:cocycle-closure} is written for a constant abelian coefficient group $A$. If $A$ is a nonconstant coefficient system, local system, or sheaf of central groups, the formula for $\delta_2$ must include the appropriate restriction or transport maps.
\end{remark}

\begin{proposition}
\label{prop:central-residual-closed}
Let $(ijkl)\in K_3$ be an oriented 3-simplex. Assume that:
\begin{enumerate}
\item the maps $g_{ab}$ are all defined on a common quadruple comparison
context $D_{ijkl}$;
\item $g_{ba}=g_{ab}^{-1}$;
\item for every oriented face $(abc)$ of $(ijkl)$, the triangle residual $c_{abc}=g_{ab}g_{bc}g_{ac}^{-1}$ lies in the central subgroup $A\subseteq Z(G)$.
\end{enumerate}
Then
\[
(\delta_2c)_{ijkl}=c_{jkl}c_{ikl}^{-1}c_{ijl}c_{ijk}^{-1}=1_A.
\]
Consequently, exact central triangle residuals obtained from a coherent
collection of transition maps on common higher overlaps define a
2-cocycle.
\end{proposition}

\begin{proof}
Using $g_{ij}g_{jk}=c_{ijk}g_{ik}$ and $g_{ik}g_{kl}=c_{ikl}g_{il}$, we obtain
\[
(g_{ij}g_{jk})g_{kl}=c_{ijk}c_{ikl}g_{il}.
\]
On the other hand, $g_{jk}g_{kl}=c_{jkl}g_{jl}$, and hence
\[
g_{ij}(g_{jk}g_{kl})= g_{ij}c_{jkl}g_{jl}.
\]
Since $c_{jkl}$ is central, this equals
\[
c_{jkl}g_{ij}g_{jl}=c_{jkl}c_{ijl}g_{il}.
\]
Associativity gives $c_{ijk}c_{ikl}=c_{jkl}c_{ijl}$. Since all four terms lie in the abelian group $A$, this is equivalent to $c_{jkl}c_{ikl}^{-1}c_{ijl} c_{ijk}^{-1}=1_A$.
\end{proof}

\begin{proposition}\label{prop:constant-edge-tetrahedron}
Let $K=\partial\Delta^3$, let $G$ be a group, and let $A\subseteq Z(G)$ be an abelian central subgroup. Assign one context-independent element $g_{ij}\in G$ to every oriented edge, with $g_{ji}=g_{ij}^{-1}$. Suppose that
\[
h_{ijk}=g_{ij}g_{jk}g_{ki}\in A
\]
for all four faces. Then
\[
h_{123}h_{023}^{-1}h_{013}h_{012}^{-1}=1.
\]
Consequently, for constant coefficients $A$, the resulting face cochain represents the zero class in $H^2(\partial\Delta^3;A)$.
\end{proposition}

\begin{proof}
From the definitions,
\[
g_{01}g_{12}=h_{012}g_{02},\qquad g_{02}g_{23}=h_{023}g_{03},
\]
so
\[
(g_{01}g_{12})g_{23}=h_{012}h_{023}g_{03}.
\]
On the other hand,
\[
g_{12}g_{23}=h_{123}g_{13},\qquad g_{01}g_{13}=h_{013}g_{03},
\]
and centrality of $h_{123}$ gives
\[
g_{01}(g_{12}g_{23})=h_{123}h_{013}g_{03}.
\]
Associativity therefore implies $h_{012}h_{023}=h_{123}h_{013}$, which is equivalent to the stated alternating product identity. This identity is the evaluation of the face cochain on the fundamental class of $\partial\Delta^3$, so the class is zero.
\end{proof}

\begin{remark}
The one-negative-face $\mu_2$ cocycle used below is a cocycle-level oracle witness. It cannot be produced by six globally composable, context-independent edge matrices satisfying the hypotheses of Proposition~\ref{prop:constant-edge-tetrahedron}. A genuine \v{C}ech realization requires overlap-dependent transition sections $g_{ij}:D_{ij}\to G$, whose restrictions to different triple contexts need not be represented by one common matrix, or a projective bundle construction in which the missing fourfold comparison context is part of the geometry.
\end{remark}

\begin{remark}\label{rem:closure-not-automatic}
Proposition~\ref{prop:central-residual-closed} applies to exact transition maps defined coherently on common higher overlaps. Its hypotheses need not hold in a numerical model-merging problem.

Closure may fail, or may be uncertified, when: \begin{enumerate}
\item the opposite edge maps are fitted independently and $g_{ji}\neq g_{ij}^{-1}$;
\item the triangle residuals are not central;
\item a nearest-central projection $\pi_A$ is applied to noncentral residuals;
\item different faces are evaluated on different comparison data and do not restrict to a common quadruple context;
\item the transition maps are known only approximately;
\item the coefficient identifications vary from face to face.
\end{enumerate}

For these reasons, the output of a numerical central projection should first be called a projected 2-cochain. It should be called a cocycle, obstruction class, or Brauer-type class only after a cocycle-closure test has been passed.
\end{remark}

\begin{definition}[Distance to coboundaries and obstruction norm]
\label{def:closed-cochain-obstruction}
Let $V$ be a finite-dimensional normed vector space in which the central coefficient data have been encoded additively. Consider the cochain complex
\[
C^1(K;V)\xrightarrow{\delta_1}C^2(K;V) \xrightarrow{\delta_2}C^3(K;V),
\]
and write
\[
Z^2(K;V)=\ker\delta_2, \qquad B^2(K;V)=\operatorname{im}\delta_1.
\]

For an arbitrary 2-cochain $u\in C^2(K;V)$, define its distance to
coboundaries by
\[
\Def(u):=\operatorname{dist}\left(u,B^2(K;V)\right)=\inf_{a\in C^1(K;V)}\left\lVert u-\delta_1a\right\rVert.
\]
Thus, for every $a\in C^1(K;V)$,
\[
\left\lVert u-\delta_1a\right\rVert\geq\Def(u).
\]

If $u$ is closed, so that $u\in Z^2(K;V)$, then $\Def(u)$ depends only on the cohomology class
\[
[u]\in H^2(K;V)=Z^2(K;V)/B^2(K;V).
\]
In this case, we may also write $\Obs([u]):= \Def(u)$ and call this quantity the obstruction norm of the class $[u]$. Since the cochain spaces are finite-dimensional,
\[
\Obs([u])=0\quad\Longleftrightarrow\quad[u]=0.
\]

If $u$ is not known to be closed, then $\Def(u)$ is only a distance-to-coboundaries diagnostic. It is not the norm of a cohomology class, because $[u]$ is not defined unless $\delta_2u=0$.
\end{definition}

We record three elementary facts that will be used below.

\begin{proposition}[Defect distance and stability]\label{prop:defect-distance}
Let $K$ be finite.
\begin{enumerate}
    \item Suppose that $V$ is a finite-dimensional inner-product space and that $C^\bullet(K;V)$ is equipped with the induced finite cochain inner product. Let
\[
\mathcal H^2(K;V)=\ker\delta_2\cap\ker\delta_1^*
\]
    be the space of harmonic 2-cochains. If $c\in Z^2(K;V)$, then
\[
\Def(c)=\Obs([c])=\left\lVert\Pi_{\mathcal H^2(K;V)}(c)\right\rVert,
\]
    where $\Pi_{\mathcal H^2(K;V)}$ denotes orthogonal projection onto $\mathcal H^2(K;V)$.

    \item If $c,\widehat c\in C^2(K;V)$ and $\|\widehat c-c\|\leq \varepsilon$, then
\[
|\Def(\widehat c)-\Def(c)|\leq \varepsilon.
\]
    In particular, if $\Def(c)>2\varepsilon$, then $\Def(\widehat c)>\varepsilon$. Hence the observed residual cannot be explained by an $\varepsilon$-sized perturbation of a trivial coboundary.

    \item Let $g_{ij},g_{jk},g_{ki}$ and $\widehat g_{ij},\widehat g_{jk},\widehat g_{ki}$ be matrices satisfying
\[
\|g_{ab}\|\leq M,\qquad \|\widehat g_{ab}-g_{ab}\|\leq \sigma
\]
    for the three edges $(ab)=(ij),(jk),(ki)$. Then
\[
\|\widehat g_{ij}\widehat g_{jk}\widehat g_{ki}-g_{ij}g_{jk}g_{ki}\|\leq 3M^2\sigma+3M\sigma^2+\sigma^3.
\]
\end{enumerate}
\end{proposition}

\begin{proof}
For part (1), we use the standard finite cochain Hodge decomposition for
simplicial complexes; see, for example, \cite{Dodziuk1976Finite,HorakJost2013Spectra}. Since $K$ is
finite and $V$ is finite-dimensional, all cochain spaces are
finite-dimensional inner-product spaces, and one has the orthogonal
decomposition
\[
C^2(K;V)=\operatorname{im}\delta_1\oplus\mathcal H^2(K;V)\oplus\operatorname{im}\delta_2^*.
\]
Because $c$ is closed, one has
\[
c\in\ker\delta_2=\left(\operatorname{im}\delta_2^*\right)^\perp.
\]
Therefore the $\operatorname{im}\delta_2^*$ component of $c$ vanishes, and there exist $a\in C^1(K;V)$ and $h\in\mathcal H^2(K;V)$ such that $c=\delta_1a+h$. The two terms are orthogonal. Hence the closest element of $\operatorname{im}\delta_1$ to $c$ is $\delta_1a$, and therefore
\[
\Def(c)=\operatorname{dist}\left(c,\operatorname{im}\delta_1\right)=\lVert h\rVert.
\]
Since $h=\Pi_{\mathcal H^2(K;V)}(c)$, we obtain
\[
\Def(c)=\Obs([c])=\left\lVert\Pi_{\mathcal H^2(K;V)}(c)\right\rVert.
\]

For (2), the distance to any fixed subset of a normed vector space is
$1$-Lipschitz \cite{BauschkeCombettes2017}. Indeed, for any $a\in C^1(K;V)$,
\[
\|\widehat c-\delta_1 a\|\leq\|\widehat c-c\|+\|c-\delta_1 a\|.
\]
Taking the infimum over $a$ gives $\Def(\widehat c)\leq \varepsilon+\Def(c)$. Interchanging $c$ and $\widehat c$ gives $\Def(c)\leq \varepsilon+\Def(\widehat c)$. Hence
\[
|\Def(\widehat c)-\Def(c)|\leq \varepsilon.
\]
If $\Def(c)>2\varepsilon$, then $\Def(\widehat c)\geq \Def(c)-\varepsilon>\varepsilon$.

For (3), write
\[
\widehat g_{ab}=g_{ab}+e_{ab}, \qquad \|e_{ab}\|\leq \sigma.
\]
Then
\begin{align*}
\widehat g_{ij}\widehat g_{jk}\widehat g_{ki} - g_{ij}g_{jk}g_{ki} ={}&
e_{ij}g_{jk}g_{ki} + g_{ij}e_{jk}g_{ki} \\
&+ g_{ij}g_{jk}e_{ki} + e_{ij}e_{jk}g_{ki} \\
&+ e_{ij}g_{jk}e_{ki} + g_{ij}e_{jk}e_{ki} \\
&+ e_{ij}e_{jk}e_{ki}.
\end{align*}
Using submultiplicativity of the matrix norm \cite{HornJohnson2012Matrix}, the first three terms are each bounded by $M^2\sigma$, the next three terms are each bounded by $M\sigma^2$, and the last term is bounded by $\sigma^3$. So
\[
\|\widehat g_{ij}\widehat g_{jk}\widehat g_{ki} - g_{ij}g_{jk}g_{ki}\| \leq 3M^2\sigma+3M\sigma^2+\sigma^3.
\]
This proves the proposition.
\end{proof}

\begin{proposition}[Margin-consistent distance decision]\label{prop:margin-obstruction-decision}
Let $V$ be a finite-dimensional normed coefficient space, let $c\in C^2(K;V)$, and let $\widehat c$ satisfy
\[
\Pr\bigl(\|\widehat c-c\|\leq\varepsilon\bigr)\geq 1-\delta.
\]
Write $\Delta=\dist(c,B^2(K;V))$. Consider the decision rule that declares a positive distance from coboundaries when $\Def(\widehat c)>\tau$. If $c\in B^2(K;V)$ and $\tau>\varepsilon$, then the rule makes no false positive-distance declaration on the stated event. If $\Delta>0$ and
\[
\varepsilon<\tau<\Delta-\varepsilon,
\]
then the rule correctly detects positive distance from coboundaries on the same event. If, in addition, both the coefficient identification and cocycle closure of $\widehat c$ are certified, this positive-distance conclusion may be promoted to a nontrivial class in $H^2(K;V)$.
\end{proposition}

\begin{proof}
The distance to a fixed subspace is $1$-Lipschitz, so
\[
\bigl|\Def(\widehat c)-\Def(c)\bigr|\leq\varepsilon.
\]
If $c$ is a coboundary, then $\Def(c)=0$ and hence $\Def(\widehat c)\leq\varepsilon<\tau$. If $\Def(c)=\Delta>0$, then $\Def(\widehat c)\geq\Delta-\varepsilon>\tau$. The final cohomological interpretation follows only after the separate closure gate, because an arbitrary $2$-cochain does not define a class in $H^2$.
\end{proof}

\begin{remark}
Proposition~\ref{prop:margin-obstruction-decision} suggests a three-way numerical output. A closed class is declared trivial when an upper confidence bound lies below a triviality threshold, nontrivial when a lower confidence bound lies above a nontriviality threshold, and uncertified otherwise. Before closure has been certified, the same bounds refer only to distance from coboundaries. For a finite coefficient group $A$ embedded in a metric ambient space, let $s_A$ be the minimum ambient distance between distinct elements of $A$. If the exact coefficient is $a\in A$ and an estimate $u$ has ambient distance less than $s_A/2$ from $a$, then nearest-coefficient projection recovers $a$ uniquely. In a numerical pipeline, the projection should be accepted only when the uncertainty bound, not merely the point estimate, lies inside such a half-separation margin. The robust calibration experiments in Section~\ref{subsec:robust-central-certification} implement this conservative abstention principle.
\end{remark}

\begin{remark}
    Part~(2) is a metric stability statement for arbitrary 2-cochains and does not require cocycle closure. Its interpretation as stability of a cohomological obstruction requires the relevant cochains to be closed. Without closure, $\Def$ remains a distance-to-coboundaries diagnostic rather than the norm of an $H^2$ class.

Proposition~\ref{prop:defect-distance}(3) controls the perturbation of the triangle residual under small pairwise alignment errors. Suppose further that the central projection $\pi_A:G\longrightarrow A$ is $L$-Lipschitz on the relevant neighborhood. If the edgewise alignment error is at most $\sigma$ and the exact transition maps have norm at most $M$, then the projected triangle residual satisfies a bound of the form
\[
\|\widehat c-c\|\leq L\left(3M^2\sigma+3M\sigma^2+\sigma^3\right),
\]
with the corresponding normalization when the cochain norm aggregates over several faces. Combining this with Proposition~\ref{prop:defect-distance}(2) gives a metric robustness criterion. If both $c$ and $\widehat c$ are closed, the same estimate gives stability of the obstruction norms. Without closure, the conclusion concerns only the distance-to-coboundaries diagnostic.
\end{remark}

\subsubsection{Compilation of results for related algorithms}

We record three elementary computational facts that explain how several existing algorithms appear in the descent-theoretic language. The first concerns fixed-chart task-vector methods, the second concerns exact ReLU gauge symmetries, and the third gives the minimal central $H^2(\mu_2)$ obstruction certificate.

\begin{proposition}[Task vectors require a common chart]
Let $\Theta$ be a finite-dimensional parameter vector space with a nontrivial gauge group $G$ acting linearly on it. Let $\theta_0\in \Theta$ be a base model and let $\theta_i\in \Theta$ be a fine-tuned model. The task vector $\Delta_i=\theta_i-\theta_0$ is not an intrinsic object on the quotient parameter stack $[\Theta/G]$. Rather, it is defined only after choosing a common parameter coordinate system, or a specified alignment convention, in which both $\theta_i$ and $\theta_0$ are represented. Thus algebraic operations on task vectors are gauge-dependent unless a common base chart or an alignment convention has been fixed.
\end{proposition}

\begin{proof}
The subtraction map
\[
d:\Theta\times \Theta\to \Theta,\qquad d(\theta,\eta)=\theta-\eta
\]
is defined in the chosen parameter vector space. If the same gauge $a\in G$ is applied to both entries, then
\[
d(a\theta_i,a\theta_0)=a(\theta_i-\theta_0).
\]
Thus the difference is equivariant for the diagonal action of $G$ on
$\Theta\times\Theta$. This is the situation in which a common chart has been chosen: both the base model and the fine-tuned model are expressed in the same trivialization.

However, the quotient stack $[\Theta/G]$ allows the two representatives to be changed independently. The relevant ambiguity is the action of $G\times G$ on $\Theta\times\Theta$, given by
\[
(a_i,a_0)\cdot(\theta_i,\theta_0)=(a_i\theta_i,a_0\theta_0).
\]
The difference map $d$ is not invariant, or even equivariant with respect to a single copy of $G$, for this independent action. Indeed, independent gauges produce
\[
d(a_i\theta_i,a_0\theta_0)=a_i\theta_i-a_0\theta_0,
\]
and there is in general no element $a\in G$ such that
\[
a_i\theta_i-a_0\theta_0=a(\theta_i-\theta_0).
\]
Equivalently, the map $d$ does not descend from $\Theta\times\Theta$ to the quotient stack $[\Theta/G]\times[\Theta/G]$.

A simple way to see the obstruction is to take $\theta_i=\theta_0=\theta$. Then the task vector in the common chart is zero: $\theta-\theta=0$. But after changing only the gauge of the second representative, one obtains $\theta-a_0\theta$, which is generally nonzero. Since the zero vector is fixed by every linear gauge transformation, this new vector cannot be obtained from the old one by a single global gauge transformation unless $a_0\theta=\theta$. Thus the difference depends on the chosen representatives.

We then have that task vectors are well-defined only after fixing a common chart, such as a shared base model and a shared parameter coordinate system, or after choosing transition maps that align the local charts. Without such a choice, the expression $\theta_i-\theta_0$ is a chart-dependent representative rather than an intrinsic object on the quotient by gauge symmetries.
\end{proof}

\begin{remark}
This is the fixed-chart assumption behind Task Arithmetic, TIES, DARE, and other task-vector methods: the local models have already been placed in a shared coordinate chart. SLERP is similar in spirit, but uses a chosen weight geometry.
\end{remark}

\begin{theorem}[Exact monomial ReLU reparameterization]\label{thm:exact-monomial-relu}
Consider a one-hidden-layer ReLU network
\[
f(x)=W_2\sigma(W_1x+b_1)+b_2,
\]
where $\sigma=\operatorname{ReLU}$ is applied coordinatewise. Let $M=DP$ be a positive monomial matrix, where $D$ is a positive diagonal matrix and $P$ is a permutation matrix. Define
\[
W_1'=MW_1,\qquad b_1'=Mb_1,\qquad W_2'=W_2M^{-1},\qquad b_2'=b_2.
\]
Then the reparameterized network computes exactly the same function:
\[
W_2'\sigma(W_1'x+b_1')+b_2'=W_2\sigma(W_1x+b_1)+b_2.
\]
The same argument applies layerwise to adjacent layers of a multi-layer ReLU MLP, provided that the inverse monomial transformation is applied to the following layer.
\end{theorem}

\begin{proof}
Since $D$ has positive diagonal entries and $P$ is a permutation matrix, ReLU is equivariant under $M=DP$: $\sigma(Mz)=M\sigma(z)$ for every vector $z$. Applying this with $z=W_1x+b_1$, we get
\[
\sigma(W_1'x+b_1')=\sigma(M(W_1x+b_1))=M\sigma(W_1x+b_1).
\]
Therefore
\[
\begin{aligned}
W_2'\sigma(W_1'x+b_1')+b_2'
&=W_2M^{-1}M\sigma(W_1x+b_1)+b_2\\
&= W_2\sigma(W_1x+b_1)+b_2.
\end{aligned}
\]
Thus the function is unchanged. The multi-layer statement follows by applying the same calculation to one hidden layer at a time and compensating the monomial change in the next linear layer.
\end{proof}

\begin{remark}
This theorem separates exact gauge equivalence from empirical performance. Monomial gauges are exact ReLU reparameterizations before averaging. Whether monomial-aligned averaging improves test accuracy is a separate empirical question.
\end{remark}

\begin{remark}[Batch-normalization boundary]
Theorem~\ref{thm:exact-monomial-relu} applies directly to ReLU layers without stateful normalization between the compensated linear maps. A separate ResNet-18 identity audit in the repository checks what remains exact in the presence of BatchNorm. Compatible graph-wide channel permutations, including residual branches, shortcut projections, BatchNorm affine parameters and buffers, and classifier inputs, preserve predictions in both evaluation and training modes within the preregistered floating-point tolerance. Positive channel scaling is more delicate: with frozen running statistics it is exact in evaluation mode only for the explicit affine or epsilon-aware running-affine parameterizations tested there. Scaling running means and variances alone is not exact when the BatchNorm constant $\varepsilon$ is positive, and arbitrary positive scaling is not claimed to be exact in training mode. These are functional identity tests, not merging-performance results.
\end{remark}

\begin{theorem}\label{thm:tetrahedral-mu2-certificate}
Let $K$ be the boundary of a tetrahedron, with vertex set $\{0,1,2,3\}$ and face set
\[
K_2=\{012,013,023,123\}.
\]
Let $\mu_2=\{\pm 1\}$. A face cochain assigns a sign $c_f\in\mu_2$ to each face, and an edge cochain assigns signs $b_{ij}\in\mu_2$ to edges, with coboundary $(\delta b)_{ijk}=b_{ij}b_{jk}b_{ki}$. Since $K$ has dimension $2$, every face cochain is a $2$-cocycle. Such a cochain $c$ is a coboundary if and only if $\prod_{f\in K_2} c_f=+1$. In particular, the sign pattern with exactly one negative face represents the nonzero class in $H^2(S^2;\mu_2)\cong \mu_2$.
\end{theorem}

\begin{proof}
First suppose $c=\delta b$. Multiplying over all four faces gives
\[
\prod_{f\in K_2} c_f=\prod_{(ijk)\in K_2} b_{ij}b_{jk}b_{ki}.
\]
Every edge of the tetrahedral boundary appears in exactly two faces. Hence every edge sign appears twice in the product, and since $b_{ij}^2=1$, the total product is $+1$.

Conversely, identify $\mu_2$ with $\mathbb F_2$ by sending $-1$ to $1$ and $+1$ to $0$. Then multiplication of signs becomes addition in $\mathbb F_2$. The tetrahedral boundary is a triangulation of $S^2$, so $H^2(K;\mathbb F_2)\cong \mathbb F_2$. Since $K$ has four faces, $\dim C^2(K;\mathbb F_2)=4$, and therefore $\im(\delta:C^1\to C^2)$ has codimension one in $C^2(K;\mathbb F_2)$. The condition
\[
\prod_{f\in K_2} c_f=+1
\]
is exactly one linear condition in additive $\mathbb F_2$ notation, namely that the sum of the four face labels is zero. This codimension-one subspace contains $\im\delta$ by the first paragraph, so it is exactly $\im\delta$. Thus $c$ is a coboundary if and only if the product of its face signs is $+1$.

If exactly one face is negative, then the product of the four face signs is $-1$, so the cocycle is not a coboundary. Since $H^2(S^2;\mu_2)\cong \mu_2$, it represents the nonzero class.
\end{proof}

\begin{remark}
\label{rem:tetrahedral-mu2-boundary}
This tetrahedral example is the minimal cocycle-level witness used in our controlled obstruction experiments. It proves that the one-negative-face cochain is not removable by a $\mu_2$-valued edge cochain on $\partial\Delta^3$. The prediction-level two-branch ($q=2$) task in Section~\ref{subsec:h2-mu2-period-index} is therefore recorded separately from this non-coboundary certificate.

The unsigned product criterion in Theorem~\ref{thm:tetrahedral-mu2-certificate} is specific to $\mu_2$. Indeed, inversion is trivial in $\mu_2$, so orientation signs do not alter the product. For a general coefficient group $\mu_n$, the corresponding expression must be written using the chosen orientation of the tetrahedral boundary. With the convention
\[
\partial[0123]=[123]-[023]+[013]-[012],
\]
the oriented evaluation of a multiplicative $2$-cochain $c$ on the fundamental $2$-cycle is $c_{123}c_{023}^{-1}c_{013}c_{012}^{-1}$. Thus one should not replace this alternating expression by the unsigned product of the four face values when $n>2$.
\end{remark}

\begin{remark}
For the controlled $\mu_2$ experiments, we distinguish the exact cohomological certificate from a raw numerical face statistic. For a $\mu_2$-valued face cochain $c=\{c_f\}_{f\in K_2}$, define its negative-face rate by
\[
R_-(c)=\frac{1}{|K_2|}\#\{f\in K_2:c_f=-1\}.
\]
This quantity records the proportion of faces carrying a negative triangle residual. It is not, in general, invariant under the addition of a coboundary.

For the chosen tetrahedral representative with exactly one negative face, one has $R_-(c)=0.250$. The exact certificate of nontriviality is instead
\[
\prod_{f\in K_2}c_f=-1,
\]
or equivalently $c\notin B^2(K;\mu_2)$. Thus the value $0.250$ is a raw representative-level residual statistic, while the nontriviality of the class is established by the finite coboundary certificate of Theorem~\ref{thm:tetrahedral-mu2-certificate}.
\end{remark}

\begin{table*}[!tbp]
\centering
\caption{
Qualitative structural coverage of model-merging methods. This table records
which mechanisms are explicitly present in the method; it is not an accuracy
leaderboard.
}
\label{tab:structural-coverage-score}
\small
\setlength{\tabcolsep}{2.5pt}
\begin{adjustbox}{max width=\textwidth}
\begin{tabular}{lccccccccc}
\toprule
Method
& Validation
& Sync.
& Perm.
& ReLU gauge
& Sign/sparse
& Cycle/hol.
& Central obs.
& Nonabelian lift
& Reject
\\
\midrule
Weight averaging
& no & no & no & no & no & no & no & no & no \\

Model Soups
& yes & no & no & no & no & no & no & no & no \\

\CtwoMthree{}-style
& no & yes & yes & no & no & cycle & no & no & no \\

SLERP
& optional & no & no & no & no & no & no & no & no \\

Task Arithmetic
& optional & no & no & no & no & no & no & no & no \\

TIES
& optional & no & no & no & yes & no & no & no & no \\

DARE
& optional & no & no & no & yes & no & no & no & no \\

TwistedMerge
& yes & yes & yes & yes & partial & yes & yes & yes & yes \\
\bottomrule
\end{tabular}
\end{adjustbox}

\vspace{0.5em}
\footnotesize
Here ``Validation'' means that validation selection is intrinsic to the stated method; ``optional'' means that validation can be used to tune or select the method but is not part of its defining algebraic operation. ``Sync.'' means synchronization of alignments or gauges; ``Perm.'' means permutation gauge handling; ``ReLU gauge'' means monomial positive-scaling/permutation gauges; and ``Sign/sparse'' means coordinatewise sign or sparsity handling. In the column headed ``Cycle/hol.,'' ``cycle'' denotes the cycle-consistency mechanism of \CtwoMthree{}, whereas ``yes'' denotes the broader cycle and holonomy diagnostics of TwistedMerge. ``Central obs.'' means central/projective obstruction detection; ``Nonabelian lift'' means a branch representation with invariant pooling; and ``Reject'' means a conservative no-lift rule when the relevant structural certificate is unavailable.
\end{table*}

\subsubsection{Period-index and rank-lift gates}
We now recall the period-index distinction for central obstruction classes. In the classical theory, period and index are invariants of Brauer classes and twisted sheaves \cite{Gille_Szamuely_2006,Lieblich_08}. In our setting, the same distinction explains why detecting the order of an obstruction is not enough: one must also know which ranks can support a compatible projective or twisted realization.

\begin{definition}
Let $A$ be an abelian coefficient sheaf on $\sC_{\mathrm{learn}}$, and let $[c]\in H^2(\sC_{\mathrm{learn}};A)$ be a torsion central obstruction class. The \emph{period} of $[c]$ is its order in cohomology:
\[
\per([c])=\min\{n>0:n[c]=0\}.
\]
When $A=\GG_m$, this is the usual period of a cohomological Brauer class.
\end{definition}

\begin{definition}
Suppose that $[c]$ is a central class in a setting where $[c]$-twisted locally free objects and their ranks are defined. The \emph{index} of $[c]$ is the rank-divisibility invariant
\[
\ind([c])=\gcd\left\{r>0:
\substack{\text{there exists a rank }r\text{ }[c]\text{-twisted}\\
\text{locally free object}}\right\}.
\]
In the standard Brauer settings in which twisted locally free objects realize the class, this rank gcd agrees with the usual index. Over a field, the latter is the degree of the division algebra in the Brauer class; equivalently, it is the greatest common divisor of the degrees of central simple algebras representing that class.

In the computational setting, let $R$ be the finite set of candidate ranks searched by the algorithm. When the set of passing ranks is nonempty, we define the tested lift threshold by
\[
r_{\mathrm{test}}([c];R)=\min\left\{r\in R:
\substack{\text{the implemented rank-}r\text{ construction}\\
\text{passes the stated test}}\right\}.
\]
If no candidate rank passes, we set $r_{\mathrm{test}}([c];R)=+\infty$. This is an empirical search statistic rather than the classical index. It depends on the candidate set, optimization procedure, implementation, and acceptance tolerances. It may be identified with $\operatorname{ind}([c])$ only when a separate mathematical argument proves that the implemented construction realizes the minimal possible rank.
\end{definition}

\begin{proposition}[Period and index over a field]\label{prop:period-index-field}
Let $F$ be a field and let $\alpha\in\Br(F)$. Then $\per(\alpha)\mid\ind(\alpha)$, and $\per(\alpha)$ and $\ind(\alpha)$ have the same prime divisors \cite[Theorem~2.8.7]{Gille_Szamuely_2006}.
\end{proposition}

Thus the period-index problem asks, for a given class of fields or spaces, to bound $\ind(\alpha)$ in terms of $\per(\alpha)$. In other words, after the period detects the order of the Brauer class, one asks how large the smallest rank or degree realizing that class must be.

\begin{remark}
Classical period-index conjectures seek bounds on $\operatorname{ind}(\alpha)$ in terms of $\operatorname{per}(\alpha)$ under specific hypotheses on the field, variety, characteristic, and torsion. The precise exponent depends on the geometric and cohomological setting. The present paper does not use any general period-index conjecture; it uses only the distinction between the order of a class and the rank divisibility of objects realizing it.

For background and recent progress on the period-index problem, see the introduction of \cite{Gong_26}. Some important related works include \cite{De_Jong_04,Lieblich_08,Hotchkiss_Perry_24,Huybrechts2024HyperkahlerPeriodIndex}.
\end{remark}

\begin{definition}[Projective-representation index]
Let $\Gamma$ be a finite group and let $[\alpha]\in H^2(\Gamma;k^\times)$ be represented by a normalized cocycle $\alpha$. Define
\[
\indrep([\alpha])
=\gcd\left\{\dim W:
\substack{W\text{ is a nonzero finite-dimensional}\\
\alpha\text{-projective representation}}\right\}.
\]
Cohomologous cocycles define equivalent projective-representation categories and hence the same set of representation dimensions, so this quantity depends only on $[\alpha]$. It is the appropriate rank-divisibility invariant for the controlled finite-group experiments. It should not be identified with a classical Brauer index without an additional geometric realization.
\end{definition}

\begin{proposition}[Clock--shift divisibility]\label{prop:clock-shift-divisibility}
Let $\zeta$ be a primitive $d$-th root of unity. If invertible matrices $A,B\in\GL_r(\CC)$ satisfy
\[
AB=\zeta BA,
\]
then $d\mid r$. Conversely, the standard $d\times d$ clock and shift matrices satisfy this relation, and direct sums realize every rank divisible by $d$.
\end{proposition}

\begin{proof}
Taking determinants gives
\[
\det(A)\det(B)=\det(\zeta BA)=\zeta^r\det(B)\det(A),
\]
so $\zeta^r=1$ and hence $d\mid r$. The standard clock matrix $Z=\operatorname{diag}(1,\zeta,\ldots,\zeta^{d-1})$ and cyclic shift matrix $X$ satisfy $ZX=\zeta XZ$. Direct sums preserve the relation.
\end{proof}

\begin{theorem}[Finite Heisenberg rank threshold]\label{thm:finite-heisenberg-index}
Let $k$ be algebraically closed, suppose that $\operatorname{char}(k)\nmid d$, and assume that $k$ contains a primitive $d$-th root of unity $\zeta$. Let
\[
V=(\ZZ/d\ZZ)^{2m},
\]
and write its elements as $(x,y)$ with $x,y\in(\ZZ/d\ZZ)^m$. Define
\[
\alpha((x,y),(x',y'))=\zeta^{y\cdot x'}.
\]
Then
\[
\per([\alpha])=d,\qquad \indrep([\alpha])=d^m.
\]
In particular, every $\alpha$-projective representation has dimension divisible by $d^m$, and one of dimension $d^m$ exists.
\end{theorem}

\begin{proof}
Bilinearity of the exponent shows that $\alpha$ is a normalized $2$-cocycle. Since its values lie in $\mu_d$, the order of $[\alpha]$ divides $d$. Its commutator bicharacter is
\[
\beta(v,v')=\alpha(v,v')\alpha(v',v)^{-1}=\zeta^{y\cdot x'-y'\cdot x}.
\]
This bicharacter is nondegenerate and assumes the primitive value $\zeta$. If $[\alpha]$ had order $e<d$, then $\alpha^e$ would be a coboundary, whose commutator bicharacter is trivial; hence $\beta^e=1$. The primitive value $\zeta$ would then satisfy $\zeta^e=1$, a contradiction. So $\per([\alpha])=d$.

For one Weyl pair, the standard clock and shift matrices give a $d$-dimensional $\alpha$-projective representation. Tensoring the $m$ pairs gives a homomorphism
\[
k^\alpha[V]\longrightarrow M_d(k)^{\otimes m}\simeq M_{d^m}(k).
\]
The $d^2$ Weyl operators span $M_d(k)$, and their tensor products span $M_{d^m}(k)$. Hence the image is the full matrix algebra. Both algebras have dimension $|V|=d^{2m}$ over $k$, so the homomorphism is an isomorphism. Every finite-dimensional module over $M_{d^m}(k)$ is a direct sum of copies of its $d^m$-dimensional simple module. Consequently every $\alpha$-projective representation has dimension divisible by $d^m$, and dimension $d^m$ occurs. Thus $\indrep([\alpha])=d^m$.
\end{proof}

\begin{theorem}[Projective lifting obstruction]
\label{thm:projective-lifting-obstruction}
Let $\mathcal U=\{U_i\to X\}$ be a cover in $\sC_{\mathrm{learn}}$, and suppose that $\bar g_{ij}:U_{ij}\longrightarrow \PGL_r$ is a $\PGL_r$-valued \v{C}ech $1$-cocycle. Thus $\bar g_{ij}\bar g_{jk}=\bar g_{ik}$ on every triple overlap $U_{ijk}$, together with the usual conventions $\bar g_{ii}=1$, and $\bar g_{ji}=\bar g_{ij}^{-1}$.

Assume that lifts $\widetilde g_{ij}:U_{ij}\longrightarrow \GL_r$ have been chosen on the pairwise overlaps, after refining the cover if necessary; after such a refinement, retain the notation $\mathcal U$. Then, on every triple overlap $U_{ijk}$, there is a unique scalar function $c_{ijk}\in\GG_m(U_{ijk})$ such that $\widetilde g_{ij}\widetilde g_{jk}=c_{ijk}\widetilde g_{ik}$. Equivalently,
\[
\widetilde g_{ij}\widetilde g_{jk}\widetilde g_{ik}^{-1}=c_{ijk}I_r.
\]

The collection $c=\{c_{ijk}\}$ is a $\GG_m$-valued \v{C}ech $2$-cocycle and
defines a fixed-cover class $[c]_{\mathcal U}\in\check H^2(\mathcal U;\GG_m)$. This class is independent of the chosen lifts.

Moreover, $[c]_{\mathcal U}=0$ if and only if the lifts can be modified by scalar functions so as to give $\GL_r$-valued transition maps satisfying strict descent on the same cover $\mathcal U$.

The image of $[c]_{\mathcal U}$ under the natural comparison map
\[
\check H^2(\mathcal U;\GG_m)\longrightarrow H^2(X;\GG_m),
\]
where $H^2(X;\GG_m)$ denotes cohomology on the induced site $\sC_{\mathrm{learn}}/X$, is the obstruction class associated with the central extension
\[
1\longrightarrow\GG_m\longrightarrow\GL_r\longrightarrow\PGL_r\longrightarrow 1.
\]
This site-level class vanishes if and only if the corresponding $\PGL_r$-torsor lifts to a $\GL_r$-torsor, equivalently if strict $\GL_r$-valued transition maps exist after possibly refining the cover.
\end{theorem}

\begin{proof}
Since the projective transition maps satisfy $\bar g_{ij}\bar g_{jk}=\bar g_{ik}$, the element $\widetilde g_{ij}\widetilde g_{jk}\widetilde g_{ik}^{-1}$ maps to the identity in $\PGL_r$. It lies in the kernel of $\GL_r\longrightarrow\PGL_r$, which is the central subgroup $\GG_m$. Consequently, there is a unique $c_{ijk}\in\GG_m(U_{ijk})$ such that $\widetilde g_{ij}\widetilde g_{jk}=c_{ijk}\widetilde g_{ik}$.

We next verify the \v{C}ech $2$-cocycle condition. On a quadruple overlap $U_{ijkl}$, associativity gives two expressions for $\widetilde g_{ij}\widetilde g_{jk}\widetilde g_{kl}$. First
\[
(\widetilde g_{ij}\widetilde g_{jk})\widetilde g_{kl}=c_{ijk}\widetilde g_{ik}\widetilde g_{kl}=c_{ijk}c_{ikl}\widetilde g_{il}.
\]
On the other hand,
\[
\widetilde g_{ij}(\widetilde g_{jk}\widetilde g_{kl})=\widetilde g_{ij}c_{jkl}\widetilde g_{jl}=c_{jkl}\widetilde g_{ij}\widetilde g_{jl}=c_{jkl}c_{ijl}\widetilde g_{il},
\]
where we used the fact that $c_{jkl}$ is scalar and therefore central. It follows that $c_{ijk}c_{ikl}=c_{jkl}c_{ijl}$. Equivalently, $c_{jkl}c_{ikl}^{-1}c_{ijl}c_{ijk}^{-1}=1$. This is the multiplicative \v{C}ech $2$-cocycle condition. Hence $c\in Z^2(\mathcal U;\GG_m)$ and defines a class $[c]_{\mathcal U}\in\check H^2(\mathcal U;\GG_m)$.

Suppose that the lifts are changed by scalar functions:
\[
\widetilde g'_{ij}=a_{ij}\widetilde g_{ij},\qquad a_{ij}\in\GG_m(U_{ij}).
\]
Then
\[
\widetilde g'_{ij}\widetilde g'_{jk}= a_{ij}a_{jk}\widetilde g_{ij}\widetilde g_{jk}=a_{ij}a_{jk}c_{ijk}\widetilde g_{ik}=a_{ij}a_{jk}a_{ik}^{-1}c_{ijk}\widetilde g'_{ik}.
\]
Thus the new scalar cocycle is
\[
c'_{ijk}=a_{ij}a_{jk}a_{ik}^{-1}c_{ijk}=(\delta a)_{ijk}c_{ijk}.
\]
Hence $c'$ and $c$ differ by a \v{C}ech coboundary, so $[c]_{\mathcal U}$ is independent of the chosen lifts.

If $[c]_{\mathcal U}=0$, then there is a $\GG_m$-valued \v{C}ech $1$-cochain $a=\{a_{ij}\}$ such that $(\delta a)_{ijk}=c_{ijk}^{-1}$. Replacing $\widetilde g_{ij}$ by $\widetilde g'_{ij}=a_{ij}\widetilde g_{ij}$ gives $c'_{ijk}= (\delta a)_{ijk}c_{ijk}=1$. Hence $\widetilde g'_{ij}\widetilde g'_{jk}=\widetilde g'_{ik}$, so the modified lifts satisfy strict $\GL_r$-descent on $\mathcal U$.

Conversely, suppose that strict $\GL_r$-valued transition maps $\widehat g_{ij}$ lifting the same projective cocycle exist on $\mathcal U$. Since $\widehat g_{ij}$ and $\widetilde g_{ij}$ have the same image in $\PGL_r$, there are scalar functions $a_{ij}\in\GG_m(U_{ij})$ such that $\widehat g_{ij}=a_{ij}\widetilde g_{ij}$. The strict cocycle condition for the $\widehat g_{ij}$ implies $1=(\delta a)_{ijk}c_{ijk}$. Thus $c$ is a \v{C}ech coboundary and $[c]_{\mathcal U}=0$.

Finally, the exact sequence of sheaves
\[
1\longrightarrow\GG_m\longrightarrow\GL_r\longrightarrow\PGL_r\longrightarrow 1
\]
gives the usual connecting obstruction from $\PGL_r$-torsors to $H^2(X;\GG_m)$. The image of $[c]_{\mathcal U}$ in site cohomology is precisely this obstruction class. Its vanishing is equivalent to the existence of a $\GL_r$-torsor lifting the given $\PGL_r$-torsor, which may require refining the original cover.
\end{proof}

\begin{theorem}[Rank-realization gate]
Let $[c]\in H^2(\sC_{\mathrm{learn}};A)$ be a central obstruction class for which an index $\ind([c])$ is defined. If a rank-$r$ locally free $[c]$-twisted realization exists, then $\ind([c])\mid r$. Consequently, a candidate rank $r$ with $\ind([c])\nmid r$ must be rejected by any conservative realization gate. Conversely, if the category of twisted objects is closed under direct sums and contains a rank-$\ind([c])$ object, then every positive multiple of $\ind([c])$ is an admissible rank for such a realization.
\end{theorem}

\begin{proof}
By definition, $\ind([c])$ divides the rank of every locally free twisted object representing the class $[c]$. Hence a rank-$r$ twisted realization can exist only if $\ind([c])\mid r$. This proves the conservative rejection rule.

For the converse, suppose that $E$ is a rank $\ind([c])$ twisted object. Then for any $m\geq 1$, the direct sum $E^{\oplus m}$ is again $[c]$-twisted and has rank $m\ind([c])$. Thus every multiple of the index is admissible under the stated closure assumption.
\end{proof}

\begin{corollary}[Algorithmic period-index rule]
In the central/projective module of \emph{TwistedMerge}, the period determines the torsion type of the obstruction, while a certified classical index, projective-representation index, or proved lower-bound gate restricts the eligible ranks. The algorithm should activate a rank-$r$ construction only when the available theorem or certificate permits that rank. If no such rank certificate is available, the arithmetic status remains uncertified; divisibility by the period alone is not sufficient.
\end{corollary}

\begin{remark}
This is the mathematical reason for separating period from index in the experiments. In the controlled clock--shift and finite-Heisenberg systems, the phase order gives the period while the projective-representation index gives the arithmetic rank threshold. More generally, a period records torsion type, whereas an index records rank divisibility for the corresponding geometric or projective realization.
\end{remark}

\begin{remark}[Arithmetic admissibility is not a construction]
The divisibility condition $\ind([c])\mid r$, or $\indrep([\alpha])\mid r$ in a controlled finite system, is necessary for the corresponding rank. The algorithm separates four events:
\[
\begin{gathered}
\text{rank admissible},\qquad \text{lift constructed},\\
\text{readout invariant},\qquad \text{candidate validated}.
\end{gathered}
\]
Only a candidate passing all four stages enters the validation envelope.
\end{remark}

\begin{proposition}[Descent after invariant pooling]\label{prop:invariant-pooling-descent}
Let $K$ be connected, let $G$ act on a lifted feature space $W$ through a representation $\rho:G\to\GL(W)$, and choose a root vertex. Let $H\leq G$ be the subgroup generated by the path-difference holonomies based at that root; equivalently, after a spanning tree is chosen, $H$ is generated by the fundamental-cycle holonomies associated with the non-tree edges. Suppose that $P:W\to Y$ satisfies
\[
P\rho(h)=P\qquad\text{for every }h\in H.
\]
Then the prediction obtained by transporting a local lifted feature to the root and applying $P$ is independent of the chosen transport path. It defines a globally consistent pooled predictor.
\end{proposition}

\begin{proof}
Let $T_p$ and $T_q$ be the transport operators associated with two paths from the same local chart to the root. With the chosen path-order convention, the two operators differ by an element of the residual holonomy subgroup: there is an $h\in H$ such that $T_q=\rho(h)T_p$. Therefore
\[
PT_q=P\rho(h)T_p=PT_p.
\]
Thus the pooled prediction is path-independent.
\end{proof}

\begin{remark}
For the regular representation of a finite group $\Gamma$, the averaging map
\[
P((v_\gamma)_{\gamma\in\Gamma})=\frac{1}{|\Gamma|}\sum_{\gamma\in\Gamma}v_\gamma
\]
satisfies $P\rho(g)=P$ for every $g\in\Gamma$. This is the mechanism checked in the executed two-loop nonabelian experiment. It explains how a nontrivial holonomy can remain present before pooling while becoming harmless after the invariant readout.
\end{remark}

\subsection{Core certification pipeline}\label{subsec:core-pipeline}
The mandatory part of TwistedMerge is a certification and abstention pipeline. Projective rank searches and nonabelian branch representations are optional research branches; they are not required to return a diagnostic decision.

\begin{breakablealgorithm}
\caption{Core TwistedMerge certification pipeline}
\label{alg:twistedmerge-core}
\begin{algorithmic}[1]
\Require Checkpoints $\{M_i\}$; frozen complex-construction data $D_K$; alignment, certification, selection, and test splits; gauge families; candidate central coefficient systems; structural and statistical thresholds.
\Ensure A selected ordinary or synchronized predictor and a diagnostic status for every tested structural branch.
\State Construct $K=\Phi(D_K)$ before fitting transitions or inspecting residuals.
\State Build ordinary fixed-chart candidates and estimate pairwise transitions on $D_{\mathrm{align}}$.
\State Synchronize each admissible gauge family and add every strict candidate that passes the pairwise gate.
\State On $D_{\mathrm{cert}}$, compute inverse-consistency, projection-fidelity, triangle, centrality, and closure diagnostics.
\If{a central coefficient system passes all structural gates}
    \State Compute a confidence interval for distance to coboundaries.
    \State Return \emph{trivial}, \emph{nontrivial on $K$}, or \emph{uncertified} by Theorem~\ref{thm:frozen-complex-certificate}.
    \If{the status is nontrivial and a proved projective-rank construction is available}
        \State Attempt it only at admissible ranks; admit it only after overlap and readout verification.
    \EndIf
\Else
    \State Treat the residual as noncentral holonomy or as uncertified; do not assign an $H^2$ class.
\EndIf
\If{a nonabelian representation and invariant readout are independently certified}
    \State Admit the branch predictor and record its parameter, memory, and inference multipliers.
\EndIf
\State Select among admitted candidates on $D_{\mathrm{select}}$; evaluate the frozen choice once on $D_{\mathrm{test}}$.
\State Return the selected predictor together with the complete diagnostic and abstention record.
\end{algorithmic}
\end{breakablealgorithm}

The detailed candidate-generation algorithm, all diagnostics, and computational-cost accounting are given in Appendix~\ref{app:full-algorithm}. The reduced form above makes the scientific claim explicit: TwistedMerge is a regime classifier with conservative fallback, we do not guarantee that a projective or branch candidate will outperform ordinary merging.

\section{Primary experiments}\label{sec:experiments}
The primary experiments are organized by hypothesis. They ask whether a cycle defect is synchronization-removable, whether a concrete factor-merging rule is gauge-invariant, whether the three-way certificate has controlled decision utility under noise, and whether the resulting diagnostics generalize to natural checkpoint collections. Secondary benchmarks and structural sanity checks are reported in Appendix~\ref{app:secondary-experiments}.

\subsection{Primary hypotheses and statistical protocol}\label{subsec:primary-statistical-protocol}
The inferential unit is an independently trained checkpoint group or base-model seed. Exact algebraic rows are reported as deterministic audits. Confirmatory and exploratory analyses are separated in Table~\ref{tab:statistical-protocol}.

\begin{table*}[!tbp]
\centering
\caption{Statistical protocol for the primary experimental claims.}
\label{tab:statistical-protocol}
\small
\begin{adjustbox}{max width=\textwidth}
\begin{tabular}{L{0.22\textwidth}L{0.16\textwidth}L{0.18\textwidth}L{0.18\textwidth}L{0.20\textwidth}}
\toprule
Hypothesis & Inferential unit & Data roles & Status & Primary endpoint \\
\midrule
Planted alignment defect is removable & trained base-model seed & train; align; evaluate & controlled causal & degradation relative to the uncorrupted aligned merge \\
Adapter-factor merge is gauge-stable & independently trained adapter group & fixed trained group; equivalent scrambles; held-out evaluation & confirmatory gauge audit & relative merged-update change and prediction disagreement \\
Central certificate abstains safely & case/noise/noise-type seed & exact generator; certification only & controlled calibration & certification coverage, uncertainty, rejection, and false-lift rate \\
Natural cycle residual predicts degradation & complete checkpoint collection & train; align; leave-one-setting-out prediction; test & preregistered confirmatory & held-out $R^2$ for weight-average degradation \\
Natural central or holonomy branch activates & independently trained collection or lineage & separate fitting, gating, and evaluation & confirmatory negative audit & preregistered gate pass and incremental predictive or corrective value \\
\bottomrule
\end{tabular}
\end{adjustbox}
\end{table*}

All final model choices use validation or selection data, while test data are evaluation-only. The complex-construction rule is conceptually assigned its own split $D_K$. In the controlled tetrahedral witness, the incidence structure is prescribed by construction; in the natural checkpoint studies, no nontrivial $K$-based central claim is promoted because an application-grounded complex-selection and sensitivity protocol has not yet passed the required gate.

Where several fixed settings or complex thresholds are inspected, exploratory screens are not promoted to a confirmatory claim merely because one row passes. The one-of-sixteen fixed-setting correlation result is reported as exploratory, while the leave-one-setting-out study is the primary natural prediction test. For a predeclared finite complex family, Theorem~\ref{thm:complex-family-certificate} supplies simultaneous error control through the reported $\sum_\lambda\delta_\lambda$ budget.

\subsection{Primary practical audit: trained low-rank adapter gauges}\label{subsec:trained-lora-gauges}
The exact-gauge discussion above concerns hidden-unit coordinates in ReLU networks. A later repository program tests an analogous gauge issue on trained low-rank adapters of LoRA form \cite{Hu2021LoRA}. For a rank-$r$ update
\[
\Delta W=BA,
\qquad
B\in\mathbb R^{d_{\mathrm{out}}\times r},
\quad
A\in\mathbb R^{r\times d_{\mathrm{in}}},
\]
the factorization is unchanged by
\[
(B,A)\longmapsto (BQ,Q^{-1}A),
\qquad Q\in\mathrm{GL}_r.
\]
Thus an averaging rule applied directly to $A$ and $B$ should be tested for invariance under equivalent rank-space representations.

The experiment reuses five independent groups of eight chart-specific rank-$4$ residual adapters over frozen ResNet-18 CIFAR-10 features. The primary audit uses three preregistered well-conditioned gauge families---orthogonal, positive diagonal with condition number at most $8$, and dense with condition number at most $30$---and twenty scrambles per family and group, giving $300$ primary rows. Individual effective updates, logits, predictions, and validation accuracies are preserved to numerical precision. Naive factor averaging is not invariant, whereas whitened global synchronization and deterministic full-delta SVD are stable; Table~\ref{tab:trained-lora-gauge-stability} records the largest observed deviations.

\begin{table*}[!tbp]
\centering
\caption{Gauge stability on trained rank-$4$ residual adapters. The individual-adapter row checks preservation under the reparameterization itself; the remaining rows compare merged outputs across equivalent representations.}
\label{tab:trained-lora-gauge-stability}
\small
\begin{adjustbox}{max width=\textwidth}
\begin{tabular}{lccc}
\toprule
Method or check & maximum relative update change & maximum logit change & maximum prediction disagreement \\
\midrule
Individual-adapter preservation & $1.447298\times10^{-15}$ & $2.131628\times10^{-14}$ & $0$ \\
Naive factor average & $1.034210\times10^{1}$ & $1.156529\times10^{1}$ & $0.579250$ \\
Whitened global synchronization & $2.709900\times10^{-13}$ & $9.865442\times10^{-13}$ & $0$ \\
Deterministic full-delta SVD & $2.711270\times10^{-15}$ & $1.154632\times10^{-14}$ & $0$ \\
\bottomrule
\end{tabular}
\end{adjustbox}
\end{table*}

Across all primary families, the reported group-bootstrap difference in relative merged-update change between global synchronization and naive averaging is $-3.218202$, with $95\%$ confidence interval $[-3.437573,-2.998831]$. This is an invariance result. Full-delta SVD is at least as stable and exceeds global synchronization by $0.051725$ in mean test accuracy across the five trained groups. Moreover, every primary natural cycle-aware row exceeds the preregistered cycle gate and selects the dense full-delta SVD fallback; no factor-only natural holonomy correction is obtained.

The associated process-isolated systems grid spans dimensions $768$--$4096$, ranks $4$--$32$, and $4$--$16$ adapters. All $336$ method/shape cases and $1008$ timed trials finish without failure or timeout. Pairwise and global factor-space methods make no dense effective-update allocation. At dimension $4096$, global synchronization has lower measured incremental peak resident-set size than deterministic dense SVD in all twelve rank/count configurations and is faster in ten of the twelve; so the runtime statement is not uniform. Table~\ref{tab:trained-lora-systems} gives one representative case.

\begin{table*}[!tbp]
\centering
\caption{Representative process-isolated systems measurement at dimension $4096$, rank $8$, and eight adapters. Temporary-memory figures are analytical counts. These numerical fixtures are scaled from the trained factors and do not provide application-accuracy evidence.}
\label{tab:trained-lora-systems}
\small
\begin{adjustbox}{max width=\textwidth}
\begin{tabular}{lcccc}
\toprule
Method & dense effective update & incremental peak RSS & temporary memory & median time \\
\midrule
Whitened global synchronization & no & $11{,}845{,}632$ bytes & $2{,}113{,}536$ bytes & $0.009687$ s \\
Deterministic full-delta SVD & yes & $54{,}214{,}656$ bytes & $68{,}157{,}440$ bytes & $0.061481$ s \\
\bottomrule
\end{tabular}
\end{adjustbox}
\end{table*}

The scope is limited. The trained corpus contains one residual LoRA-form feature layer and a classification head, not a multi-layer transformer LoRA stack. The ill-conditioned $10^8$ boundary is excluded from the primary claim and produces $204$ alignment failures. The experiment does not establish adapter fingerprinting, unrestricted $\mathrm{GL}_r$ robustness, accuracy superiority, uniqueness, or a natural Brauer or holonomy class.

The experiment is deliberately reported as a bounded systems and invariance study. It does not satisfy the stronger modern-scale test of multilayer transformer or vision-transformer adapters proposed in Section~\ref{sec:limitations}; that experiment remains necessary for a performance-oriented claim.

\subsection{Causal parameter-level bridge: planted alignment inconsistency}\label{subsec:planted-alignment-bridge}
The controlled obstruction examples below begin from prescribed cocycles or holonomy representations. We add an intermediate neural-network experiment in which the local functions are held fixed while the observed alignment data are perturbed. For each seed, one MNIST MLP is trained and four copies are produced by exact hidden-unit permutations together with the compensating classifier-column permutations. The copies are functionally equivalent: the mean base accuracy is $0.8632$, the maximum measured copy-accuracy standard deviation is $0$, and the maximum logit disagreement is approximately $2.9\times10^{-6}$.

One observed alignment edge is then corrupted at four defect levels. Table~\ref{tab:planted-alignment-inconsistency} and Figure~\ref{fig:planted-alignment-inconsistency} report the central-involution family. Git Re-Basin-style \cite{Ainsworth2022GitReBasin} degradation increases monotonically from $0$ to $0.0387$, and the Spearman correlation between cycle score and degradation is $0.8741$. \CtwoMthree{}-style \cite{Crisostomi2024C2M3} synchronization removes the one-edge inconsistency at every level. The rank-lift branch agrees with \CtwoMthree{} and gives no additional gain.

\begin{table*}[!tbp]
\centering
\caption{Causal planted-alignment benchmark on functionally identical MNIST MLP copies. The degradation is measured relative to the uncorrupted aligned merge.}
\label{tab:planted-alignment-inconsistency}
\small
\begin{tabular}{lccc}
\toprule
Defect level & Git Re-Basin degradation & \CtwoMthree{} degradation & Rank-lift minus \CtwoMthree{} \\
\midrule
none & $0.0000$ & $0.0000$ & $0.0000$ \\
low & $0.0031$ & $0.0000$ & $0.0000$ \\
medium & $0.0149$ & $0.0000$ & $0.0000$ \\
high & $0.0387$ & $0.0000$ & $0.0000$ \\
\bottomrule
\end{tabular}
\end{table*}

\begin{figure}[!tbp]
\centering
\includegraphics[width=\columnwidth]{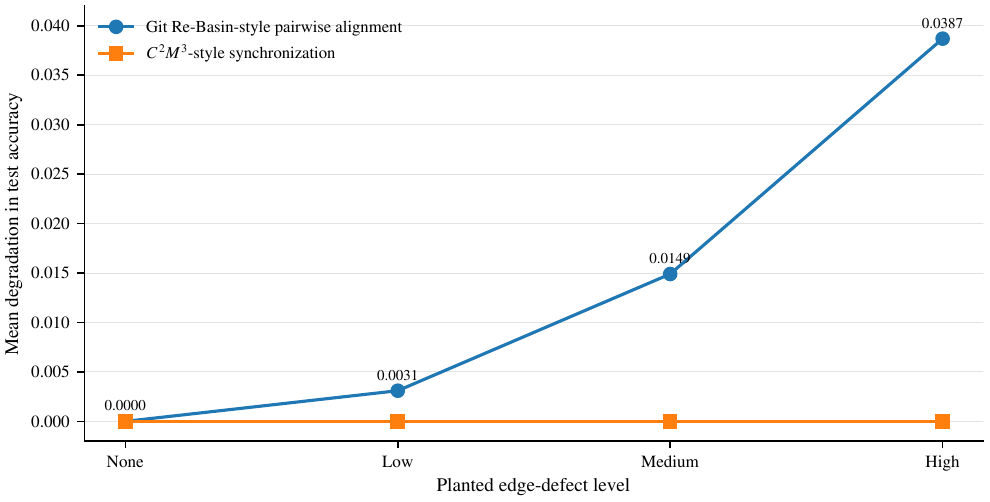}
\caption{Causal parameter-level alignment inconsistency. The vertical axis is the decrease in test accuracy relative to the uncorrupted aligned merge. The pairwise merge degrades with the planted edge defect, while cycle-consistent synchronization removes this particular inconsistency.}
\label{fig:planted-alignment-inconsistency}
\end{figure}

A random noncentral corruption produces a weaker trend, with Spearman correlation $0.6967$ and high-defect degradation $0.0190$. This comparison shows that the cycle statistic is not a universal scalar explanation of every merging failure. More importantly, the central planted family is a synchronization-positive benchmark: the residual is removed by an ordinary edge correction.

\subsection{Controlled central and projective structures}\label{subsec:controlled-central-projective}
We now turn to the controlled central and projective experiments. Three levels must be kept separate. A cocycle-level obstruction witness proves that a specified edge correction is impossible. A projective or Morita construction realizes a central commutation relation at an admissible rank. A charted branch predictor produces an invariant output on a controlled task. None of the last two statements implies that a non-coboundary cocycle has become a coboundary.

\subsubsection{Tetrahedral $H^2(\mu_2)$ witness and prediction-level branch task}\label{subsec:h2-mu2-period-index}
Let $K=\partial\Delta^3$ and let the face signs be
\[
(c_{012},c_{013},c_{023},c_{123})=(-1,+1,+1,+1).
\]
By Theorem~\ref{thm:tetrahedral-mu2-certificate}, this cocycle is not a coboundary. The negative-face rate is $0.250$, but the exact certificate is the product of the four face signs, not the raw rate. By Proposition~\ref{prop:constant-edge-tetrahedron}, this one-negative-face pattern is not produced by six context-independent edge matrices satisfying the exact central hypotheses. We treat it as a cocycle-level oracle witness on the prescribed comparison complex.

\begin{table*}[!tbp]
\centering
\caption{Tetrahedral $H^2(\mu_2)$ cocycle-level witness. No rank-lift success is asserted for the non-coboundary class; the current edge-level algorithm records that no permitted edge correction removes it.}
\label{tab:h2-mu2-obstruction-data}
\small
\begin{tabular}{lcccc}
\toprule
Case & Local loss & Pairwise loss & Negative-face rate & Coboundary? \\
\midrule
trivial face cocycle & $0.000$ & $0.000$ & $0.000$ & yes \\
nontrivial face cocycle & $0.000$ & $0.000$ & $0.250$ & no \\
\bottomrule
\end{tabular}
\end{table*}

A separate controlled finite central task supplies a charted representation with branch count $q=2$ and tests prediction-level recovery. This is the benchmark shown in Table~\ref{tab:controlled-finite-central-all-methods}. Strict pairwise and synchronized methods plateau at $0.7500$, while the correct $q=2$ branch and a learned context router reach $1.0000$. Wrong-context and wrong-twist controls remain at chance. The successful rows show that the supplied charted representation carries task-relevant central structure; they are not a construction of ordinary edge-level descent for the non-coboundary tetrahedral class.

\begin{table*}[!tbp]
\centering
\caption{Controlled finite central charted-prediction benchmark with branch count $q=2$. Branch, router, and ensemble rows are not capacity-matched single-model outputs.}
\label{tab:controlled-finite-central-all-methods}
\small
\begin{tabular}{lcc}
\toprule
Method & Width 32 acc. & Width 64 acc. \\
\midrule
ordinary weight average & $0.4993$ & $0.4998$ \\
Git Re-Basin pairwise & $0.7500$ & $0.7500$ \\
\CtwoMthree{} synchronized & $0.7500$ & $0.7500$ \\
no-twist branch control & $0.7500$ & $0.7500$ \\
validation-selected branch ensemble & $0.7500$ & $0.7500$ \\
distilled charted single model & $0.7500$ & $0.7500$ \\
parameter-matched wide control & $0.7373$ & $0.7393$ \\
random branch ensemble & $0.5500$ & $0.5167$ \\
wrong-context control & $0.5000$ & $0.5000$ \\
wrong-twist control & $0.5000$ & $0.5000$ \\
correct $q=2$ charted branch & $\mathbf{1.0000}$ & $\mathbf{1.0000}$ \\
learned context router & $\mathbf{1.0000}$ & $\mathbf{1.0000}$ \\
\bottomrule
\end{tabular}
\end{table*}

\begin{figure}[!tbp]
\centering
\includegraphics[width=\columnwidth]{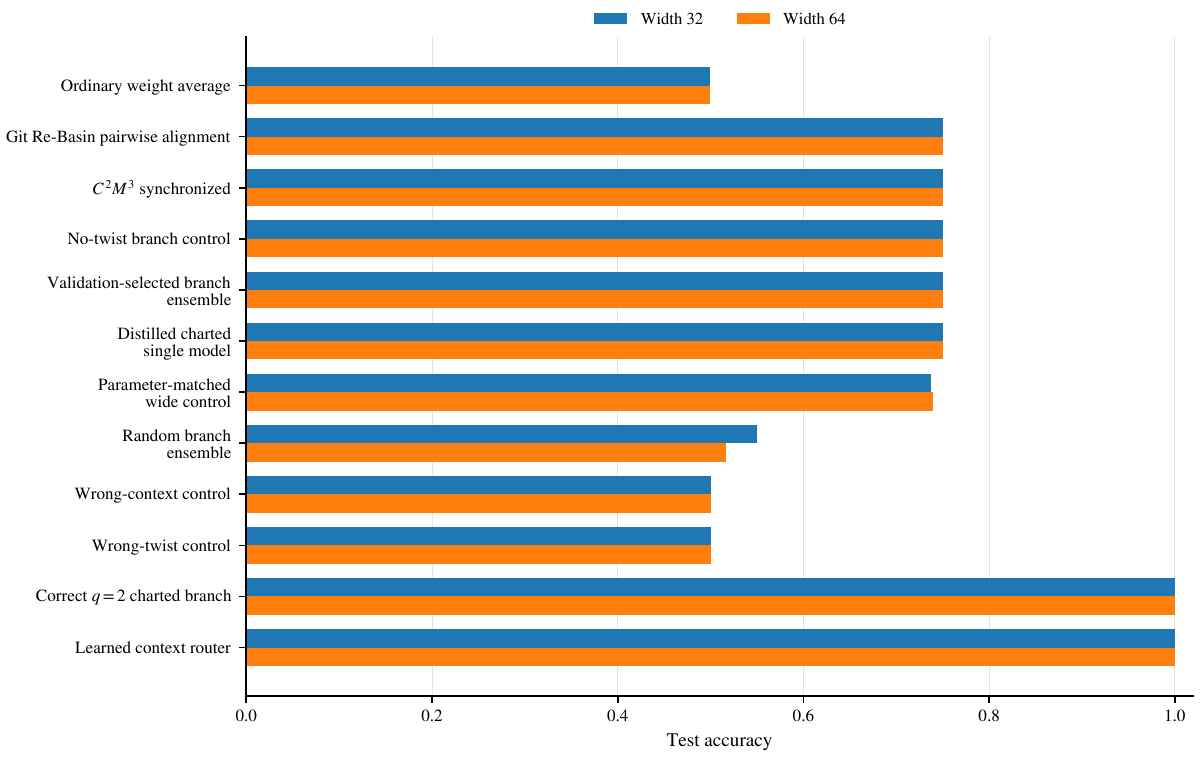}
\caption{Controlled finite central charted-prediction benchmark on the full accuracy scale $[0,1]$. The branch count is $q=2$. The correct branch and router recover the task, but this prediction-level result does not make the tetrahedral non-coboundary class vanish.}
\label{fig:controlled-finite-central-all-methods}
\end{figure}

The repository also contains a complete-graph $\mu_2$ edge-noise sweep. The face cochain in this experiment is $c=\delta b$, so it is a coboundary rather than a nontrivial $H^2$ class. At low triangle-residual rates, ordinary and rank-lifted predictions agree. At residual rates $0.3992$, $0.4795$, and $0.4947$, ordinary accuracy is $0.9214$, $0.8100$, and $0.7196$, while the lifted branch remains at $0.9338$, $0.9351$, and $0.9356$. This is a robustness experiment for structured edge inconsistency.

\begin{figure}[!tbp]
\centering
\includegraphics[width=\columnwidth]{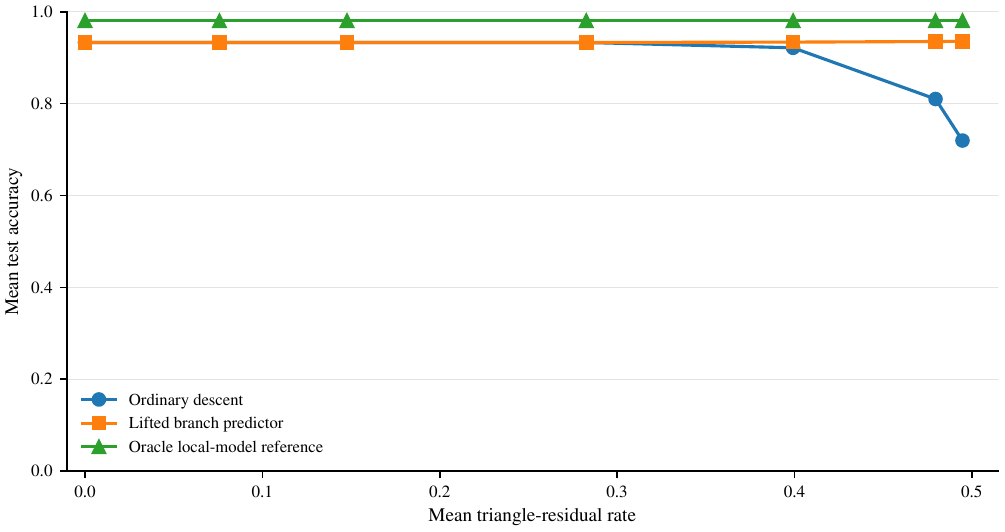}
\caption{Synthetic $\mu_2$ edge-noise sweep on the full accuracy scale $[0,1]$. The horizontal axis is the mean triangle-residual rate before complete edge-level repair. Because the face cochain is induced from an edge cochain, it is a coboundary and not a nontrivial $H^2$ class. The lifted curve is a branch-prediction control rather than a cohomology certificate.}
\label{fig:mu2-accuracy-vs-triangle-residual}
\end{figure}

\subsubsection{Robust certification under noisy transition maps}\label{subsec:robust-central-certification}
The exact rank identities do not determine what to do with noisy estimates. The robust calibration uses the three-way output from Proposition~\ref{prop:margin-obstruction-decision}. Here $\tau_{\mathrm{phase}}$ is the tolerance for the scalar phase-relation residual. Under the selected policy
\[
\tau_{\mathrm{cent}}=\tau_{\mathrm{phase}}=3\times10^{-4},\qquad \text{confidence margin}=0.25,
\]
the small-noise certification rate is $1$, the false-central rate is $0$, the false-lift rate is $0$, the medium-noise uncertainty rate is $0.7233$, and the large-noise rejection rate is $1$. No grouped trivial or noncentral control produces a false lift. Uncertain rows do not activate a lift.

\begin{figure}[!tbp]
\centering
\includegraphics[width=\columnwidth]{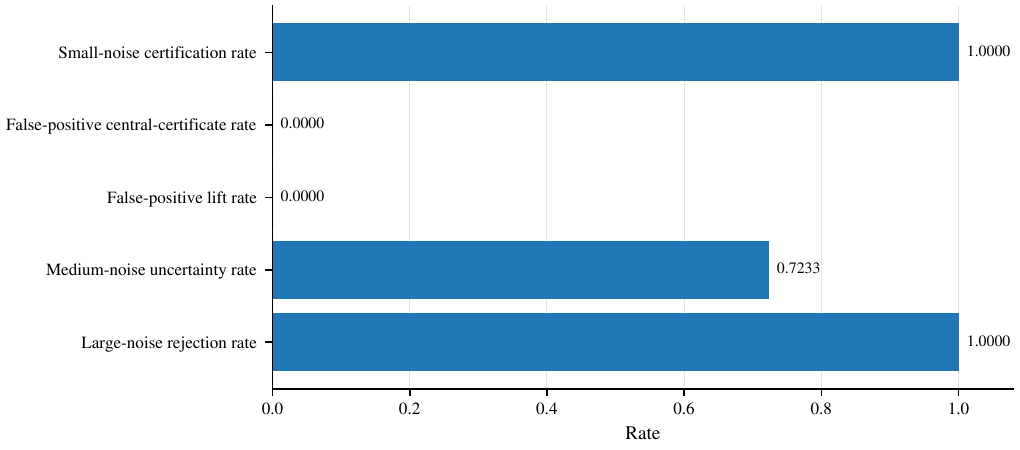}
\caption{Robust period-index certification and abstention. Intermediate noise is assigned to an uncertain state rather than being forced into a binary decision.}
\label{fig:robust-period-index-calibration}
\end{figure}

\subsubsection{Decision utility and threshold sensitivity}\label{subsec:decision-utility}
The calibration can be read as a selective decision system. ``Coverage'' is the fraction of designated small-noise positive cases that receive a central/projective certificate; ``selective error'' is the false-lift rate among tested decisions; intermediate-noise uncertainty records abstention rather than an error. Table~\ref{tab:certificate-threshold-sensitivity} reports the precomputed policy sweep at fixed confidence margin $0.25$.

\begin{table*}[!tbp]
\centering
\caption{Controlled coverage--abstention trade-off for the centrality and phase tolerances. All tested policies have zero false-central and false-lift rate on the grouped negative controls; the table does not claim calibration under natural distribution shift.}
\label{tab:certificate-threshold-sensitivity}
\small
\begin{tabular}{ccccc}
\toprule
Tolerance & small-noise coverage & medium-noise uncertainty & false-lift rate & large-noise rejection \\
\midrule
$10^{-5}$ & $0.6417$ & $0.8892$ & $0$ & $1.0000$ \\
$3\times10^{-5}$ & $0.8017$ & $0.8892$ & $0$ & $1.0000$ \\
$10^{-4}$ & $0.9506$ & $0.8892$ & $0$ & $1.0000$ \\
$3\times10^{-4}$ & $1.0000$ & $0.7233$ & $0$ & $1.0000$ \\
\bottomrule
\end{tabular}
\end{table*}

The selected policy increases positive coverage while retaining zero false lifts in this controlled grid. Its practical utility is rejection of unsupported structure. A complete risk--coverage study across architectures and distribution shifts remains open.

\subsection{Natural-data and architecture boundary}\label{subsec:natural-boundary}
The controlled experiments identify their coefficient group and representation by construction. Natural neural checkpoints do not provide this information. The repository contains several repeated-seed and held-out tests designed to prevent a descriptive residual from being promoted into a general obstruction claim.

\subsubsection{Natural diagnostic prediction}
An earlier quality-gated permutation/activation study using the repository architecture labelled \texttt{mlp2} evaluated sixteen primary fixed settings with thirty observed seeds per setting. One setting passed the predefined positive-correlation gate: Fashion-MNIST, $N=3$, no shift, activation matching, with Pearson correlation $0.2420$, bootstrap lower bound $0.0065$, and Spearman correlation $0.2629$. The other fifteen settings did not pass. A later monomial-gauge verifier, which varied MNIST and Fashion-MNIST, $N\in\{3,4\}$, widths $64,128$, no shift and input noise, and monomial activation or weight matching, produced no passing observed setting.

The preregistered held-out study uses $120$ natural checkpoint collections and leaves out one complete $(N,\text{width})$ setting at a time. Its primary predictor is the cycle residual and its target is weight-average degradation. The cycle residual has Pearson correlation $-0.1828$, Spearman correlation $-0.2626$, and held-out $R^2=-0.2055$. Pairwise alignment loss has held-out $R^2=0.3872$, while validation loss and validation delta have held-out $R^2=0.5108$ and $0.9723$. Thus cycle residual is not supported as a general held-out predictor in this dataset.

\begin{figure}[!tbp]
\centering
\includegraphics[width=\columnwidth]{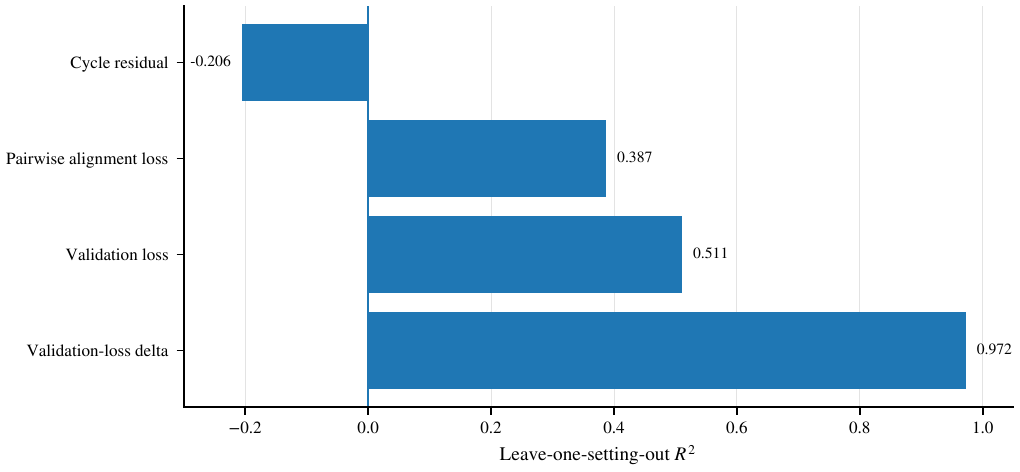}
\caption{Natural-data held-out prediction boundary. The horizontal axis is leave-one-setting-out $R^2$. ``Validation-loss delta'' is the validation-loss change relative to the synchronized baseline used in the repository audit. The cycle residual has negative held-out $R^2$, while ordinary validation information is substantially more predictive in this run.}
\label{fig:natural-diagnostic-prediction-boundary}
\end{figure}

\subsubsection{Natural holonomy applications and model lineages}\label{subsec:natural-holonomy-applications}
The most recent repository programs test whether the structural diagnostics yield a positive practical holonomy application. The first program reuses a shared frozen-encoder CIFAR-10 corpus with eight $D_4$ chart adapters and five training seeds. Its held-out-selected weight-derived connections are trivial or coboundary. Among fifteen nontrivial alternative seed--estimator settings, all fifteen are noncentral under the strict, medium, and loose tolerance policies. No natural seed--estimator setting is a central finite-order candidate, so the term Brauer class is not warranted for this corpus.

The predictive and corrective gates are also negative. Orbit invariant pooling is $0.0160$ below a generic mixture-of-experts control and $0.0062$ below a matched random-branch control, with both reported confidence intervals excluding zero in the unfavorable direction. A mixed natural/controlled linter improves descriptively, but on natural rows alone the full holonomy diagnostic changes AUROC by $-0.0273$, with reported $95\%$ cell-bootstrap interval $[-0.0429,-0.0057]$. In controlled finite-Heisenberg carrier layers on the same frozen-image features and saved logits, the $(\operatorname{period},\operatorname{index})$ cases $(2,2)$, $(2,4)$, and $(3,3)$ recover the exact minimum structurally coherent capacity for all five seeds. The coherent lift nevertheless ties a parameter-matched generic unitary carrier in accuracy in every predicted-index comparison. This is a controlled structural validation, not a superior practical capacity planner.

A separate confirmatory model-lineage audit uses one frozen ResNet-18 encoder, rank-$4$ feature adapters, a classifier, three deterministic CIFAR-10 corruptions, and five independent training seeds as the inferential units. Different learning orders produce measurable terminal differences, but no loop/layer row passes the frozen nonidentity-stability gate, no commutator row passes the noncommutation gate, and the relevant admissible loop distances are only $1.933\times10^{-14}$ to $6.229\times10^{-14}$. The four preregistered predictive and corrective hypotheses all fail; in particular, there are zero harmful raw merge rows and no held-out improvement beyond pairwise drift. The largest raw adapter-layer distance is $0.9734$, but it is not stable and is not promoted to a holonomy signal.

\begin{table*}[!tbp]
\centering
\caption{Recent practical holonomy and model-lineage audits. These are negative or controlled-structural results and do not establish a natural Brauer class, a practical invariant-pooling advantage, or a general model-lineage holonomy predictor.}
\label{tab:latest-holonomy-application-boundary}
\small
\begin{adjustbox}{max width=\textwidth}
\begin{tabular}{lll}
\toprule
Audit & Recorded outcome & Interpretation \\
\midrule
Natural central certificate & $0$ central finite-order candidates & Brauer terminology not warranted \\
Invariant pooling & $-0.0160$ vs generic MoE; $-0.0062$ vs random branch & practical gate fails \\
Natural-only linter & AUROC delta $-0.0273$, CI $[-0.0429,-0.0057]$ & no incremental mergeability value \\
Controlled real-feature carrier & exact thresholds $(2,2),(2,4),(3,3)$; accuracy delta $0$ vs generic & structural validation only \\
Model-lineage loops & $0$ stable nonidentity rows; $0$ noncommuting rows & no incremental predictive/corrective value \\
\bottomrule
\end{tabular}
\end{adjustbox}
\end{table*}

These audits strengthen the conservative interpretation. The supported practical use is certification and rejection: a selected trivial connection is kept trivial, noncentral alternatives are not forced into a Brauer interpretation, and unstable loop estimates are not used to activate a lift.

The conclusion of the natural-data section is negative. The repository does not certify a natural central, Brauer, or period-index class; cycle residual is not a general held-out predictor; and exact gauge refinements do not displace greedy soup as the strongest generic boundary baseline in these experiments. These outcomes are consistent with the conservative design of Algorithm~\ref{alg:twistedmerge}: when the structural gates do not pass, no lift is activated.

\subsection{Secondary evidence summary}\label{subsec:secondary-evidence-summary}
The remaining experiments support the taxonomy. The independent-seed MLP ladder places the selector at the greedy-soup boundary; the matched attribution audit does not establish a TwistedMerge-specific selector gain. Same-base task-vector methods can outperform the original checkpoint-only soup because a common chart is supplied. Finite-Heisenberg and learned time-frequency systems verify projective-representation rank thresholds. The executed two-loop $S_3$ and $D_4$ experiment verifies noncommuting holonomy and invariant pooling but has no discriminative accuracy value. Batch-normalization, CNN, bridge-dataset, and branch-lift audits are engineering or boundary results. Full tables, figures, and claim qualifications appear in Appendix~\ref{app:secondary-experiments}.

\section{Limitations and decisive next tests}\label{sec:limitations}
The first limitation is the construction of $K$. Definition~\ref{def:frozen-comparison-complex} gives an ex ante protocol, Theorem~\ref{thm:complex-family-certificate} controls a predeclared threshold family, and Proposition~\ref{prop:complex-refinement} gives an exact topology-sensitivity test, but the present natural checkpoint collections do not yet supply an application-grounded higher-overlap rule that persists across datasets and architectures. The controlled tetrahedral boundary is prescribed; it should not be mistaken for a topology discovered from four ordinary checkpoints.

The second limitation is a realization gap. The one-negative-face $\mu_2$ class is a cocycle-level oracle witness. The separate $q=2$ charted task realizes a twisted prediction rule, but it does not make the original class a coboundary and it is not an ordinary edge-level descent realization of the same neural system. The decisive next controlled experiment should estimate overlap-dependent transition sections and certify the nontrivial class and the successful capacity-accounted projective predictor in one data-generating mechanism.

The third limitation is scale. The low-rank-adapter audit uses one residual feature layer over frozen ResNet-18 features. It is not a multilayer LoRA stack on a transformer or vision transformer. A journal-level practical follow-up should compare naive factor averaging, pairwise and global factor synchronization, dense delta averaging or SVD, and applicable task-vector baselines across several backbones, task families, adapter ranks, and independent training groups. It should report accuracy, invariance error, peak accelerator memory, runtime, and inference cost under a predeclared primary endpoint.

The fourth limitation is decision utility. The controlled calibration reports a useful coverage--abstention trade-off and zero false lifts on the tested negative controls, but it does not yet provide natural distribution-shift calibration, full gate ablations, or regret relative to an oracle regime selector. The current natural evidence instead shows that pairwise and validation diagnostics are more predictive than cycle residual. Accordingly, the supported use of the certificate is conservative rejection rather than natural accuracy forecasting.

Finally, several mathematical statements in the appendices are standard or adapted background. The new theorem-level content is deliberately isolated in Table~\ref{tab:theory-status}. The neuro-algebraic-geometry discussion is retained only as additional context.

\section{Conclusion}\label{sec:conclusion}
Cycle inconsistency in model merging is not automatically cohomology. A higher-obstruction claim depends on a frozen comparison complex, coherent transition data, a fixed coefficient system, centrality, closure, and a stated repair model. TwistedMerge organizes these requirements into a three-way certification pipeline with ordinary or synchronized fallback.

The causal planted benchmark shows why the distinction matters: a nonzero cycle score can be produced by an inconsistent observed edge and removed completely by synchronization. The controlled central experiments show the opposite possibility relative to a prescribed complex: a closed central cocycle may be non-coboundary, and noisy estimates should then be certified, rejected, or left uncertain. The low-rank-adapter audit supplies a practical gauge example in which naive factor averaging depends on an equivalent representation, although dense SVD remains the accuracy boundary. Natural checkpoint and lineage studies do not support a general cycle predictor or a natural central class.

The established contribution is diagnostic. The framework states what has been ruled out, what has merely been represented in a larger predictor, and when no higher-order conclusion is justified. The next decisive evidence must come from an ex ante natural comparison complex, an end-to-end obstruction realization, or a modern multilayer adapter study with measurable decision and systems utility.

\section*{Code and data availability}
The repository state audited for this revision is commit \texttt{b0e1ac4}. The complete source package contains the clean manuscript, a referee-change version, the bibliography, every referenced figure, the source CSV files for generated figures, deterministic plotting code, the comparison-complex and claim-evidence audits, and the report snapshots used for the numerical statements. No invalidated nonabelian accuracy-evidence artifact is used.

\appendix
\section{Extended descent dictionary}\label{app:extended-dictionary}
The compact terminology used in the main text is expanded here. The table is interpretive and does not identify a computational comparison complex with a site-level descent problem without the hypotheses stated in Section~\ref{sec:theory}.

\subsection{Full dictionary}\label{subsec:dictionary}
Below we give a dictionary translating between descent-theoretic terms and their learning-theoretic interpretations. This dictionary is meant as a guide for the reader: it indicates which practical learning objects should be kept in mind when abstract terms such as covers, restriction maps, cocycles, and global sections appear. The goal is to make the following technical introduction readable both to mathematicians and to machine-learning theorists.

\begin{table*}[!tbp]
\centering
\caption{Dictionary between descent-theoretic language and practical learning language.}
\label{tab:learning-descent-dictionary}
\small
\begin{tabular}{L{0.43\textwidth}L{0.49\textwidth}}
\toprule
\textbf{Mathematical/descent term} & \textbf{Practical learning interpretation} \\
\midrule
Global object $X$, site $\sC_{\mathrm{learn}}$, cover $\{U_i\to X\}$ &
The whole merging problem, together with its decomposition into local pieces such as clients, tasks, domains, or checkpoints. \\

Overlaps $U_{ij}$ and $U_{ijk}$ &
Places where local pieces are compared. Pairwise overlaps compare two pieces; triple overlaps check whether the pairwise comparisons fit together.  \\

Local object $M_i$, restriction $M_i|_{U_{ij}}$, and restriction maps &
A local model and the part of it visible on an overlap. The restriction map is the operation that sends a local model to the part that can be compared with another one. \\

Transition maps $g_{ij}$, gauges, gauge group $G$, synchronization, and nonabelian $H^1$ &
The data used to align local models. Gauge changes are allowed changes of coordinates, and synchronization means choosing them so that the alignments become mutually consistent. When the relevant torsor problem is defined, nonabelian $H^1$ classifies gauge-descent data; it is not a numerical mismatch score. \\

Descent datum, strict descent, gluing, and ordinary object &
A descent datum is the collection of local objects and their transition maps. Strict descent means that these data satisfy the cocycle condition. When descent is effective in the ambient category, they glue to an ordinary global object or section; in the computational analogy, this corresponds to a consistent global merge. \\

Global sections, space of global sections, and $H^0$ &
The possible globally consistent choices. In a literal descent problem, these are global sections. In the computational analogy, greedy soup searches only a prescribed finite family of fixed-chart candidate predictors; it should not be identified with the full space $H^0$ unless the relevant descent conditions have been established. \\

Triple-overlap defect $c_{ijk}$, coefficient sheaf $A$, $2$-cocycle,
obstruction class $[c]\in H^2(\sC_{\mathrm{learn}};A)$, and coboundary &
If $g_{ij}g_{jk}g_{ki}=c_{ijk}$, then $c_{ijk}$ first records a triangle residual. When the residuals lie in a fixed central coefficient system and satisfy cocycle closure, they define a class on the chosen comparison complex or cover. A site-level class is obtained only when the relevant comparison map applies. If the certified class is trivial, it can be removed by an edge correction; if it is nontrivial, strict gluing fails on the corresponding complex or cover. \\

Twisted descent datum, failed gluing, and invariant readout &
In a genuine descent problem, a nonzero certified class prevents the permitted edge correction from producing strict descent on the stated complex or cover. The resulting datum is twisted; it determines a twisted sheaf only when the ambient category supports such objects. A projective or branch representation with invariant readout may still yield a globally consistent predictor, but it does not make the original class vanish. \\

Nerve of the cover and \v{C}ech complex &
The combinatorial and algebraic devices used to organize the calculation. The nerve records local pieces, pairwise overlaps, and triple overlaps; the \v{C}ech complex records the corresponding alignment and obstruction data. \\

Obstruction score &
For an arbitrary projected $2$-cochain, this is a numerical distance from the space of coboundaries. It is interpreted as the norm of an $H^2$ class only after cocycle closure has been certified. \\

Period $\per([c])$ and index $\ind([c])$ of an obstruction class $[c]\in H^2(\sC_{\mathrm{learn}};A)$, where $\per([c])$ is the order of $[c]$ and $\ind([c])$ is the rank-divisibility obstruction for twisted objects realizing $[c]$ &
The period tells us the torsion type. The classical index, or the projective-representation index in a controlled finite system, gives an arithmetic rank-divisibility condition. Passing this condition does not by itself construct a lifted predictor; the implemented overlap relations and invariant readout must also pass. \\
\bottomrule
\end{tabular}
\end{table*}

\begin{remark}
    The dictionary above includes both a literal descent interpretation and a finite computational interpretation. In the literal interpretation, $\mathcal U=\{U_i\to X\}$ is a cover in a specified learning site and $N(\mathcal U)$ is its nerve. In the computational interpretation, the algorithm may instead be given a finite comparison complex $K$ whose simplices record the available pairwise and higher-order comparison data. These distinctions are made precise in Definitions~\ref{def:learning-site-cover}--\ref{def:finite-descent-instance}.
\end{remark}

\section{Additional theory and proofs}\label{app:additional-theory}
This appendix retains the fixed-chart empirical-descent proof and the extended background used to relate the certification pipeline to existing merging algorithms. Their placement here should not be read as a claim of theorem-level novelty.

\subsubsection{Greedy soup as a fixed-chart empirical descent analogy}

We now reinterpret greedy model soup \cite{Wortsman2022ModelSoups} through a finite descent analogy. Fix a common parameter chart in which checkpoint averaging is defined. Let $\mathcal S=\{\sigma_1,\ldots,\sigma_N\}$ be a finite set of trained checkpoints. For every nonempty subset $S\subseteq\{1,\ldots,N\}$, let $\sigma_S$ denote the averaged predictor obtained from $\{\sigma_i\}_{i\in S}$. Model souping therefore produces the finite family
\[
\mathcal P_{\mathrm{soup}}:=\{\sigma_S:S\subseteq\{1,\ldots,N\},\ S\neq\varnothing\}.
\]
The validation set defines an empirical risk functional $\widehat R:\mathcal P_{\mathrm{soup}}\to\mathbb R$. These averaged checkpoints are candidate predictors in a fixed chart. They become literal sheaf-theoretic global sections only when the relevant compatibility and effectivity conditions have also been established. Greedy soup is a greedy search over $\mathcal P_{\mathrm{soup}}$, analogous to empirical descent on a finite family of global-section candidates.

\begin{definition}[Empirical compatibility of soup candidates]
Let $S\subseteq\{1,\ldots,N\}$ be the current soup, and let $j\notin S$ be a new checkpoint. We say that $\sigma_j$ is \emph{empirically compatible} with $\sigma_S$ if $\widehat R(\sigma_{S\cup\{j\}})\leq\widehat R(\sigma_S)$. Equivalently, adding $\sigma_j$ does not decrease validation accuracy. Greedy soup accepts exactly the empirically compatible candidates in its prescribed validation order.
\end{definition}

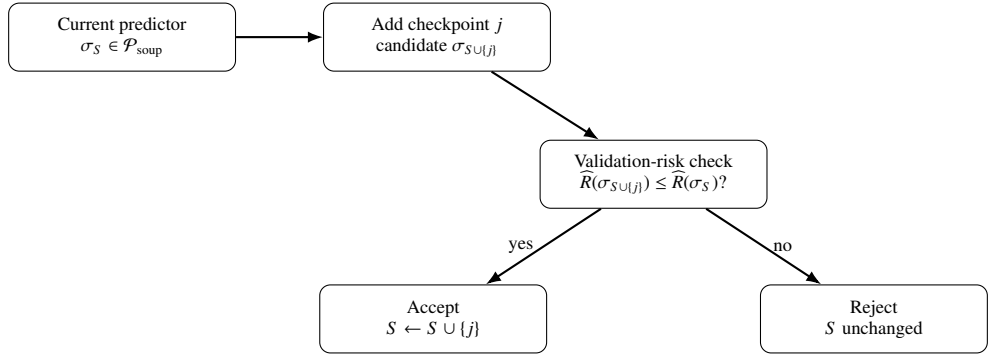
\begin{figure*}[!tbp]
\centering
\begin{tikzpicture}[
    box/.style={
        draw,
        rounded corners,
        align=center,
        minimum width=3.0cm,
        minimum height=0.9cm,
        font=\scriptsize,
        inner sep=4pt
    },
    arrow/.style={-{Latex[length=2mm]}, thick},
    node distance=1.1cm and 1.15cm
]

\node[box] (current) {
    Current predictor\\
    $\sigma_S\in\mathcal P_{\mathrm{soup}}$
};

\node[box, right=of current] (candidate) {
    Add checkpoint $j$\\
    candidate $\sigma_{S\cup\{j\}}$
};

\node[box, below right=0.9cm and -0.15cm of candidate] (test) {
    Validation-risk check\\
    $\widehat R(\sigma_{S\cup\{j\}})
    \leq
    \widehat R(\sigma_S)$?
};

\node[box, below left=1.0cm and -0.1cm of test] (accept) {
    Accept\\
    $S\leftarrow S\cup\{j\}$
};

\node[box, below right=1.0cm and -0.1cm of test] (reject) {
    Reject\\
    $S$ unchanged
};

\draw[arrow] (current) -- (candidate);
\draw[arrow] (candidate) -- (test);
\draw[arrow] (test) -- node[left,font=\scriptsize] {yes} (accept);
\draw[arrow] (test) -- node[right,font=\scriptsize] {no} (reject);

\end{tikzpicture}
\caption{Greedy soup as a fixed-chart empirical descent analogy. A checkpoint is accepted only when the validation risk of the resulting averaged predictor does not increase. The diagram describes a finite candidate-predictor search; it is not a computation of the full space of sheaf-theoretic global sections.}
\label{fig:greedy-soup-empirical-descent}
\end{figure*}

The next proposition records the elementary consequences of this validation rule. The finite-class comparison event used below is standard in empirical risk minimization and uniform-convergence arguments \cite{Vapnik1998,ShalevShwartzBenDavid2014}.

\begin{proposition}[Greedy soup as empirical descent]\label{prop:greedy-soup-descent}
Let $S_0\subseteq S_1\subseteq\cdots\subseteq S_T$ be the sequence of accepted soups produced by greedy soup, and let $\mathcal P_{\mathrm{eval}}$ be the finite family of predictors evaluated by the procedure.

\begin{enumerate}
    \item The empirical risk is nonincreasing along accepted steps:
\[
\widehat R(\sigma_{S_T})\leq\widehat R(\sigma_{S_{T-1}})\leq\cdots\leq\widehat R(\sigma_{S_0}).
\]
    Equivalently, validation accuracy is nondecreasing along accepted steps.

    \item Suppose that, with probability at least $1-\delta$, one has
\[
\sup_{\sigma\in\mathcal P_{\mathrm{eval}}}\left|R(\sigma)-\widehat R(\sigma)\right|\leq\varepsilon,
\]
    where $R$ denotes population risk. Then, on this event,
\[
R(\sigma_{S_T})\leq\min_i R(\sigma_{\{i\}})+2\varepsilon.
\]

    \item Under the same uniform comparison event, suppose that $\sigma_j$ is rejected at the current soup $\sigma_S$, so that $\widehat R(\sigma_{S\cup\{j\}})>\widehat R(\sigma_S)$. Then adding $\sigma_j$ cannot improve population risk by more than $2\varepsilon$:
\[
R(\sigma_{S\cup\{j\}})\geq R(\sigma_S)-2\varepsilon.
\]
If the algorithm uses the strict margin rule $\widehat R(\sigma_{S\cup\{j\}})>\widehat R(\sigma_S)+\tau$, then
\[
R(\sigma_{S\cup\{j\}})\geq R(\sigma_S)+\tau-2\varepsilon.
\]
\end{enumerate}
\end{proposition}

\begin{proof}
For (1), at step $t$, the algorithm accepts $j\notin S_t$ only if $\widehat R(\sigma_{S_t\cup\{j\}})\leq\widehat R(\sigma_{S_t})$. If the candidate is accepted, then $S_{t+1}=S_t\cup\{j\}$, and hence $\widehat R(\sigma_{S_{t+1}})\leq\widehat R(\sigma_{S_t})$. Rejected candidates leave the soup unchanged. Induction gives the monotone chain.

For (2), let $i^*\in\arg\min_i R(\sigma_{\{i\}})$. The initial predictor $\sigma_{S_0}$ is chosen to be the best single checkpoint by validation, so $\widehat R(\sigma_{S_0})\leq\widehat R(\sigma_{\{i^*\}})$. By (1), $\widehat R(\sigma_{S_T})\leq\widehat R(\sigma_{S_0})$. Using the uniform comparison event twice gives
\[
R(\sigma_{S_T})\leq\widehat R(\sigma_{S_T})+\varepsilon
\leq\widehat R(\sigma_{S_0})+\varepsilon
\leq\widehat R(\sigma_{\{i^*\}})+\varepsilon
\leq R(\sigma_{\{i^*\}})+2\varepsilon.
\]

For (3), the uniform comparison event gives $R(\sigma_{S\cup\{j\}})\geq\widehat R(\sigma_{S\cup\{j\}})-\varepsilon$. Since $j$ is rejected, $\widehat R(\sigma_{S\cup\{j\}})>\widehat R(\sigma_S)$, and the same event gives $\widehat R(\sigma_S)\geq R(\sigma_S)-\varepsilon$. Combining these inequalities yields $R(\sigma_{S\cup\{j\}})>R(\sigma_S)-2\varepsilon$. The margin version is identical.
\end{proof}

\begin{corollary}[Validation envelope]\label{cor:validation-envelope}
Let $\mathcal F$ be a finite family of candidate-predictor generators, such as weight averaging, greedy soup, Task Arithmetic, DARE, TIES, and SLERP\@. For each $f\in\mathcal F$, let $\sigma_f$ be the candidate predictor produced by $f$. Let $\widehat f\in\arg\min_{f\in\mathcal F}\widehat R(\sigma_f)$ be selected by validation, and let $f^*\in\arg\min_{f\in\mathcal F}R(\sigma_f)$ minimize population risk within the same finite family. If
\[
\sup_{f\in\mathcal F}\left|R(\sigma_f)-\widehat R(\sigma_f)\right|\leq\varepsilon,
\]
then $R(\sigma_{\widehat f})\leq R(\sigma_{f^*})+2\varepsilon$.
\end{corollary}

\begin{proof}
Since $\widehat f$ is chosen by validation, $\widehat R(\sigma_{\widehat f})\leq\widehat R(\sigma_{f^*})$. Therefore
\[
R(\sigma_{\widehat f})\leq\widehat R(\sigma_{\widehat f})+\varepsilon
\leq\widehat R(\sigma_{f^*})+\varepsilon
\leq R(\sigma_{f^*})+2\varepsilon.
\]
\end{proof}

\begin{remark}
Greedy soup is a greedy algorithm on a finite family of candidate predictors in a fixed chart. It does not compute the full sheaf of models or test twisted cocycle data. Its validation rule is analogous to empirical descent only within the prescribed finite family. This explains why greedy soup is a strong baseline in low-obstruction settings and why it can still miss structured compatibility information that is not represented in its candidate pool.
\end{remark}

\section{Full algorithm and diagnostics}\label{app:full-algorithm}
The following version exposes the optional rank, projective, and nonabelian branches that were suppressed in Algorithm~\ref{alg:twistedmerge-core}.

\subsection{Detailed candidate-generation pipeline}
In this subsection, we record the algorithmic form of TwistedMerge after the distinctions above. The guiding principle is that a model merge is a search over candidate global predictors together with a diagnosis of the comparison data. It also records the gauge family, output type, capacity multiplier, synchronization status, centrality and closure margins, projective rank gate, invariant-readout certificate, and the reason for every abstention.

\begin{table*}[!tbp]
\centering
\caption{Main modules of TwistedMerge. A module either produces a candidate predictor or records a diagnostic explaining why the corresponding branch is unavailable.}
\label{tab:twistedmerge-modules}
\small
\setlength{\tabcolsep}{3pt}
\begin{adjustbox}{max width=\textwidth}
\begin{tabular}{llll}
\toprule
Module & Mathematical object & Computation & Output \\
\midrule
Candidate predictors & $\mathcal P_{\mathrm{cand}}$ & averages, soups, task-vector, and path candidates & ordinary candidates \\
Pairwise alignment & transitions $g_{ij}$ & permutation, monomial, block, or projective fitting & aligned models and edge maps \\
Strict synchronization & candidate untwisted datum & synchronize gauges over $K$ & strict candidate \\
Triangle diagnostics & $g_{ij}g_{jk}g_{ki}$ & cycle or holonomy residuals & face-level defect data \\
Central certification & $\widehat c\in C^2(K;A)$ & inverse, projection, centrality, closure, margin & trivial/nontrivial/uncertified \\
Projective rank gate & $\per([\alpha])$ and $\indrep([\alpha])$ & arithmetic and representation checks & admissible or rejected rank \\
Invariant readout & $P\rho(h)=P$ & path-independence certificate & charted predictor \\
Validation envelope & empirical risk $\widehat R$ & evaluate certified candidates only & selected output \\
Conservative rejection & abstention state & retain ordinary candidates & safe fallback \\
\bottomrule
\end{tabular}
\end{adjustbox}
\end{table*}

\begin{definition}\label{def:residual-certification-diagnostics}
Fix metrics $d_G$ on $G$ and $d_A$ on $A$, write $d_G(x,A)=\inf_{a\in A}d_G(x,a)$, and let $K_1^+$ be a choice of one orientation for each unoriented edge. The diagnostics below are used only when the corresponding edge, face, or tetrahedron set is nonempty; otherwise that branch is skipped rather than evaluated with a zero denominator. When a structured edge family $\mathcal S\subseteq G^{K_1^+}$ is used, let
\[
\Pi_{\mathcal S}(g)=\widetilde g=\{\widetilde g_{ij}\}_{(ij)\in K_1^+}
\]
be the fitted or projected edge connection, extended to reverse orientations by inversion. If no structural projection is used, set $\widetilde g=g$. Define
\[
\widetilde h_{ijk}=\widetilde g_{ij}\widetilde g_{jk}\widetilde g_{ki},
\qquad \widehat c_{ijk}=\pi_A(\widetilde h_{ijk})
\]
whenever the coefficient projection is defined. We use the diagnostics
\begin{align*}
E_{\mathrm{inv}}&=\left(\frac{1}{|K_1^+|}\sum_{(ij)\in K_1^+}d_G(g_{ji}g_{ij},1_G)^2\right)^{1/2},\\
E_{\mathrm{proj}}&=\left(\frac{1}{|K_1^+|}\sum_{(ij)\in K_1^+}d_G(g_{ij},\widetilde g_{ij})^2\right)^{1/2},\\
E_{\mathrm{cent}}&=\left(\frac{1}{|K_2|}\sum_{(ijk)\in K_2}d_G(\widetilde h_{ijk},A)^2\right)^{1/2},\\
E_{\mathrm{cob}}&=\inf_{b\in C^1(K;A)}\left(\frac{1}{|K_2|}\sum_{(ijk)\in K_2}d_A(\widehat c_{ijk},(\delta_1b)_{ijk})^2\right)^{1/2}.
\end{align*}
When $K_3\neq\varnothing$, define
\[
E_{\mathrm{closed}}=\left(\frac{1}{|K_3|}\sum_{(ijkl)\in K_3}d_A((\delta_2\widehat c)_{ijkl},1_A)^2\right)^{1/2};
\]
when $K_3=\varnothing$, set $E_{\mathrm{closed}}=0$ by convention. The projection residual is an edge-level fidelity test: a small post-projection cycle score is not accepted when the learned connection is far from the structural family used to produce it. If $\dim K\leq2$, closure is vacuous for dimensional reasons, but inverse consistency, structural-projection fidelity, centrality, and distance to coboundaries remain nonvacuous.

Here $g_{ji}$ in $E_{\mathrm{inv}}$ denotes a reverse transition fitted independently when such a fit is part of the protocol. If reverse transitions are defined by $g_{ji}=g_{ij}^{-1}$ rather than estimated separately, then $E_{\mathrm{inv}}=0$ identically and serves only as a bookkeeping check.
\end{definition}

\begin{definition}[Three-way certification]
Let $L_{\mathrm{cob}}$ and $U_{\mathrm{cob}}$ be lower and upper confidence bounds for the distance to coboundaries after the structural gates have passed. For thresholds $\tau_0<\tau_1$, the central branch returns
\[
\begin{cases}
\text{trivial},&U_{\mathrm{cob}}\leq\tau_0,\\
\text{nontrivial on }K,&L_{\mathrm{cob}}\geq\tau_1,\\
\text{uncertified},&\text{otherwise}.
\end{cases}
\]
The same uncertified state is returned whenever inverse consistency, centrality, coefficient projection, closure, or the confidence margin fails.
\end{definition}

\begin{breakablealgorithm}
\caption{TwistedMerge}
\label{alg:twistedmerge}
\begin{algorithmic}[1]
\Require Checkpoints $\{M_i\}_{i=1}^n$; alignment data $D_{\mathrm{align}}$; certification data $D_{\mathrm{cert}}$; selection data $D_{\mathrm{select}}$; optional held-out test data $D_{\mathrm{test}}$; a specified finite comparison complex $K$ and, when applicable, a cover $\mathcal U$ with $K=N(\mathcal U)$; gauge families $\mathcal G$; candidate central coefficient systems and projections $\pi_A$; candidate ranks $R$; structural and statistical tolerances.
\Ensure A selected predictor $M^*$ and a diagnostic record $\mathfrak R$.

\State Initialize candidate set $\mathcal C\gets\emptyset$ and diagnostic record $\mathfrak R\gets\emptyset$.
\State Construct ordinary candidates: weight average, greedy soup, and any available same-base task-vector or path candidates. Add them to $\mathcal C$ and record output type and cost for each candidate.

\For{each gauge family $G\in\mathcal G$}
    \State Estimate the edge transitions $g_{ij}\in G$ on $D_{\mathrm{align}}$ and compute pairwise alignment losses.
    \If{the pairwise gate passes}
        \State Synchronize the edge gauges over $K$, construct the strict candidate $\sigma_{\mathrm{sync}}$, and add it to $\mathcal C$.
    \EndIf
    \State Fit or project the structured edge connection $\widetilde g=\Pi_{\mathcal S}(g)$ when applicable, and compute $E_{\mathrm{inv}}$ and $E_{\mathrm{proj}}$ on $D_{\mathrm{cert}}$.
    \State Compute the triangle residuals $\widetilde h_{ijk}=\widetilde g_{ij}\widetilde g_{jk}\widetilde g_{ki}$.

    \For{each candidate central coefficient system $A$ and projection $\pi_A$}
        \State Compute $E_{\mathrm{cent}}$ relative to $A$.
        \If{the inverse-consistency, centrality, and projection-fidelity gates pass}
            \State Form $\widehat c_{ijk}=\pi_A(\widetilde h_{ijk})$.
            \If{$K_3\neq\varnothing$}
                \State Compute $E_{\mathrm{closed}}=\|\delta_2\widehat c\|$.
            \Else
                \State Set $E_{\mathrm{closed}}=0$ for dimensional reasons.
            \EndIf
            \If{the closure gate passes}
                \State Compute a confidence interval for $E_{\mathrm{cob}}=\dist(\widehat c,B^2(K;A))$.
                \If{the status is trivial}
                    \State Solve for an $A$-valued edge correction and add the corrected strict candidate when the overlap equations pass.
                \ElsIf{the status is nontrivial on $K$}
                    \If{a certified classical index, projective-representation index, or proved rank lower-bound gate is available}
                        \For{each candidate rank $r\in R$}
                            \If{$r$ passes the arithmetic rank gate}
                                \State Attempt the projective or charted rank-$r$ construction.
                                \If{the lifted overlap relations and invariant-readout certificate pass}
                                    \State Add the lifted candidate and record its output type and cost.
                                \Else
                                    \State Record construction or readout rejection.
                                \EndIf
                            \Else
                                \State Record arithmetic no-lift rejection for rank $r$.
                            \EndIf
                        \EndFor
                    \Else
                        \State Keep the central class status \emph{nontrivial on $K$}; record the rank gate as uncertified and skip rank constructions.
                    \EndIf
                \Else
                    \State Record central uncertified status and do not activate a lift.
                \EndIf
            \Else
                \State Record a projected cochain that failed closure; do not assign an $H^2$ class.
            \EndIf
        \Else
            \State Record a noncentral, inverse-inconsistent, or projection-unstable result for this coefficient system.
        \EndIf
    \EndFor

    \If{the residual is treated as nonabelian holonomy}
        \State Estimate generators $\gamma_1,\ldots,\gamma_s$ of the holonomy subgroup.
        \State Search for a representation $\rho$ and pooling map $P$.
        \If{$P\rho(\gamma_\ell)=P$ within tolerance for every generator}
            \State Construct the branch predictor, record its branch and inference multipliers, and add it to $\mathcal C$.
        \Else
            \State Record nonabelian no-lift rejection.
        \EndIf
    \EndIf
\EndFor

\State Evaluate $\widehat R(\sigma)$ on $D_{\mathrm{select}}$ for every $\sigma\in\mathcal C$.
\State Select $\sigma^*\in\arg\min_{\sigma\in\mathcal C}\widehat R(\sigma)$.
\If{$D_{\mathrm{test}}$ is supplied}
    \State Evaluate the frozen selection once on $D_{\mathrm{test}}$.
\EndIf
\State Return $M^*=\sigma^*$ and $\mathfrak R$.
\end{algorithmic}
\end{breakablealgorithm}

The separation of $D_{\mathrm{align}}$, $D_{\mathrm{cert}}$, $D_{\mathrm{select}}$, and $D_{\mathrm{test}}$ is conceptually important. Alignment data estimate the transition maps, certification data decide whether a structural branch is admissible, selection data compare the resulting candidates, and test data are evaluation-only. When data are limited, these roles can be cross-fitted rather than literally assigned disjoint samples.

\begin{table*}[!tbp]
\centering
\caption{Diagnostics returned by TwistedMerge. The diagnostic record is part of the output rather than an auxiliary log.}
\label{tab:twistedmerge-diagnostics}
\small
\setlength{\tabcolsep}{3pt}
\begin{adjustbox}{max width=\textwidth}
\begin{tabular}{lll}
\toprule
Diagnostic & Meaning & Use \\
\midrule
Pairwise alignment loss & discrepancy after an edge alignment & edge-level compatibility \\
Inverse residual $E_{\mathrm{inv}}$ & discrepancy between $g_{ji}$ and $g_{ij}^{-1}$ & coherent orientation data \\
Triangle residual & value of $g_{ij}g_{jk}g_{ki}$ & strict-descent failure \\
Centrality residual $E_{\mathrm{cent}}$ & distance to $A\subseteq Z(G)$ & applicability of abelian branch \\
Projection residual $E_{\mathrm{proj}}$ & learned-to-structural distance & rejection of projection traps \\
Closure residual $E_{\mathrm{closed}}$ & size of $\delta_2\widehat c$ & existence of a cohomology class \\
Coboundary interval & uncertainty for $\dist(\widehat c,B^2)$ & trivial/nontrivial/uncertified status \\
Period and index gate & torsion type and rank divisibility & arithmetic lift eligibility \\
Lifted-overlap residual & error in implemented projective relations & construction certificate \\
Pooling certificate & verification of $P\rho(h)=P$ & path-independent readout \\
Output type and cost & single, soup, ensemble, branch, diagnostic & capacity-aware comparison \\
Validation risk & empirical risk of each admitted candidate & final selection \\
\bottomrule
\end{tabular}
\end{adjustbox}
\end{table*}

Let
\[
\mathcal C_{\mathrm{TM}}=\mathcal C_{\mathrm{base}}\cup\mathcal C_{\mathrm{sync}}\cup\mathcal C_{\mathrm{proj}}\cup\mathcal C_{\mathrm{br}}
\]
be the admitted candidate family. The final predictor is
\[
\sigma^*\in\arg\min_{\sigma\in\mathcal C_{\mathrm{TM}}}\widehat R(\sigma).
\]
If no projective or nonabelian branch is certified, then $\mathcal C_{\mathrm{proj}}\cup\mathcal C_{\mathrm{br}}=\varnothing$, and TwistedMerge reduces to validation selection among ordinary and synchronized candidates. This is the conservative fallback used in the natural-data experiments.

\subsubsection{Computational cost}
Let $|E(K)|$, $|F(K)|$, and $|T(K)|$ denote the numbers of edges, faces, and tetrahedra. Let $C_{\mathrm{align}}$ be the cost of fitting one edge alignment, $C_{\mathrm{tri}}$ the cost of one triangle diagnostic, and $C_{\mathrm{val}}$ the cost of validating one candidate. The base connection and triangle diagnostics cost
\[
O\bigl(|E(K)|C_{\mathrm{align}}+|F(K)|C_{\mathrm{tri}}\bigr).
\]
For each candidate coefficient system, one must additionally pay for coefficient projection, closure on $T(K)$, and the distance-to-coboundaries computation; these costs depend on the representation of the coefficient group and on the chosen solver. Rank and branch searches add their own construction costs. The final validation cost is
\[
O\bigl(|\mathcal C_{\mathrm{TM}}|C_{\mathrm{val}}\bigr).
\]
A branch or rank lift may also have an inference multiplier, which is recorded explicitly. The conservative design avoids the construction and inference costs of structural branches whenever their certificates are absent or uncertain.

\section{Geometric discussion}\label{app:neuro-ag}
The following discussion is conceptual background and is not used in the experimental conclusions.

\subsection{Neural quotient stacks and coarse spaces}
We end the theoretical part with a remark on neuro-algebraic geometry. The point is only to explain how the descent-theoretic language above can be viewed as a finite computational shadow of a more geometric object: the \'{e}tale site of a neural variety, or more naturally, of its quotient stack by gauge symmetries.

Let $\Theta$ denote the parameter space for a fixed neural architecture, and let $\operatorname{ev}:\Theta\to \sY$ be an evaluation map to a space of functions, outputs, or finite evaluations. For a polynomial architecture $A$, one can regard this as an algebraic map $\Theta_A\to \cF_A$, where $\Theta_A$ is the parameter space of the architecture and $\cF_A$ is the finite-dimensional affine space of functions, or finite evaluation vectors, represented by the chosen algebraic output coordinates. Then define the associated neurovariety
\[
\sN_A:=\overline{\operatorname{im}(\Theta_A\to \cF_A)}.
\]
The coarse neurovariety records expressivity and algebraic constraints on the realizable functions. For ReLU networks, the same picture should be understood stratumwise or piecewise algebraically, since ReLU is piecewise linear rather than globally algebraic. This viewpoint is compatible with recent work on neuro-algebraic geometry and with earlier algebraic and tropical approaches to neural networks \cite{MarchettiEtAl2025Invitation,KileelTragerBruna2019Polynomial,ZhangNaitzatLim2018Tropical,BrandenburgLohoMontufar2024RealTropical}.

However, the coarse neurovariety does not retain the automorphism data of a parameter presentation. Neural parameter spaces usually carry nontrivial gauge symmetries: hidden-unit permutations, positive ReLU scalings, monomial gauges, and changes of basis can change the parameters without changing the represented function. Thus a natural object for descent is the quotient stack $\mathfrak N_A=[\Theta_A/G_A]$, where $G_A$ is the relevant gauge group. This quotient stack records the parameter symmetries that the image in the coarse space alone does not retain. The quotient stack remembers automorphisms of local models, while the coarse quotient may forget them. This is important for us because descent obstructions are sensitive precisely to such automorphism data.

\begin{definition}
Let $\mathfrak N=[\Theta/G]$ be a quotient stack associated to a neural
architecture. Its \emph{\'{e}tale learning site} is the site
$\mathfrak N_{\et}$ whose objects are \'{e}tale charts $U\to \mathfrak N$ and whose covers are \'{e}tale covers. A local model over such a chart is a section or object of the model stack over $U$, and transition maps on overlaps record how two local parameter charts or latent charts are identified.
\end{definition}

In practice, TwistedMerge works with a finite comparison complex $K$. When the local model charts arise from a genuine cover $\mathcal U$ of a neural quotient stack, one may take $K=N(\mathcal U)$ and interpret the computation as a finite \v{C}ech approximation. In controlled or purely combinatorial experiments, however, $K$ is specified directly, and the resulting cohomology is the cohomology of that finite comparison complex rather than an automatically identified class in the \'{e}tale cohomology of the neural stack.

\begin{definition}
The cohomological Brauer group of the neural stack is $\Br'(\mathfrak N):=H^2(\mathfrak N_{\et},\GG_m)_{\mathrm{tors}}$. The Azumaya Brauer group $\Br(\mathfrak N)$ maps naturally to $\Br'(\mathfrak N)$, and the two are identified only when the comparison map is known to be an isomorphism. Their elements measure torsion central or projective descent obstructions on the quotient stack; see, for example, \cite{GrothendieckBrauer,EdidinHassettKreschVistoli2001}.
\end{definition}

\begin{proposition}[A coarse quotient can forget twisting]
Let $k$ be an algebraically closed field, let $\Theta_A\simeq \mathbb A^N_k$, and let a finite group $G_A$ whose order is invertible in $k$ act on $\Theta_A$. Assume that the coarse quotient $Q_A=\Theta_A/G_A$ exists and is an affine space. Then $\Br'(Q_A)=0$. On the other hand, the quotient stack $\mathfrak N_A=[\Theta_A/G_A]$ can have a nontrivial cohomological Brauer group. More precisely, under the vanishing
\[
H^1((\Theta_A)_{\et};\GG_m)=0, \qquad H^2((\Theta_A)_{\et};\GG_m)=0,
\]
one has a natural identification
\[
\Br'(\mathfrak N_A)\cong H^2(G_A;k^\times),
\]
where $G_A$ acts trivially on $k^\times$. In particular, the quotient stack can carry twisting data that is invisible on its coarse quotient.
\end{proposition}

\begin{proof}
Since $Q_A$ is an affine space over an algebraically closed field, its cohomological Brauer group is trivial: $\Br'(Q_A)=0$; see \cite[Chapter~1]{Colliot-Thelene_Skorobagatov_21}.

For the action of $G_A$ on $\Theta_A$, we use the equivariant Cartan--Leray spectral sequence; see \cite{Stacks}:
\[
E_2^{p,q} = H^p(G_A,H^q((\Theta_A)_{\et},\GG_m)) \Longrightarrow H^{p+q}(([\Theta_A/G_A])_{\et},\GG_m).
\]
Because $\Theta_A\simeq \mathbb A^N_k$, the units are $H^0((\Theta_A)_{\et};\GG_m)=k^\times$, and the action of $G_A$ on these constant units is trivial. The stated vanishing removes the terms $E_2^{1,1}$ and $E_2^{0,2}$ of total degree two. It also eliminates the only possible incoming differential to $E_2^{2,0}$, namely the differential from $E_2^{0,1}$. No differential can leave $E_r^{2,0}$ because the target would have negative second index. Hence the total degree-two group is naturally identified with $H^2(G_A;k^\times)$. Since $G_A$ is finite, this group is torsion, so
\[
\Br'([\Theta_A/G_A])=H^2(([\Theta_A/G_A])_{\et},\GG_m)_{\mathrm{tors}}\cong H^2(G_A;k^\times).
\]
This group can be nontrivial and is the familiar source of projective representations and central extensions. Thus a coarse quotient may have trivial Brauer group while the quotient stack still carries nontrivial twisting.
\end{proof}

\begin{remark}
This proposition explains why the Brauer group appears in our framework. A central/projective obstruction is the finite-computational analogue of a class in $H^2(-,\GG_m)$ on the stack of local model charts. In the experiments, we compute such a class only on a finite nerve. Thus the finite-complex residual or class is a diagnostic for the chosen comparison data, while $\Br'(\mathfrak N)$ is the conceptual geometric object behind the projective site-level theory.
\end{remark}

\begin{ex}[Permutation and ReLU quotient stacks]
    For an MLP with hidden width $m$, hidden-unit permutations give an action of $S_m$ on a parameter chart. ReLU networks also have positive scaling symmetries, so the relevant gauge group is enlarged to a monomial group. Thus one should think of a local parameter chart as being replaced by a quotient stack $[\Theta/G_{\mathrm{mono}}]$, where $G_{\mathrm{mono}}$ contains permutations and positive diagonal rescalings. This explains why raw parameter averaging is not intrinsic, and why alignment or synchronization is needed before averaging.
\end{ex}

\begin{remark}
    The Brauer group records only the central/projective part of the obstruction. If the residual transition data are genuinely nonabelian, then the correct object is not simply a class in $\Br'(\mathfrak N)$. One must also remember the nonabelian holonomy data and the way local gauge groups are identified up to inner automorphism. This is why the executed $S_3$ and $D_4$ construction in this paper is described as a nonabelian holonomy experiment rather than an ordinary Brauer group computation.
\end{remark}

\section{Secondary experiments}\label{app:secondary-experiments}
These experiments support the taxonomy, audit implementation boundaries, or provide controlled algebraic checks. 

\subsection{Independent-seed MLP merging and exact ReLU gauges}\label{subsec:independent-seed-exact-gauges}
We first consider independently trained MLPs. This is the regime in which a common base chart is absent and hidden-unit alignment is part of the problem.

\subsubsection{Small benchmark}
The initial benchmark uses four fixed settings, obtained from $N\in\{3,4\}$ local models and hidden width in $\{32,64\}$, with five seeds per setting. Table~\ref{tab:real-benchmark-accuracy} and Figure~\ref{fig:real-benchmark-accuracy} reproduce the compact comparison. Uniform weight averaging performs poorly, while pairwise and synchronized gauge methods recover much of the lost accuracy. The TwistedMerge selector reaches $0.8690$, compared with $0.8480$ for internal \CtwoMthree{}-style synchronization and $0.8714$ for greedy soup.

\begin{table*}[!tbp]
\centering
\caption{Small independent-seed MLP benchmark. The Git Re-Basin-style and \CtwoMthree{}-style rows are in-repository baselines.}
\label{tab:real-benchmark-accuracy}
\small
\begin{tabular}{lccc}
\toprule
Method & Mean test acc. & $\Delta$ vs \CtwoMthree{} & $\Delta$ vs greedy \\
\midrule
Weight averaging & $0.7186$ & $-0.1294$ & $-0.1528$ \\
Git Re-Basin-style pairwise alignment & $0.8460$ & $-0.0021$ & $-0.0254$ \\
\CtwoMthree{}-style synchronization & $0.8480$ & $0.0000$ & $-0.0233$ \\
Raw monomial scaling & $0.8521$ & $+0.0040$ & $-0.0193$ \\
TwistedMerge selector & $0.8690$ & $+0.0209$ & $-0.0024$ \\
Model Soups / greedy soup & $0.8714$ & $+0.0233$ & $0.0000$ \\
\bottomrule
\end{tabular}
\end{table*}

\begin{figure}[!tbp]
\centering
\includegraphics[width=\columnwidth]{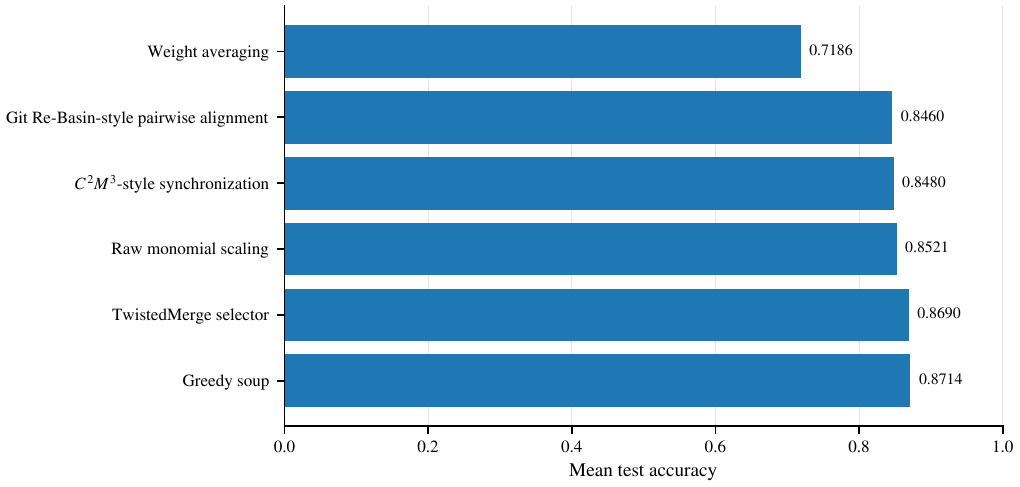}
\caption{Small independent-seed MLP benchmark on the full accuracy scale $[0,1]$. Exact values are printed at the bar ends; TwistedMerge remains close to greedy soup and above the internal \CtwoMthree{}-style baseline.}
\label{fig:real-benchmark-accuracy}
\end{figure}

\begin{table*}[!tbp]
\centering
\caption{Greedy soup versus the TwistedMerge selector on the small benchmark. Each fixed setting uses five seeds; the overall row aggregates twenty runs.}
\label{tab:greedy-vs-twistedmerge-independent-seed}
\small
\begin{tabular}{cccccccc}
\toprule
$N$ & width & runs & greedy & TwistedMerge & $\Delta$ & reported $95\%$ CI & W/T/L \\
\midrule
$3$ & $32$ & $5$ & $0.8622$ & $0.8623$ & $+0.0000$ & $[0.0000,0.0001]$ & $1/4/0$ \\
$3$ & $64$ & $5$ & $0.8803$ & $0.8756$ & $-0.0047$ & $[-0.0076,-0.0012]$ & $1/0/4$ \\
$4$ & $32$ & $5$ & $0.8621$ & $0.8604$ & $-0.0017$ & $[-0.0051,0.0002]$ & $1/2/2$ \\
$4$ & $64$ & $5$ & $0.8809$ & $0.8775$ & $-0.0034$ & $[-0.0080,0.0006]$ & $1/1/3$ \\
\midrule
overall & all & $20$ & $0.8714$ & $0.8690$ & $-0.0024$ & $[-0.0042,-0.0009]$ & $4/7/9$ \\
\bottomrule
\end{tabular}
\end{table*}

\subsubsection{Expanded exact-gauge ladder}
The repository also contains a larger confirmatory exact-gauge ladder. It uses MNIST, one-hidden-layer ReLU MLPs, $N\in\{3,4\}$, widths $16,32,64$, and twenty seeds per fixed setting, giving $120$ paired runs. The candidate family includes raw, shrinkage, global, and optimized positive monomial gauges, gauge-aware soups, and validation selectors.

\begin{table*}[!tbp]
\centering
\caption{Expanded independent-seed exact-gauge ladder over $120$ paired runs. Exact positive scaling variants improve the internal \CtwoMthree{}-style baseline, but the improved selector remains slightly below greedy soup.}
\label{tab:expanded-exact-gauge-ladder}
\small
\begin{tabular}{lc}
\toprule
Method & Mean test accuracy \\
\midrule
\CtwoMthree{}-style synchronization & $0.8119$ \\
Raw monomial scaling & $0.8197$ \\
Shrinkage monomial scaling & $0.8268$ \\
Global positive scaling & $0.8269$ \\
Optimized monomial scaling & $0.8206$ \\
Improved TwistedMerge selector & $0.8557$ \\
Greedy soup & $0.8572$ \\
\bottomrule
\end{tabular}
\end{table*}

The paired mean difference between the improved selector and internal \CtwoMthree{}-style synchronization is $+0.0438$, with reported $95\%$ paired bootstrap confidence interval $[0.0372,0.0508]$ and win/tie/loss count $117/0/3$. Against greedy soup, the difference is $-0.0015$, with reported $95\%$ paired bootstrap confidence interval $[-0.0023,-0.0008]$ and count $20/49/51$. Shrinkage and global scaling improve raw monomial scaling by $+0.0070$ and $+0.0071$, respectively. These results support the exact-gauge mechanism from Theorem~\ref{thm:exact-monomial-relu} only in the sense of function-preserving ReLU reparameterization; the empirical gain comes from the candidate construction and validation procedure, not from a central obstruction certificate.

\subsubsection{Official-core audit and selector attribution}\label{subsec:official-selector-attribution}
A later audit connects adapter-assisted official matching cores to exact checkpoint families. Table~\ref{tab:official-core-audit} reports the successful rows. All entries are same-capacity single models with one inference path. The official \CtwoMthree{} core improves over its matching internal implementation by $+0.0128$, and the official Git Re-Basin core improves over its internal pairwise implementation by $+0.0086$. Official TIES agrees exactly with the internal TIES-style merge on the three audited same-base rows.

\begin{table*}[!tbp]
\centering
\caption{Adapter-assisted official-core audit. The rows use official matching or merge cores through explicit checkpoint adapters; they are not unmodified end-to-end executions of the original repositories.}
\label{tab:official-core-audit}
\small
\begin{tabular}{llccc}
\toprule
Core & Regime & settings & mean test acc. & $\Delta$ vs matching internal \\
\midrule
Official Git Re-Basin & independent initialization & $20$ & $0.8546$ & $+0.0086$ \\
Official \CtwoMthree{} & independent initialization & $20$ & $0.8608$ & $+0.0128$ \\
Official TIES & same base & $3$ & $0.8215$ & $0.0000$ \\
\bottomrule
\end{tabular}
\end{table*}

On the twenty independent-initialization settings, official \CtwoMthree{} exceeds the pure TwistedMerge monomial-gauge candidate by $+0.0087$, with reported $95\%$ paired bootstrap confidence interval $[0.0050,0.0123]$. The existing validation selector exceeds official \CtwoMthree{} by $+0.0081$, with reported $95\%$ paired bootstrap confidence interval $[0.0030,0.0136]$, and official Git Re-Basin by $+0.0143$, with reported $95\%$ paired bootstrap confidence interval $[0.0074,0.0221]$. Because that selector chooses from an enlarged candidate pool, these comparisons do not isolate a TwistedMerge-specific algorithmic contribution.

The budget-matched attribution experiment addresses this issue on forty exact settings grouped into ten independent training seeds. We use the repository's notation: A0 is the ordinary greedy baseline; A1 is the official-synchronization pool; A2 is gauge-only augmentation; A3 is gauge-plus-soup augmentation; A4 is diagnostic-only augmentation; A5 is the full TwistedMerge selector; B0 is an ordinary soup control matched to A5 in candidate count and selector validation evaluations; and A6 is a nondeployable test-oracle upper bound. A5 contains no lift. Candidate-generation kernels and total generation compute are recorded but are not identical between A5 and B0.

The full selector A5 has mean test accuracy $0.8707$, whereas B0 has mean test accuracy $0.8726$. Their group-bootstrap paired difference is $-0.0019$, with reported $95\%$ group-bootstrap confidence interval $[-0.0026,-0.0012]$. The primary attribution gate fails. Figure~\ref{fig:repo-selector-attribution} gives a publication-scale summary regenerated from the exact values in the repository report. The supported conclusion is enriched-pool validation selection, not a TwistedMerge-specific selector gain.

\begin{figure*}[!tbp]
\centering
\includegraphics[width=0.82\textwidth]{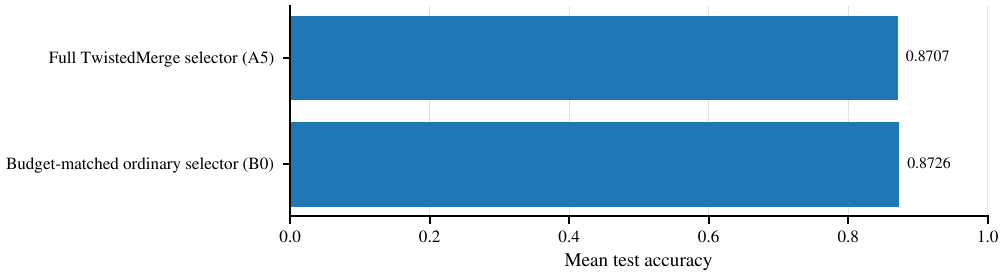}
\par\medskip
\includegraphics[width=0.82\textwidth]{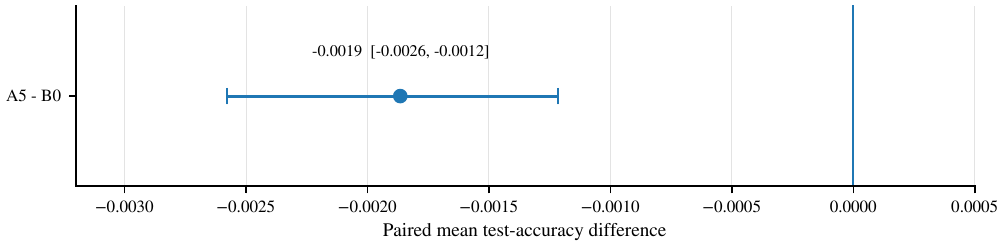}
\caption{Primary selector-attribution comparison regenerated from the repository report. A5 is the full TwistedMerge selector and contains no lift; B0 is the ordinary selector matched to A5 in candidate count and selector validation evaluations. Top: mean test accuracy on the full scale $[0,1]$. Bottom: paired mean difference A5$-$B0 with the reported $95\%$ group-bootstrap confidence interval.}
\label{fig:repo-selector-attribution}
\end{figure*}

\subsection{Same-base fixed-chart merging and empirical descent}\label{subsec:same-base-fixed-chart}
Task Arithmetic, TIES, DARE, and SLERP require a common base chart. The repository therefore evaluates them in a separate same-base benchmark rather than placing them in the independent-seed leaderboard. One base checkpoint using the repository architecture labelled \texttt{mlp2} is trained, task-specific copies are fine-tuned from that base, and all candidate hyperparameters are chosen on validation data. Git Re-Basin and \CtwoMthree{} are recorded as not-run secondary diagnostics because the construction does not intentionally introduce an independent permutation mismatch.

\begin{table*}[!tbp]
\centering
\caption{Representative same-base task-vector results over twenty seeds per setting. The fine-tuned oracle mean is not a single merged model.}
\label{tab:same-base-task-vector}
\small
\begin{adjustbox}{max width=\textwidth}
\begin{tabular}{lccc}
\toprule
Method & MNIST digit subsets, width 64 & MNIST digit subsets, width 128 & Fashion subsets, width 64 \\
\midrule
Greedy soup & $0.8164$ & $0.8345$ & $0.6828$ \\
Sequential SLERP & $0.8166$ & $0.8347$ & $0.6832$ \\
Task Arithmetic & $\mathbf{0.8603}$ & $\mathbf{0.8839}$ & $0.6932$ \\
DARE & $0.8556$ & $0.8770$ & $\mathbf{0.6936}$ \\
TIES & $0.8245$ & $0.8500$ & $0.6462$ \\
Fine-tuned oracle mean & $0.9267$ & $0.9372$ & $0.7779$ \\
\bottomrule
\end{tabular}
\end{adjustbox}
\end{table*}

\begin{figure}[!tbp]
\centering
\includegraphics[width=\columnwidth]{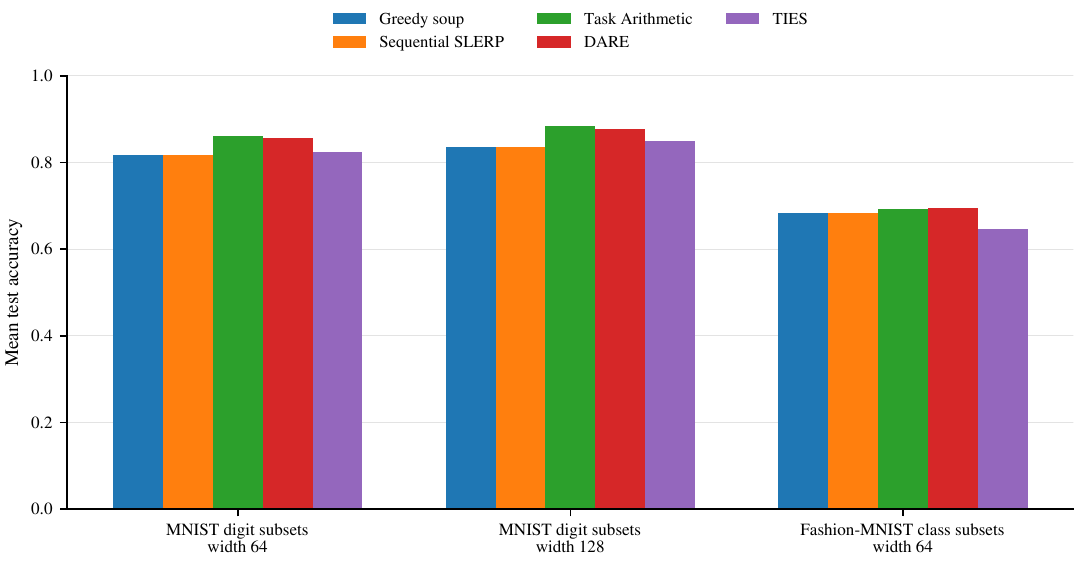}
\caption{Same-base fixed-chart merging on the full accuracy scale $[0,1]$. Task Arithmetic and DARE improve over the original greedy-soup pool in the stated settings because the common base already supplies the chart required by the task-vector operations.}
\label{fig:same-base-fixed-chart}
\end{figure}

On MNIST digit subsets with width $64$, Task Arithmetic improves over the original greedy soup by $+0.0439$, with reported $95\%$ paired bootstrap confidence interval $[0.0407,0.0474]$. At width $128$, the corresponding gain is $+0.0494$, with reported $95\%$ paired bootstrap confidence interval $[0.0469,0.0521]$. On Fashion-MNIST width $64$, DARE improves by $+0.0109$, with reported $95\%$ paired bootstrap confidence interval $[0.0048,0.0161]$. These are exact-setting same-base results, not evidence about independently trained models.

The descent-envelope experiment enlarges the generated candidate family to include weight averages, base and fine-tuned models, SLERP, Task Arithmetic, TIES, DARE, and the original greedy soup. Greedy soup over this enriched generated family improves over the checkpoint-only greedy soup by $+0.0227$, with reported $95\%$ paired bootstrap confidence interval $[0.0189,0.0267]$. This is a direct empirical illustration of Corollary~\ref{cor:validation-envelope}: the improvement comes from enlarging the finite family of candidate predictors while retaining validation-only selection.

\subsubsection{Finite-index and higher period-index thresholds}
Proposition~\ref{prop:clock-shift-divisibility} gives the first rank gate. The repository tests primitive orders $d=2,3,4,5,6$ and several nonprimitive roots. In every case, the smallest successful clock--shift rank equals the order $d$, and, among the tested ranks, success occurs exactly at multiples of $d$. This experiment is described as a projective or Morita realization of the relation $AB=\zeta BA$, not as the vanishing of the original scalar $2$-cocycle.

The higher finite-Heisenberg benchmark realizes Theorem~\ref{thm:finite-heisenberg-index}. Table~\ref{tab:controlled-period-index-thresholds} shows that the period remains $d$ while the projective-representation index grows to $d^k$ for $k$ independent Weyl pairs.

\begin{table*}[!tbp]
\centering
\caption{Controlled finite-Heisenberg period and projective-representation index. The minimal successful tested rank equals $d^k$.}
\label{tab:controlled-period-index-thresholds}
\small
\begin{tabular}{cccc}
\toprule
Case & Period $d$ & Pairs $k$ & $\indrep([\alpha])=d^k$ \\
\midrule
$d=2,k=1$ & $2$ & $1$ & $2$ \\
$d=2,k=2$ & $2$ & $2$ & $4$ \\
$d=2,k=3$ & $2$ & $3$ & $8$ \\
$d=3,k=1$ & $3$ & $1$ & $3$ \\
$d=3,k=2$ & $3$ & $2$ & $9$ \\
$d=4,k=1$ & $4$ & $1$ & $4$ \\
$d=4,k=2$ & $4$ & $2$ & $16$ \\
\bottomrule
\end{tabular}
\end{table*}

\begin{figure}[!tbp]
\centering
\includegraphics[width=\columnwidth]{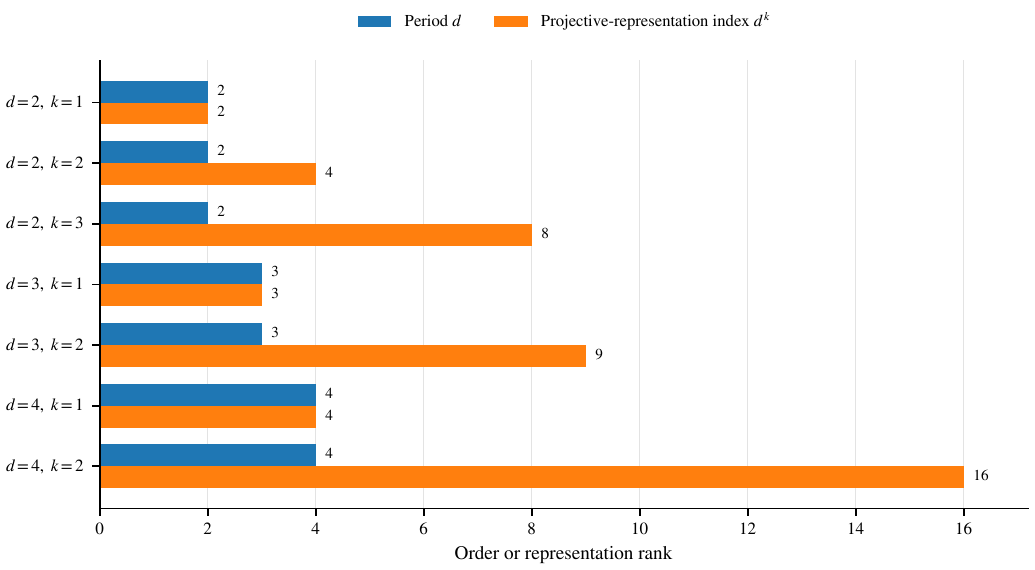}
\caption{Controlled period-index thresholds. Here $d$ is the commutator-phase order and $k$ is the number of independent Weyl pairs. Period divisibility alone is insufficient when several independent pairs are present.}
\label{fig:controlled-period-index-thresholds}
\end{figure}

\subsubsection{Known and learned time-frequency charts}
Finite signals on $\ZZ/d\ZZ$ give a natural controlled source of the Weyl relation. If $T$ is the cyclic time shift and $M$ is modulation, then
\[
MT=\zeta TM.
\]
For $k$ independent pairs, the expected threshold is $d^k$. The known-operator benchmark recovers this threshold: $(d,k)=(2,2),(2,3),(3,2),(4,1)$ have minimal certified ranks $4,8,9,4$, respectively. The corresponding controlled orbit-invariant classifier accuracies at the admissible rank are $0.9155$, $0.9734$, $1.0000$, and $0.9557$.

\begin{figure}[!tbp]
\centering
\includegraphics[width=\columnwidth]{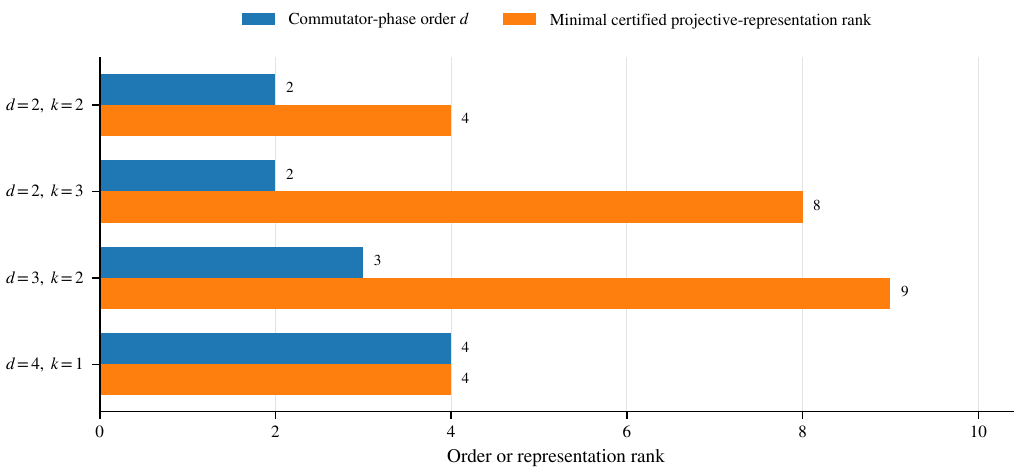}
\caption{Finite time-frequency chart thresholds. Here $d$ is the commutator-phase order and $k$ is the number of independent time-frequency pairs. The phase order and the minimal certified projective-representation rank differ when $k>1$.}
\label{fig:time-frequency-rank-thresholds}
\end{figure}

The learned-chart experiment estimates the transition operators from paired observations. Clean least-squares and full-dimensional linear-autoencoder estimates recover the expected phase order and projective-representation rank threshold with certification rate $1$ and operator errors near numerical precision. Raw estimates at noise levels $0.01$ and $0.05$ are rejected. For the $d=3,k=2$ case, candidate ranks $3$ and $6$ remain rejected and rank $9$ is selected only when the learned structural relations pass. A subsequent denoising experiment shows that nearest-unitary projection can raise the $d=2,k=2$, rank-$4$, noise-$10^{-4}$ certification rate from $0.5$ to $0.9$. Finite-Heisenberg projection extends the range further, but only when the learned-to-projected residual also passes. This projection residual is the empirical reason for including $E_{\mathrm{proj}}$ in Definition~\ref{def:residual-certification-diagnostics}.

\subsection{Executed two-loop noncommuting holonomy}\label{subsec:controlled-holonomy}
The central experiments use abelian coefficient data. We now consider a genuinely nonabelian residual, where the relevant object is a holonomy subgroup rather than an ordinary $H^2$ class. The executed construction uses the groups $S_3$ and
\[
D_4=\langle r,s\mid r^4=s^2=1,\ srs=r^{-1}\rangle.
\]
The comparison complex is a wedge of two length-three cycles, $0$--$1$--$2$--$0$ and $0$--$3$--$4$--$0$. The first loop carries a planted transposition or reflection $s$, and the second carries a planted $3$-cycle or rotation $r$. Five local checkpoints are exact hidden-unit reparameterizations of the same executed one-hidden-layer ReLU MLP\@. The duplicated regular hidden orbit supplies exact automorphisms carrying the two transitions, while the remaining hidden units are generic. The run grid uses widths $32$ and $64$, with seeds $0,\ldots,49$, giving $100$ executed instances per group.

Table~\ref{tab:two-loop-holonomy-structural} records the structural certificates. The pre-lift, commutator, and post-lift entries are the normalized operator residuals defined by the executed benchmark; they are report-specific diagnostics rather than universal group invariants. The positive commutator residual verifies that the two recovered generators do not commute. The local functional-equivalence residual is at numerical precision, the regular-action multiplication residual is zero, and both invariant-pooling residuals vanish. Thus the experiment verifies the hypotheses of Proposition~\ref{prop:invariant-pooling-descent} for two noncommuting generators.

\begin{table*}[!tbp]
\centering
\caption{Executed two-loop noncommuting holonomy certificates. The post-lift and pooling residuals vanish, while the positive commutator residual certifies noncommutation.}
\label{tab:two-loop-holonomy-structural}
\small
\begin{adjustbox}{max width=\textwidth}
\begin{tabular}{lccccccc}
\toprule
Group & pre-lift residual & commutator residual & post-lift residual & pool $\gamma_1$ & pool $\gamma_2$ & action multiplication & local equivalence \\
\midrule
$S_3$ & $0.522693$ & $0.522693$ & $0$ & $0$ & $0$ & $0$ & $5.40\times10^{-16}$ \\
$D_4$ & $0.603553$ & $0.603553$ & $0$ & $0$ & $0$ & $0$ & $5.61\times10^{-16}$ \\
\bottomrule
\end{tabular}
\end{adjustbox}
\end{table*}

\begin{figure}[!tbp]
\centering
\includegraphics[width=\columnwidth]{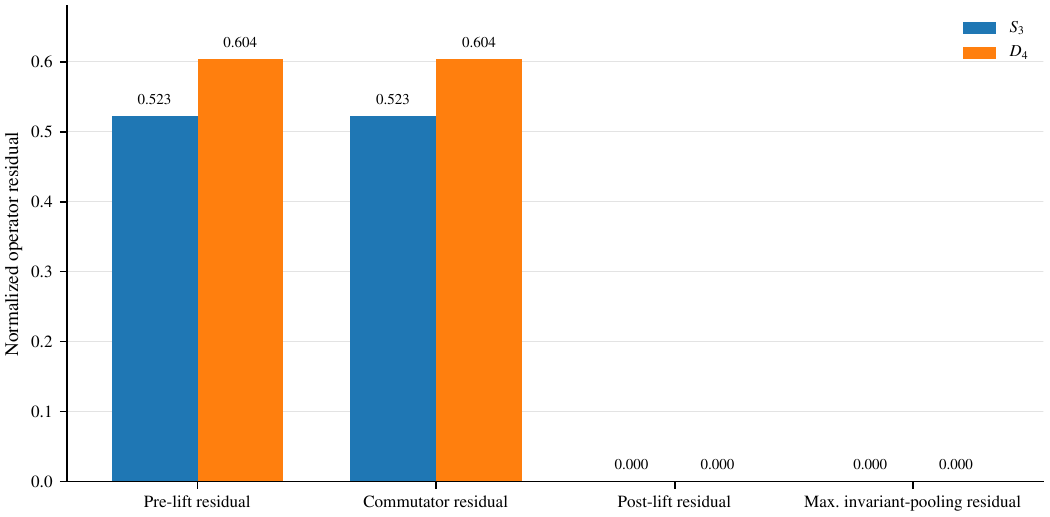}
\caption{Executed two-loop structural residuals. The displayed quantities are the normalized operator residuals used by the benchmark. The maximum invariant-pooling residual is the maximum of the two generator-wise pooling residuals. The pre-lift and commutator residuals are positive for both groups, whereas the post-lift and pooling residuals vanish.}
\label{fig:two-loop-noncommuting-holonomy}
\end{figure}

The executed accuracy results impose an equally important negative boundary. Table~\ref{tab:two-loop-holonomy-accuracy} averages the two widths. Ordinary unaligned weight averaging is substantially worse, but every evaluated aligned, synchronized, soup, and branch method reaches accuracy $1$. The random same-branch-count and wrong-action controls also reach accuracy $1$. Consequently, the paired accuracy difference between the branch regular lift and Git Re-Basin-style alignment, \CtwoMthree{}-style synchronization, greedy soup, or any of the recorded branch controls is exactly zero, with $100$ ties per group.

\begin{table*}[!tbp]
\centering
\caption{Executed two-loop accuracy boundary, averaged over widths $32$ and $64$. The branch model has additional inference paths. The equality of the aligned and lifted rows is retained as a negative empirical result.}
\label{tab:two-loop-holonomy-accuracy}
\small
\begin{adjustbox}{max width=\textwidth}
\begin{tabular}{lcccl}
\toprule
Method & $S_3$ acc. & $D_4$ acc. & inference multiplier & output type \\
\midrule
Ordinary weight average & $0.5667$ & $0.5768$ & $1$ & single model \\
Git Re-Basin-style pairwise alignment & $1.0000$ & $1.0000$ & $1$ & single model \\
\CtwoMthree{}-style synchronization & $1.0000$ & $1.0000$ & $1$ & single model \\
Greedy soup & $1.0000$ & $1.0000$ & $1$ & soup \\
Naive regular representation, no pooling & $1.0000$ & $1.0000$ & $6$ / $8$ & branch model \\
Branch regular lift with invariant pooling & $1.0000$ & $1.0000$ & $6$ / $8$ & branch model \\
Random same-branch-count control & $1.0000$ & $1.0000$ & $6$ / $8$ & branch model \\
Wrong generator, order, and group-action controls & $1.0000$ & $1.0000$ & $6$ / $8$ & branch models \\
\bottomrule
\end{tabular}
\end{adjustbox}
\end{table*}

Accordingly, this experiment supports only a structural nonabelian claim: the executed models carry two noncommuting holonomies, and invariant pooling makes both generators invisible at the readout. It does not support an accuracy advantage for the branch lift. The distinction is essential, because the structural certificate remains meaningful even when ordinary alignment and the negative controls already solve the particular classification task.

\subsubsection{Observed branch and rank lifts}
Representative MNIST and Fashion-MNIST branch rows are shown in Table~\ref{tab:real-branch-lift-boundary}. Greedy soup remains a strong boundary. The small rank and branch candidates sometimes improve over synchronized or averaged baselines, but they do not reliably exceed greedy soup and do not carry a certified natural finite-index class.

\begin{table*}[!tbp]
\centering
\caption{Representative observed MNIST/Fashion-MNIST branch boundary. In the last column, the first value is the rank-lift accuracy and the second is the validation-selected branch accuracy. These outputs may have additional inference capacity.}
\label{tab:real-branch-lift-boundary}
\small
\begin{adjustbox}{max width=\textwidth}
\begin{tabular}{lccccc}
\toprule
Dataset & $N$ & width & setting & greedy soup & rank lift / validation branch \\
\midrule
Fashion-MNIST & $3$ & $64$ & input noise, activation & $0.8154$ & $0.7982\,/\,0.8095$ \\
Fashion-MNIST & $3$ & $64$ & input noise, weight & $0.8154$ & $0.8023\,/\,0.8095$ \\
Fashion-MNIST & $3$ & $64$ & none, activation & $0.8242$ & $0.8180\,/\,0.8247$ \\
Fashion-MNIST & $4$ & $128$ & none, weight & $0.8328$ & $0.8265\,/\,0.8332$ \\
MNIST & $3$ & $64$ & input noise, activation & $0.9016$ & $0.8963\,/\,0.8995$ \\
MNIST & $3$ & $64$ & input noise, weight & $0.9016$ & $0.8955\,/\,0.8995$ \\
\bottomrule
\end{tabular}
\end{adjustbox}
\end{table*}

\subsubsection{Exact channel gauges and vision boundaries}
The exact positive-gauge experiment extends to a small no-BatchNorm Fashion-MNIST convolutional neural network (CNN)\@. Across fifteen settings, internal \CtwoMthree{}-style channel synchronization reaches $0.7716$, optimized exact channel scaling reaches $0.7809$, and greedy soup reaches $0.8442$. The optimized channel scale improves over \CtwoMthree{} by $+0.00933$, with reported $95\%$ paired bootstrap confidence interval $[0.00479,0.01456]$, but remains $0.06323$ below greedy soup. The residuals are noncentral under the tested diagnostics, so no CNN Brauer or period-index claim is made.

A separate ResNet-18 identity audit examines BatchNorm rather than merge accuracy. Compatible graph-wide channel permutations preserve predictions in evaluation and training modes within the preregistered floating-point tolerance. Positive scaling is evaluation-exact only for the explicit frozen-statistic affine or epsilon-aware running-affine parameterizations tested there; scaling the stored running moments alone fails when $\varepsilon>0$, and static positive scaling is not train-mode exact. This audit limits the scope of Theorem~\ref{thm:exact-monomial-relu}; it does not provide an independent-initialization ResNet merging result.

\begin{figure*}[!tbp]
\centering
\includegraphics[width=0.82\textwidth]{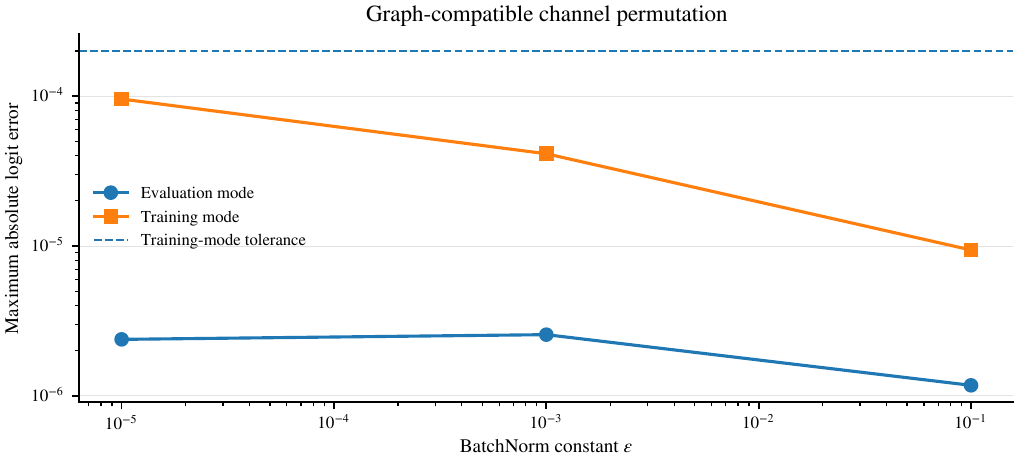}
\par\medskip
\includegraphics[width=0.82\textwidth]{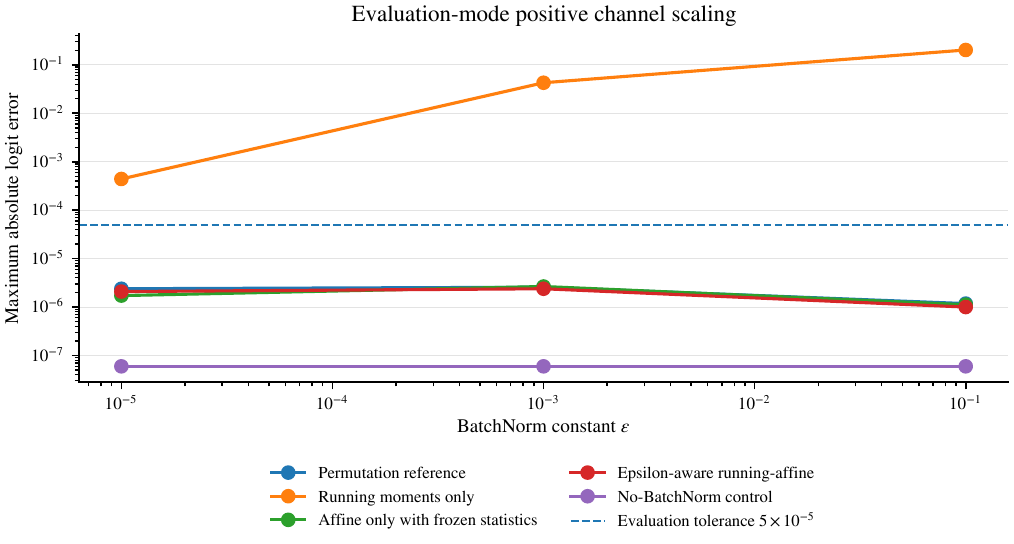}
\caption{ResNet-18 BatchNorm identity audit regenerated from the exact repository summary values. Top: graph-compatible channel permutations in evaluation and training modes; the dashed line is the preregistered training-mode tolerance. Bottom: evaluation-mode scaling strategies. ``Affine only with frozen statistics'' leaves the stored moments fixed, while ``epsilon-aware running-affine'' transforms the moments and applies the corresponding $\varepsilon$-aware affine correction. The dashed line is the evaluation-mode tolerance. Both vertical axes report maximum absolute logit error. These are functional-identity tests, not merge-accuracy results.}
\label{fig:repo-batchnorm-gauge-audit}
\end{figure*}

\begin{figure}[!tbp]
\centering
\includegraphics[width=\columnwidth]{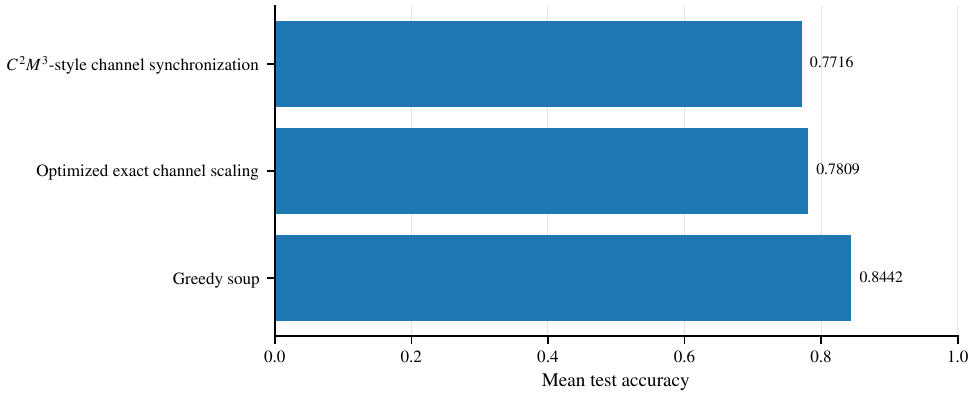}
\caption{Fashion-MNIST no-BatchNorm CNN boundary on the full accuracy scale $[0,1]$. Exact channel scaling gives a limited improvement over internal strict synchronization but remains below greedy soup.}
\label{fig:fashion-cnn-channel-gauge-boundary}
\end{figure}

Rotated- and colored-MNIST bridge datasets show the same boundary pattern. Across seventeen settings, greedy soup improves over \CtwoMthree{}-style channel synchronization by $+0.0847$, while the validation-safe selector ties greedy soup. These are MNIST-derived bridge datasets and do not imply CIFAR or broad vision performance.

The bounded no-BatchNorm CIFAR rescue clears the base-accuracy gate, with mean per-run best individual-model accuracy $0.6506$. Nevertheless, optimized channel scaling is only $+0.00054$ versus \CtwoMthree{}, with reported $95\%$ paired bootstrap confidence interval $[-0.00088,0.00274]$, while greedy soup reaches $0.6480$ and internal \CtwoMthree{} reaches $0.4644$. The union candidate soup is only $+0.00044$ over greedy soup, with a reported $95\%$ paired bootstrap confidence interval that contains zero. We therefore retain this no-BatchNorm CIFAR study as a descriptive appendix boundary rather than as evidence for the main method. A later torchvision-style CIFAR-10 ResNet-18 pipeline is test-isolated, checkpoint-resumable, and smoke-tested, but its base-quality pilot and confirmatory merge study remain gated and unrun at the recorded commit; it supplies no positive or negative model-merging result.

\FloatBarrier

\section*{Report on AI Usage}
In preparing this paper, we used ChatGPT 5.6 Pro to improve the language and organization of the manuscript. We also used Codex to assist with coding and running the experiments. All mathematical arguments, experimental designs, results, and conclusions were reviewed and verified by the authors, who take full responsibility for the content of the paper.

\section*{Acknowledgements}
The first author thanks the UW Math AI Lab for supporting this project and his advisor, Max Lieblich, for his guidance. Besides, the first author thanks the hospitality of the Beijing International Center for Mathematical Research (BICMR), where part of this work was carried out. The first author also thanks Soham Ghosh and Michael Ruofan Zeng for discussion. The second author thanks his mentor Zhiyu Tian for his support and guidance and BICMR for providing an excellent working environment.

\begingroup
\sloppy
\bibliographystyle{elsarticle-num}
\bibliography{references}
\endgroup
\end{document}